\documentclass[twoside]{article}

\usepackage[accepted]{aistats2021}

\usepackage[round]{natbib}

\usepackage{mkolar_definitions}
\usepackage{mathtools}
\usepackage{amsmath}

\usepackage{nameref}
\usepackage{cleveref}
\usepackage[vlined,ruled,linesnumbered,noend]{algorithm2e}
\usepackage{dsfont}
\usepackage{xcolor}
\usepackage{verbatim}
\usepackage[shortlabels]{enumitem}
\usepackage{amssymb}
\usepackage{subcaption}
\usepackage{commath}
\usepackage[shortlabels]{enumitem}
\usepackage{thmtools}
\usepackage{thm-restate}
\usepackage{enumitem}
\usepackage{bibunits}
\usepackage{authblk}

\defaultbibliography{ref}
\defaultbibliographystyle{abbrvnat}

\DeclareMathOperator{\UCB}{UCB}
\DeclareMathOperator{\Ber}{Ber}
\DeclareMathOperator{\KL}{KL}
\DeclareMathOperator{\TV}{TV}

\DeclareMathOperator{\Unif}{Unif}

\renewcommand{\ind}{\mathds{1}}
\newcommand{\lc}{l^C}

\newcommand{\robustagg}{\ensuremath{\textsc{RobustAgg}}\xspace}
\newcommand{\arobustagg}{\ensuremath{\textsc{RobustAgg-Agnostic}}\xspace}
\newcommand{\inducb}{\ensuremath{\textsc{Ind-UCB}}\xspace}
\newcommand{\naiveagg}{\ensuremath{\textsc{Naive-Agg}}\xspace}
\newcommand{\adaptedrobustagg}{\ensuremath{\textsc{RobustAgg-Adapted}}\xspace}
\newcommand{\wbar}[1]{ {\ensuremath{\overline{#1}}} }
\newcommand{\corral}{\textsc{Corral}\xspace}

\begin{document}

\runningauthor{Zhi Wang\textsuperscript{*}, Chicheng Zhang\textsuperscript{*}, Manish Kumar Singh, Laurel D. Riek, Kamalika Chaudhuri}

\twocolumn[

\aistatstitle{Multitask Bandit Learning Through  Heterogeneous Feedback Aggregation}

\aistatsauthor{Zhi Wang\textsuperscript{*}$^1$, Chicheng Zhang\textsuperscript{*}$^2$, Manish Kumar Singh$^1$, Laurel D. Riek$^1$, Kamalika Chaudhuri$^1$}

\aistatsaddress{$^1$University of California San Diego, $^2$University of Arizona} ]

\begin{bibunit}

\begin{abstract}
In many real-world applications, multiple agents seek to learn how to perform highly related yet slightly different tasks in an online bandit learning protocol. We formulate this problem as the \textit{$\epsilon$-multi-player multi-armed bandit} problem, in which a set of players concurrently interact with a set of arms, and for each arm, the reward distributions for all players are similar but not necessarily identical.
We develop an upper confidence bound-based algorithm, $\robustagg(\epsilon)$, that adaptively aggregates rewards collected by different players.
In the setting where an upper bound on the pairwise dissimilarities of reward distributions between players is known, we achieve instance-dependent regret guarantees that depend on the amenability of information sharing across players. We complement these upper bounds with nearly matching lower bounds.
In the setting where pairwise dissimilarities are unknown, we provide a lower bound, as well as an algorithm that trades off minimax regret guarantees for adaptivity to unknown similarity structure.
\end{abstract}

\section{Introduction}
\label{sec:intro}

Online multi-armed bandit learning has many important real-world applications \cite[see][for a few examples]{vbw15,swjz15,lcls10}.
In practice, a group of online bandit learning agents are often deployed for similar tasks, and they learn to perform these tasks in similar yet nonidentical environments.
For example, a group of assistive healthcare robots may be deployed to provide personalized cognitive training to people with dementia (PwD), e.g., by playing cognitive training games with people \citep{jessie2020}.
Each robot seeks to learn the preferences of its paired PwD so as to recommend tailored health intervention based on how the PwD reacts to and is engaged with the activities (as captured by sensors on the robots) \citep{jessie2020}.
As PwD may have similar preferences and may therefore exhibit similar reactions, one natural question arises---can the robots as a multi-agent system learn to perform their respective tasks faster through collaboration? In this paper, we 
develop multi-agent bandit learning algorithms where each agent can robustly aggregate data from other agents to better perform its respective task.

We generalize the multi-armed bandit problem \citep{acf02} and formulate the $\epsilon$-Multi-Player Multi-Armed Bandit ($\epsilon$-MPMAB) problem, which models \textit{heterogeneous multitask learning} in a multi-agent bandit learning setting.
In an $\epsilon$-MPMAB problem instance, a set of $M$ players are deployed to perform similar tasks---simultaneously they interact with a set of actions/arms, and for each arm, different players receive feedback from similar but not necessarily identical reward distributions.
In the above assistive robotics example, each player corresponds to a robot; each arm corresponds to one of the cognitive activities to choose from; for each player and each arm, there is a separate reward distribution which reflects a PwD's personal preferences. Informally, $\epsilon \geq 0$ is a \textit{dissimilarity parameter} that upper bounds the pairwise distances between different reward distributions for different players on the same arm (see Definition~\ref{assumption:epsilon} in the next section). The players can communicate and share information among each other, with a goal of maximizing their collective reward. %

While multi-player bandit learning has been studied extensively in the literature \cite[e.g.,][]{lsl16,cgz13,glz14}, warm-starting bandit learning using different feedback sources has been investigated \citep{zailn19},
and sequential transfer between similar tasks in a bandit learning setting has also been studied \citep{glb13,soaremulti},
to our knowledge, no prior work models multitask learning in a multi-player bandit learning perspective with a focus on adaptive and robust aggregation of player-dependent heterogeneous feedback.
In \Cref{sec:relatedwork}, we further discuss and compare our problem formulation with related papers.

It is worth noting that naively utilizing data collected by other players may substantially hurt a player's regret \citep{zailn19}, if there are large disparities between the sources of feedback. This is also well-known as {\em negative transfer} in transfer learning~\citep{rosenstein2005transfer,brunskill2013sample}.

Therefore, the main challenge of the $\epsilon$-MPMAB problem is for the players to properly manage \textit{when and how} to utilize auxiliary data shared by others---while auxiliary data can be useful to maintain more accurate estimates of the rewards for each player and each arm, they can also easily be inefficacious or even misleading. While transfer learning in the offline setting has been well studied, in this paper we seek to characterize the difficulty of the more challenging problem of learning through heterogeneous feedback aggregation in a multi-player online setting.

We will first study the $\epsilon$-MPMAB problem when the dissimilarity parameter $\epsilon$ is known, and then move on to the harder setting in which $\epsilon$ is unknown.
Here is a summary of our main contributions:
\begin{itemize}
    \item We model online multitask bandit learning from heterogeneous data sources as the $\epsilon$-MPMAB problem, with a goal of studying how to adaptively and robustly aggregate data to improve the collective performance of the players. %
    
    \item In the setting where 
    $\epsilon$ is known,
    we propose an upper confidence bound (UCB)-based algorithm, $\robustagg(\epsilon)$, that adaptively aggregates rewards collected by different players.
    
    We provide (suboptimality)-{gap-dependent} and {gap-independent} upper bounds on the collective regret of $\robustagg(\epsilon)$. Our regret bounds depend on the set of arms that admit information sharing among the players. When this set is large, $\robustagg(\epsilon)$ can potentially improve the gap-dependent regret bound by nearly a factor of $M$ compared to the baseline of  players acting individually using UCB-1~\citep{acf02}.    
    
    We complement these upper bounds with nearly matching gap-dependent and gap-independent lower bounds. 

    \item In the setting where $\epsilon$ is unknown, we first establish a lower bound, showing that if an algorithm guarantees sublinear minimax regret with respect to all MPMAB instances, then it must be unable to significantly utilize inter-player similarity in a large collection of instances. 
    To complement the above result, we use the framework of Corral~\citep{alns17,pparzls20,amm20} and present an algorithm that trades off minimax regret guarantee for adaptivity to ``easy'' MPMAB problem instances.

\end{itemize}

\section{Problem Specification}
\label{sec:formulation}

We formulate the $\epsilon$-MPMAB problem, building on the standard model of stochastic multi-armed bandits \citep{lai1985asymptotically,acf02}.

Throughout, we denote by $[n] = \cbr{1,\ldots,n}$. An {\em MPMAB problem instance} consists of a set of $M$ players, labeled as elements in $[M]$, and a set of $K$ arms, labeled as elements in $[K]$. In addition, each player $p \in [M]$ and each arm $i \in [K]$ is associated with an unknown reward distribution $\mathcal{D}_i^p$ with support $[0,1]$ and mean $\mu_i^p$. If all $\Dcal_i^p$'s are Bernoulli distributions, we call this instance a {\em Bernoulli MPMAB problem instance}; under the Bernoulli reward assumption, $\mu = (\mu_i^p)_{i \in [K], p \in [M]}$ completely specifies the instance.

The reward distributions of the same arm are not necessarily identical for different players---we consider the following notion of dissimilarity between the reward distributions of the players. Related
conditions have been considered in works on multi-task bandit learning~\citep[e.g.,][]{glb13,soaremulti}.

\begin{definition}
\label{assumption:epsilon}
An MPMAB problem instance is said to be an $\epsilon$-MPMAB  problem instance, if
for every pair of players $p, q \in [M]$, $\max_{i \in [K]} |\mu^{p}_i - \mu^{q}_i| \le \epsilon$, where $\epsilon \in [0,1]$. We call $\epsilon$ the dissimilarity parameter.
\end{definition}

\paragraph{Interaction protocol.} 
Let $T > \max(M,K)$ be the horizon of an MPMAB ($\epsilon$-MPMAB) problem instance.
In each round $t \in [T]$, every player $p \in [M]$ 
pulls an arm $i^p_t$, and observes an independently-drawn reward $r^p_t \sim \mathcal{D}^p_{i_t^p}$.
Once all the $M$ players finish pulling arms in round $t$, each decision, $i^p_t$, together with the corresponding reward received, $r^p_t$, is immediately shared with all players.

\paragraph{Arm pulls, gaps, and performance measure.}
Let $\mu^p_* = \max_{i \in [K]} \mu^{p}_i$ be the optimal mean reward for every player $p \in [M]$.
Denote by $n^p_i(t)$ the number of pulls of arm $i$ by player $p$ after $t$ rounds, and $\Delta_i^p = \mu^p_* - \mu_i^p \ge 0$ the \textit{suboptimality gap} (abbrev. gap) between the means of the reward distributions associated with some optimal arm $i^p_*$ and arm $i$ for player $p$. 
For any arm $i \in [K]$, define $\Delta_i^{\min} = \min_{p \in [M]} \Delta_i^p$.
To measure the performance of MPMAB algorithms, we use the following notion of regret.
The expected regret of player $p$ is defined as $\EE[\mathcal{R}^p(T)] = \sum_{i\in [K]} \Delta_i^p \cdot \mathbb{E} [n^p_i(T)]$, and the players' {\em expected collective regret} is defined as $\EE[\mathcal{R}(T)] = \sum_{p \in [M]} \EE[\mathcal{R}^p(T)]$. 

\paragraph{Bandit learning algorithms.} A multi-player bandit learning algorithm $\Acal$ with horizon $T$ is defined as a sequence of conditional probability distributions $\cbr{\pi_t}_{t=1}^T$, where  for every $t$ in $[T]$, $\pi_t$ is the policy used in round $t$; specifically, $\pi_t(\cdot \mid (i_s^p, r_s^p)_{s \in [t-1], p \in [M]})$ is a conditional probability distribution of actions taken by all $M$ players in round $t$, given historical data. 
A bandit learning algorithm is said to have {\em sublinear regret} for the $\epsilon$-MPMAB (resp. MPMAB) problem, if there exists some $C>0$ and $\alpha>0$ such that $\EE[\Rcal(T)] \leq C T^{1-\alpha}$ for all $\epsilon$-MPMAB (resp. MPMAB) problem instances.

\paragraph{Miscellaneous notations.}
Throughout, we use $\tilde{O}$ notation to hide logarithmic factors. 
Given a universe set $\Hcal$ and any $\Jcal \subseteq \Hcal$, we use $\Jcal^C$ to denote the set $\Hcal \setminus \Jcal$.

\paragraph{Baseline: Individual UCB.} We now consider a baseline algorithm that runs the UCB-1 algorithm individually for each player without communication---hereafter, we refer to it as \inducb.
By \citep[Theorem 1]{acf02}, and summing over the individual regret guarantees of all players, the expected collective regret of \inducb satisfies 
$$\EE[\mathcal{R}(T)] \le  O\bigg(\sum_{i \in [K]} \sum_{p \in [M]: \Delta_i^p > 0} \frac{\ln T}{\Delta_i^p} \bigg).$$
In addition, \inducb has a gap-independent regret bound of $\tilde{\order}\del{M\sqrt{KT}}$ \cite[e.g.,][Theorem 7.2]{lattimore2020bandit}.

\subsection{Can auxiliary data always help?}
Since the interaction protocol allows information sharing among players, in any round $t > 1$, each player has access to more data than they would have without communication. Can the players always expect benefits from such auxiliary data and collectively perform better than \inducb? 

Below we provide an example that illustrates that the role of auxiliary data depends on the dissimilarities between the player-dependent reward distributions, as indicated by $\epsilon$, as well as the intrinsic difficulty of each multi-armed bandit problem each player faces individually, as indicated by the gaps $\Delta^p_i$'s.
Specifically, we show in the example that when $\epsilon$ is much larger than the gaps $\Delta^p_i$'s, any sublinear-regret bandit learning algorithm for the $\epsilon$-MPMAB problem cannot significantly take advantage of auxiliary data.

\begin{example}
\label{example:delta-small}
For a fixed $\epsilon \in (0,\frac1{8})$ and $\delta \le \epsilon/4$, consider the following Bernoulli MPMAB problem instance: for each $p \in [M]$, $\mu_1^p = \frac12 + \delta$, $\mu_2^p = \frac12$. 
This is a $0$-MPMAB instance, hence an $\epsilon$-MPMAB problem instance.
Also, note that $\epsilon$ is at least  four times larger than the gaps $\Delta^p_2 = \delta$.
\end{example}

\begin{claim}
\label{claim:delta-small}
For the above example, any sublinear regret algorithm for the $\epsilon$-MPMAB problem must have $\Omega(\frac{M \ln T}{\delta})$ regret on this instance, 
matching the \inducb regret upper bound.
\end{claim}

The claim follows from Theorem~\ref{thm:gap_dep_lb} in Section~\ref{sec:lb_known_eps};
see Appendix~\ref{appendix:claim2} for details.
The intuition is that 
any sublinear regret $\epsilon$-MPMAB algorithm must have $\Omega\del{\frac{\ln T}{\delta^2}}$ pulls of arm 2 from every player; otherwise, as $\delta$ is small compared to $\epsilon$, we can create a new $\epsilon$-MPMAB instance such that arm 2 is optimal for some player and is sufficiently indistinguishable from the original MPMAB problem, causing the algorithm to fail its sublinear regret guarantee.

Complementary to the above negative result, in the next section, we establish algorithms and sufficient conditions for the players to take advantage of the auxiliary data to achieve better regret guarantees.

\section{$\epsilon$-MPMAB with Known $\epsilon$}
\label{sec:known_eps}
In this section, we study the $\epsilon$-MPMAB problem with the dissimilarity parameter $\epsilon$ known to the players.
We first present our main algorithm $\robustagg(\epsilon)$ in Section~\ref{sec:mpmab_alg}; Section~\ref{sec:analysis} shows its regret guarantees; Finally, Section~\ref{sec:lb_known_eps} provides nearly matching regret lower bounds. Our proofs are deferred to Appendices~\ref{appendix:proof_of_fact},~\ref{appendix:upper_bounds_known_eps} and~\ref{appendix:lb}.

\subsection{Algorithm:
$\robustagg(\epsilon)$}
\label{sec:mpmab_alg}

\begin{algorithm*}[t]
\SetAlgoLined
\SetKwInOut{Input}{Input}
\SetKw{Return}{return}
 \Input{Distribution dissimilarity parameter $\epsilon \in [0,1]$\;}
 \textbf{Initialization:} Set $n^p_i = 0$ for all $p \in [M]$ and all $i \in [K]$.\\
 
 \For{$t = 1, 2 \ldots, T$}{
    \For{$p \in [M]$}{
        \For{$i \in [K]$}{
        
            Let $m^p_i = \sum_{q \in [M]: q \neq p} n^q_i$\;
            \label{line:ucb-start}
            
            Let $\wbar{n^p_i} = \max(1, n^p_i)$ and $\wbar{m^p_i} = \max(1, m^p_i)$\; 
                        
            Let
            $$\zeta_i^p(t) = \frac{1}{\wbar{n^p_i}} \sum_{\substack{s < t\\ i^p_s = i}} r^p_s,
            \ \eta_i^p(t) =  \frac{1}{\wbar{m^p_i}} \sum_{\substack{q \in [M] \\ q \neq p}}\sum_{\substack{s < t\\ i^q_s = i}} r^q_s, 
            \text{and}\ \kappa_i^p(t, \lambda) = \lambda \zeta_i^p(t) + (1-\lambda) \eta_i^p(t);
            $$
            \label{line:emp-mean}
            
            Let $F(\wbar{n^p_i}, \wbar{m^p_i}, \lambda, \epsilon ) = 8\sqrt{13\ln T\left[\frac{\lambda^2}{\wbar{n^p_i}} + \frac{(1-\lambda)^2}{\wbar{m^p_i}}\right]} + (1-\lambda)\epsilon$\; 
            \label{line:dev-bound}
            
            Compute $\lambda^* = \argmin_{\lambda \in [0,1]} F(\wbar{n^p_i},\wbar{m^p_i}, \lambda, \epsilon)$\;
            \label{line:opt-lambda}

            Compute an upper confidence bound of the reward of arm $i$ for player $p$:
            $$\UCB^p_i(t) = \kappa_i^p(t, \lambda^*) + F(\wbar{n^p_i}, \wbar{m^p_i}, \lambda^*, \epsilon).$$
            \label{line:ucb-end}
        }
        Let $i^p_t = \text{argmax}_{i \in [K]} \text{UCB}^p_i(t)$\;
        Player $p$ pulls arm $i^p_t$ and observes reward $r^p_t$\;
   }
   \For{$p \in [M]$}{
       Let $i = i^p_t$ and set $n^p_{i} = n^p_{i} + 1$.
   }
 }
 \caption{$\robustagg(\epsilon)$: Robust learning in $\epsilon$-MPMAB}
 \label{alg: mpmab}
\end{algorithm*}

We present $\robustagg(\epsilon)$, %
an algorithm that adaptively and robustly aggregates rewards collected by different players in $\epsilon$-MPMAB problem instances, given dissimilarity $\epsilon$ as an input parameter. 

Intuitively,
in any round, a player may decide to take advantage of data from other players who have similar reward distributions.
Deciding how to use auxiliary data is tricky---on the one hand, they can help reduce variance and get a better mean reward estimate, but on the other hand, if the dissimilarity between players' reward distributions is large, auxiliary data can substantially bias the estimate. Our algorithm is built upon this insight of balancing bias and variance. A similar tradeoff in offline transfer learning for classification is studied in the work of \citet{bbckpv10}; we discuss %
the connection and differences between our work and theirs in Section~\ref{sec:weighting_samples}.

Algorithm \ref{alg: mpmab} provides a pseudocode of $\robustagg(\epsilon)$.
Specifically, it builds on the classic UCB-$1$ algorithm \citep{acf02}: for each player $p$ and arm $i$, it maintains an upper confidence bound $\UCB_i^p(t)$ for mean reward $\mu^p_i$ over time (lines~\ref{line:ucb-start} to~\ref{line:ucb-end}), such that with high probability, $\mu_i^p \leq \UCB_i^p(t)$, for all $t$.

To achieve the best regret guarantees, we would like our confidence bounds on $\mu_i^p$ to be as tight as possible. To this end, we consider a family of confidence intervals for $\mu_i^p$, parameterized by a weighting factor $\lambda \in [0,1]$: $[\kappa_i^p(t, \lambda) \pm F(\wbar{n^p_i}, \wbar{m^p_i}, \lambda, \epsilon)]$.

In the above confidence interval formula, $\kappa_i^p(t, \lambda)$ estimates $\mu_i^p$ by taking a convex combination of $\xi_i^p(t)$ and $\eta_i^p(t)$, the empirical mean reward of arm $i$ based on the player's own samples and the auxiliary samples, respectively  (line~\ref{line:emp-mean}). 
The width $F(\wbar{n^p_i}, \wbar{m^p_i}, \lambda, \epsilon)$ is a high-probability upper bound on $\abs{\kappa_i^p(t, \lambda)-\mu_i^p}$ (line~\ref{line:dev-bound}). Varying $\lambda$ reveals the aforementioned bias-variance tradeoff: the first term, $8\sqrt{13\ln T [\frac{\lambda^2}{\wbar{n^p_i}} + \frac{(1-\lambda)^2}{\wbar{m^p_i}}]}$, is a high probability upper bound on the deviation of $\kappa_i^p(t, \lambda)$ from its expectation $\EE[\kappa_i^p(t, \lambda)]$; the second term, $(1-\lambda)\epsilon$, is an upper bound on the difference between $\EE[\kappa_i^p(t, \lambda)]$ and $\mu_i^p$.
We choose $\lambda^* \in [0,1]$ to minimize the width of our confidence interval for $\mu_i^p$  (line~\ref{line:opt-lambda}), similar to the calculation in \citep[Section 6]{bbckpv10}.\footnote{See Appendix~\ref{appendix:lambda_star} for an analytical solution to the optimal weighting factor $\lambda^*$.}

\subsection{Regret analysis}
\label{sec:analysis}

We first define the notion of \textit{subpar arms}. %
Let $$\Ical_\epsilon = \cbr{i: \exists p \in [M], \mu_*^p - \mu_i^p > 5\epsilon}$$ %
be the set of subpar arms for an $\epsilon$-MPMAB problem instance. Intuitively, $\Ical_\epsilon$ contains the set of ``easier'' arms for which data aggregation between players can be \textit{effective}.
For each arm $i \in \Ical_\epsilon$, the following fact shows that the gap $\Delta_i^p = \mu_*^p - \mu_i^p$ is sufficiently larger than the dissimilarity parameter $\epsilon$ for all players $p \in [M]$. This allows $\robustagg(\epsilon)$ to exploit the ``easiness'' of these arms through data aggregation across players, thereby reducing avoidable individual explorations.

\begin{fact}
\label{fact:optimal_arm}
$\abs{\Ical_\epsilon} \leq K-1$. In addition, for each arm $i \in \Ical_\epsilon$, $\Delta_i^{\min} > 3\epsilon$; in other words,  for all players $p$ in $[M]$, $\Delta_i^p = \mu_*^p - \mu_i^p > 3\epsilon$; consequently, arm $i$ is suboptimal for all players $p$ in $[M]$. 
\end{fact}

We now present regret guarantees of $\robustagg(\epsilon)$. %

\begin{theorem}
\label{thm:gap_dep_ub}
Let $\robustagg(\epsilon)$ run on an $\epsilon$-MPMAB problem instance for $T$ rounds. 
Then, its expected collective regret satisfies
\begin{align*}
\mathbb{E}[\mathcal{R}(T)] \le \order \Bigg(
& \sum_{i \in \Ical_\epsilon} \del{ \frac{\ln T}{\Delta_i^{\min}} + M \Delta_i^{\min}} + \\
& \sum_{i \in \Ical^C_\epsilon}
\sum_{p \in [M]: \Delta_i^p > 0}
\frac{\ln T}{\Delta_i^p}
 \Bigg).
\end{align*}
\end{theorem}

The first term in the above bound shows that the collective regret incurred by the players for the subpar arms $\Ical_\epsilon$ and the second term for arms in $\Ical_\epsilon^C = [K] \setminus \Ical_\epsilon$. Observe that for each subpar arm, the regret of the players \textit{as a group} can be upper-bounded by $O\rbr{\frac{\ln T}{\Delta^{\min}_i} + M \Delta^{\min}_i}$, whereas for each arm in $\Ical^C_\epsilon$, the regret on \textit{each} player is $O({\frac{\ln T}{\Delta^{p}_i}})$ unless $\Delta^p_i = 0$.

\begin{fact}
\label{fact:i-eps-inv-gap-short}
For any $i \in \Ical_{\epsilon}$, 
$\frac{1}{\Delta_i^{\min}} \leq \frac{2}{M} \sum_{p \in [M]} \frac{1}{\Delta_i^p}$.
\end{fact}

\paragraph{Fallback guarantee.} The regret guarantee of $\robustagg(\epsilon)$ by \Cref{thm:gap_dep_ub} is always no worse than that of \inducb by a constant factor, as from Fact~\ref{fact:i-eps-inv-gap-short}, for all $i$ in $\Ical_{\epsilon}$, $\frac{\ln T}{\Delta_i^{\min}} + M \Delta_i^{\min} = O\del{\sum_{p \in [M]} \frac{\ln T}{\Delta_i^p}}$.

\paragraph{Two extreme cases of $\abs{\mathcal{I}_\epsilon}$.} If $\mathcal{I}_\epsilon = \O$, in which case we do not expect data aggregation across players to be beneficial, 
the above bound can be simplified to:
\begin{align*}
\mathbb{E}[\mathcal{R}(T)] \le
\order\rbr{
\sum_{i \in [K]}
\sum_{p \in [M]: \Delta_i^p > 0}
\frac{\ln T}{\Delta_i^p}}. \label{eq:Ical_empty}
\end{align*}

In contrast, when $\Ical_\epsilon$ has a larger size, namely, more arms admit data aggregation across players, $\robustagg(\epsilon)$ has an improved regret bound. The following corollary gives regret bounds in the most favorable case when $\Ical_\epsilon$ has size $K-1$. It is not hard to see that, in this case, $\Ical^C_\epsilon$ is equal to a singleton set $\cbr{i_*}$, where arm $i_*$ is optimal for all players $p$.

\begin{corollary}
\label{cor:Ical0}
Let $\robustagg(\epsilon)$ run on an $\epsilon$-MPMAB problem instance with $|\mathcal{I}_\epsilon| = K - 1$ for $T$ rounds. 
Then, its expected collective regret satisfies
\[  
\mathbb{E}[\mathcal{R}(T)] \le \order\rbr{
\sum_{i \neq i_*} \frac{\ln T}{\Delta_i^{\min}}
+
M \sum_{i \neq i_*} \Delta_i^{\min}}.
\]
\end{corollary}

It can be observed that, compared to the \inducb baseline, under the assumption that $|\mathcal{I}_\epsilon| = K - 1$, $\robustagg(\epsilon)$ improves the regret bound by nearly a factor of $M$: if we set aside the $O\del{M \sum_{i \neq i_*} \Delta_i^{\min} }$ term, which is of lower order than the rest under the mild assumption that $M = O\del{ \min_{i \neq i^*} \frac{\ln T}{(\Delta_i^{\min})^2} }$, then the expected collective regret in Corollary \ref{cor:Ical0} is a factor of $O(\frac{1}{M})$ times that of \inducb, in light of Fact~\ref{fact:i-eps-inv-gap-short}.

\paragraph{Gap-independent upper bound.}
We now provide an upper bound on the expected collective regret that is independent of the gaps $\Delta^p_i$'s.
\begin{theorem}
\label{thm:gap_ind_upper}
Let $\robustagg(\epsilon)$ run on an $\epsilon$-MPMAB problem instance for $T$ rounds. Then its expected collective regret satisfies
\[ 
\EE[\mathcal{R}(T)] \le
\tilde{\order}\rbr{
\sqrt{\abr{\Ical_\epsilon} M T}
+ 
M \sqrt{(\abr{\Ical^C_\epsilon}-1) T}
+ 
M \abr{\Ical_\epsilon}
}.
\]
\end{theorem}

Recall that \inducb has a gap-independent bound of $\tilde{\order}\del{M\sqrt{KT}}$.
By algebraic calculations, we can see that when $T = \Omega(KM)$, the regret bound of $\robustagg(\epsilon)$ is a factor of $O\rbr{\max\del{\sqrt{\frac{\abs{\Ical_\epsilon^C}-1}{K}}, \sqrt{\frac1M}} }$ times \inducb's regret bound. Therefore, when $M = \omega(1)$ and $\abs{\Ical_\epsilon^C} = o(K)$, i.e., when there is a large number of players, and an overwhelming portion of subpar arms,
\robustagg has a gap-independent regret bound of strictly lower order than \inducb.

Observe that the above bound has a term $M \sqrt{(\abr{\Ical^C_\epsilon}-1) T}$ with a peculiar dependence on $\abr{\Ical^C_\epsilon}-1$; this is due to the fact that in the special case of $\abr{\Ical_\epsilon} = K-1$, i.e., $\abr{\Ical^C_\epsilon} = 1$, the contribution to the regret from arms in $\Ical_\epsilon^C$ is zero. Indeed, in this case, $\Ical_\epsilon^C$ is a singleton set $\cbr{i_*}$, where arm $i_*$ is optimal for all players.

\subsection{Lower bounds}
\label{sec:lb_known_eps}

\paragraph{Gap-dependent lower bound.}
To complement our gap-dependent upper bound in \Cref{thm:gap_dep_ub}, we now present a gap-dependent lower bound. We show that, for any fixed $\epsilon$, any sublinear regret algorithm for the $\epsilon$-MPMAB problem must have regret guarantees not much better than that of $\robustagg(\epsilon)$ for a large family of $\frac{\epsilon}{2}$-MPMAB problem instances.

\begin{theorem}
\label{thm:gap_dep_lb}
Fix $\epsilon \geq 0$. Let $\Acal$ be an algorithm and $C > 0, \alpha > 0$ be constants, such that $\Acal$ has $C T^{1-\alpha}$ regret in all $\epsilon$-MPMAB environments. 
Then, for any Bernoulli $\frac\epsilon2$-MPMAB instance $\mu = (\mu_i^p)_{i \in [K], p \in [M]}$ such that $\mu_i^p \in [\frac {15}{32}, \frac{17}{32}]$ for all $i$ and $p$, we have:
\begin{align*}
\EE_\mu[\Rcal(T)]
\geq 
\Omega & \left(
\sum_{i \in \Ical_{\epsilon/20}^C}
\sum_{p \in [M]: \Delta_i^p > 0} 
\frac{\ln( \Delta_i^p T^\alpha / C) }{ \Delta_i^p} 
\right.
\\ 
& +
\left.
\sum_{i \in \Ical_{\epsilon / {20}}: \Delta_i^{\min} > 0}
\frac{\ln (\Delta_i^{\min} T^\alpha / C)}{ \Delta_i^{\min}} 
\right)
.
\end{align*}
\end{theorem}

\Cref{thm:gap_dep_lb} is nearly tight compared with the upper bound presented in \Cref{thm:gap_dep_ub} with two differences. First, the upper bound is in terms of $\Ical_{\epsilon}$, while the lower bound is in terms of $\Ical_{\epsilon/20}$; we leave the possibility of exploiting data aggregation for arms in $\Ical_{\epsilon} \setminus \Ical_{\epsilon/20}$ as an open question. 
Second, the upper bound has an extra   $O(\sum_{i \in \Ical_{\epsilon}} M\Delta_i^{\min} )$ term, caused by the players issuing arm pulls in parallel in each round; we conjecture that it may be possible to remove this term by developing more efficient multi-player exploration strategies.

\paragraph{Gap-independent lower bound.} The following theorem shows that, 
there exists a value of $\epsilon$ (that depends on $T$ and $\abs{\Ical_{\epsilon}}$), such that any algorithm must have a minimax collective regret not much lower than the upper bound shown in Theorem~\ref{thm:gap_ind_upper} in the family of all $\epsilon$-MPMAB problems.

\begin{theorem}
\label{thm:gap_ind_lower}
For any $K \geq 2, M, T \in \NN$, and $l, l^C$ in $\NN$ such that $l \leq K-1, l + l^C = K$, there exists some $\epsilon > 0$, such that for any algorithm $\Acal$, there exists an $\epsilon$-MPMAB problem instance, in which 
$\abs{\Ical_\epsilon} \geq l$, and $\Acal$ has a collective regret at least
$\Omega(M\sqrt{(l^C-1) T} + \sqrt{M l T})$.
\end{theorem}

The above lower bound is nearly tight in light of the upper bound in Theorem~\ref{thm:gap_ind_upper}: as long as $T = \Omega(K M)$, the upper and lower bounds match within a constant.

\section{$\epsilon$-MPMAB with Unknown $\epsilon$}
\label{sec:unknown_eps}
We now turn to the setting when $\epsilon$ is unknown to the learner. 
Unlike the $\robustagg(\epsilon)$ algorithm developed in the last section, which only has nontrivial regret guarantees for all $\epsilon$-MPMAB instances, in this section,
we aim to design algorithms that have nontrivial regret guarantees for all MPMAB instances.

Recall that $\robustagg(\epsilon)$ relies on the knowledge of $\epsilon$ to construct reward confidence intervals for each arm and player;  
when $\epsilon$ is unknown, constructing such confidence interval becomes a big challenge.
In Appendix~\ref{sec:adaptive-ci}, we give evidence showing that it may be impossible to design confidence interval-based algorithms that significantly benefit from inter-player information sharing. This suggests that new algorithmic ideas seem necessary to obtain nontrivial results in this setting.

\subsection{Gap-dependent lower bound}
Recall that \inducb achieves a gap-dependent regret bound of $O\del{\sum_{i \in [K]} \sum_{p \in [M]: \Delta_i^p > 0} \frac{\ln T}{\Delta_i^p}}$ for all MPMAB problems without knowing $\epsilon$. 
Interestingly, we show in the following theorem that any sublinear regret algorithm for the MPMAB problem must have gap-dependent lower bound not much better than \inducb for a large family of MPMAB problem instances, regardless of the value of $\epsilon$ and the size of $\Ical_\epsilon$ of that instance.

\begin{theorem}
\label{thm:gap_dep_lb_unk}
Let $\mathcal{A}$ be an algorithm and $C > 0, \alpha > 0$ be constants such that $\Acal$ has $CT^{1-\alpha}$ regret in all MPMAB problem instances. Then, for any Bernoulli MPMAB instance $\mu = (\mu_i^p)_{i \in [K], p \in [M]}$ such that $\mu_i^p \in [\frac{15}{32}, \frac{17}{32}]$ for all $i \in [K], p \in [M]$,
$$ \mathbb{E}_\mu [\Rcal(T)] \ge \Omega\del{ \sum_{i \in [K]} \sum_{p \in [M]: \Delta_i^p > 0} \frac{\ln (T^\alpha \Delta_i^p / C ) }{ \Delta_i^p} }.$$
\end{theorem}

\subsection{Gap-independent upper bound}
\label{sec:gap_ind_upper}

While we have shown gap-dependent lower bounds that nearly matches the upper bounds for \inducb for sublinear regret MPMAB algorithms in Theorem~\ref{thm:gap_dep_lb_unk}, this does not rule out the possibility of achieving regret that improves upon \inducb in small-gap instances.
To see this, note that if $\Delta_i^p$ is of order $O(T^{-\alpha})$ for all $i$ in $[K]$ and $p$ in $[M]$, the above lower bound becomes vacuous. Therefore, it is still possible to get gap-independent upper bounds that improve over  the $\tilde{O}(M\sqrt{KT})$ upper bound by \inducb.

We present \arobustagg in Appendix~\ref{appendix:ub_unk_eps}, an algorithm that achieves such guarantee: specifically, it achieves a gap-independent regret upper bound adaptive to $\abs{\Ical_{2\epsilon}}$.
In a nutshell, the algorithm aggregates over a set of $\robustagg(\epsilon)$ base learners with different values of $\epsilon$, using the strategy of Corral~\citep{alns17}. We have the following theorem:

\begin{theorem}
\label{thm:gap_ind_upper_unk}
Let \arobustagg run on an $\epsilon$-MPMAB problem instance with any $\epsilon \in [0,1]$.
Its expected collective regret in a horizon of $T$ rounds satisfies
$$
\mathbb{E}[\mathcal{R}(T)] \le \tilde{O}\left( \left( \abs{\Ical_{2\epsilon}} + M \abr{\Ical_{2\epsilon}^C} \right) \sqrt{T} + M |\Ical_{\epsilon}| \right).
$$
\end{theorem}

Under the mild assumption that $T = \Omega(\min(K^2, M^2))$, the above regret bound becomes $\tilde{O}\Bigg( \Big( \abs{\Ical_{2\epsilon}} + M \abr{\Ical_{2\epsilon}^C} \Big) \sqrt{T} \Bigg)$. 
If  furthermore $\abs{\Ical_{2\epsilon}} = K - o(\sqrt{K})$ 
and $M = \omega(\sqrt{K})$,
the regret bound of \arobustagg is of lower order than \inducb's $\tilde{O}(M\sqrt{KT})$ regret guarantee. In the most favorable case when $|\Ical_{2\epsilon}| = K - 1$, \arobustagg has expected collective regret $\tilde{O}\left((M + K) \sqrt{T}\right)$.

Such adaptivity of \arobustagg to unknown similarity structure comes at a price of higher minimax regret guarantee:
when $\Ical_\epsilon = \O$, \arobustagg has a regret of
$ 
\tilde{O}\left(MK\sqrt{T}\right)
$, a factor of $\sqrt{K}$ higher than $\tilde{O}(M\sqrt{KT})$, the worst-case regret of \inducb. We conjecture that this may be unavoidable due to lack of knowledge of $\epsilon$,
similar to results in adaptive Lipschitz bandits~\citep{locatelli2018adaptivity,krishnamurthy2019contextual,hadiji2019polynomial}.

\section{Related Work and Comparisons}
\label{sec:relatedwork}
\subsection{Multi-agent bandits.} 
We first compare existing multi-agent bandit learning problems with the $\epsilon$-MPMAB problem. We provide a more detailed review of the literature in Appendix~\ref{appendix:related_work}.

A large portion of prior studies \citep{kpc11,sbhojk13,lsl16,cgm19,kjg18,sgs19,whcw19,dp20a,csgs20,wpajr20} focuses on the setting where a set of players collaboratively work on one bandit learning problem instance, i.e., the reward distributions of an arm are identical across all players. In contrast, we study multi-agent bandit learning where the reward distributions across players can be different. 

Multi-agent bandit learning with heterogeneous feedback has also been covered by previous studies. In \citep{srj17}, a group of players seek to find the arm with the largest average reward over all players; however, in each round, the players have to reach a consensus and choose the same arm. 
\citet{cgz13} study a network of linear contextual bandit players with heterogeneous rewards, where the players can take advantage of reward similarities hinted by a graph. They use a Laplacian-based regularization, whereas we study when and how to use information from other players based on a dissimilarity parameter.
\citet{glz14,lkg16} assume that the players' reward distributions have a cluster structure; in addition, players that belong to one cluster share a \textit{common} reward distribution; our paper do not assume such cluster structure.
\citet{dp20b} assume access to some side information for every player, and learns a 
reward predictor that takes both player's side information 
models and action as input. 
In comparison, our work do not assume access to such side information.

Similarities in reward distributions are explored in \citep{sj12,zailn19} to warm start bandit learning agents.
\citet{glb13,soaremulti} investigate multitask learning in bandits through \textit{sequential transfer} between tasks that have similar reward distributions. In contrast, we study the multi-player setting, where all players learn continually and concurrently.

There are other practical formulations of multi-player bandits with player-dependent reward distributions~\citep{bblb20,bkmp20}, where the existence of collision is assumed; i.e., two players pulling the same arm in the same round receive zero reward. In comparison, collision is not modeled in this paper.

\subsection{Learning using weighted data aggregation}
\label{sec:weighting_samples}
Our design of confidence interval in Section~\ref{sec:mpmab_alg} has resemblance to the weighted empirical risk minimization algorithm proposed for domain adaptation by \citet{bbckpv10}, but our purposes are different from theirs.
Specifically, our choice of $\lambda$ minimizes the length of the confidence intervals, whereas 
\cite{bbckpv10} find $\lambda$ that minimizes classification error in the target domain.
Furthermore, our setting in Section~\ref{sec:unknown_eps} is more challenging: in offline domain adaptation, one may use a validation set drawn from the target domain to fine-tune the optimal weight $\lambda^*$, to adapt to unknown dissimilarity between the source and the target; however, in our setting (and online bandit learning in general), such tuning does not result in sample efficiency improvement. 

The idea of assigning weights to different sources of samples %
has also been studied by
\citet{zailn19} for warm starting contextual bandit learning from misaligned distributions and by \citet{rvc19} for online learning in non-stationary environments. 
\citet{zhx20} use a weighted compound of player-based estimator and cluster-based estimator for collaborative Thompson sampling, where the weights are given by a hyper-parameter; in contrast, we adaptively compute our weighting factor based on the numbers of samples collected by the players as well as the dissimilarity parameter $\epsilon$.

\begin{figure*}[t]
    \centering
    \begin{subfigure}{.3\textwidth}
        \centering
        \includegraphics[height=0.75\linewidth]{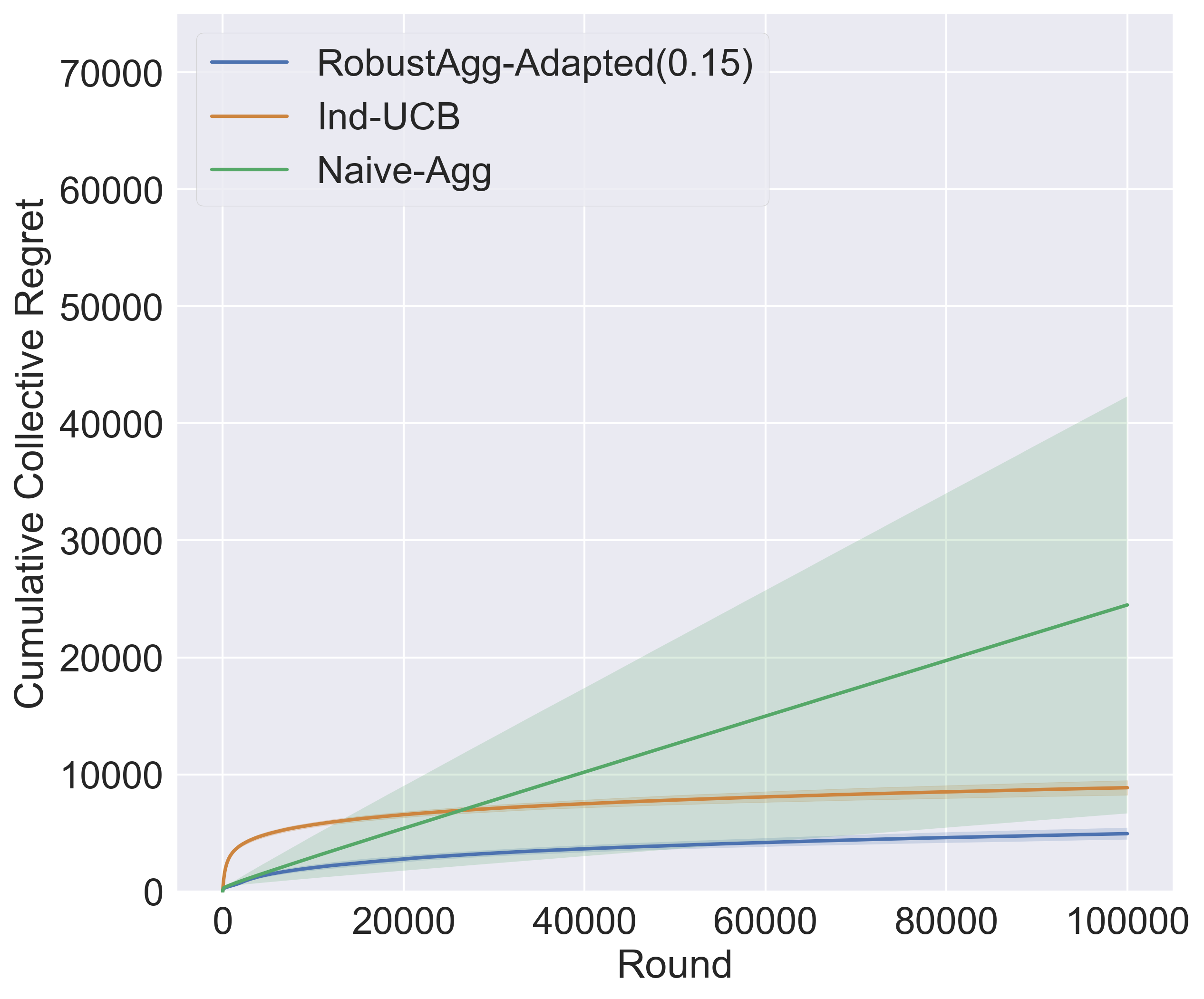}
        \caption{$|\Ical_\epsilon| = 8$}
        \label{figure:exp1a}
    \end{subfigure}
    \begin{subfigure}{0.3\textwidth}
        \centering
        \includegraphics[height=0.75\linewidth]{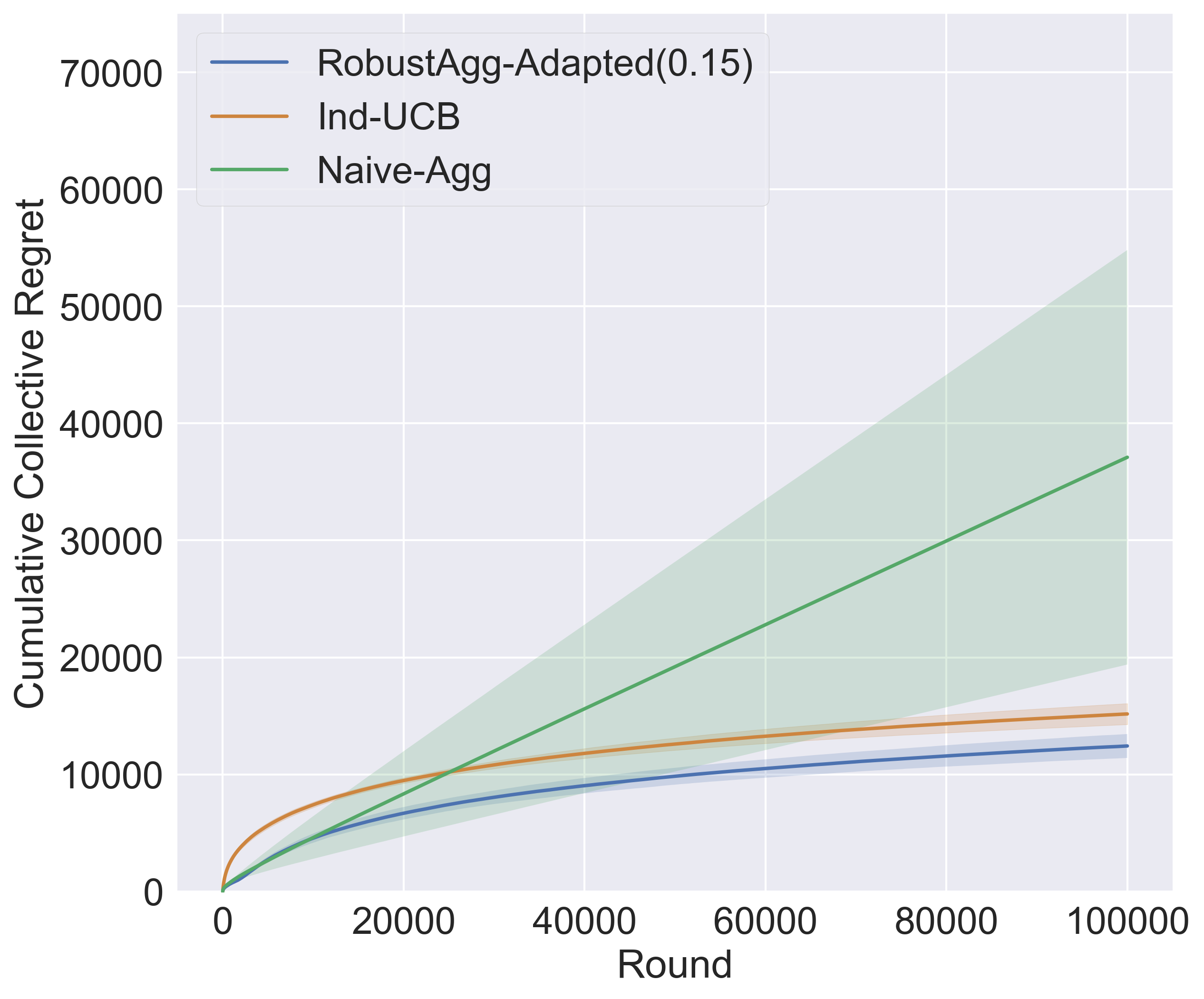}
        \caption{$|\Ical_\epsilon| = 6$}
        \label{figure:exp1b}
    \end{subfigure}
    \begin{subfigure}{0.3\textwidth}
        \centering
        \includegraphics[height=0.75\linewidth]{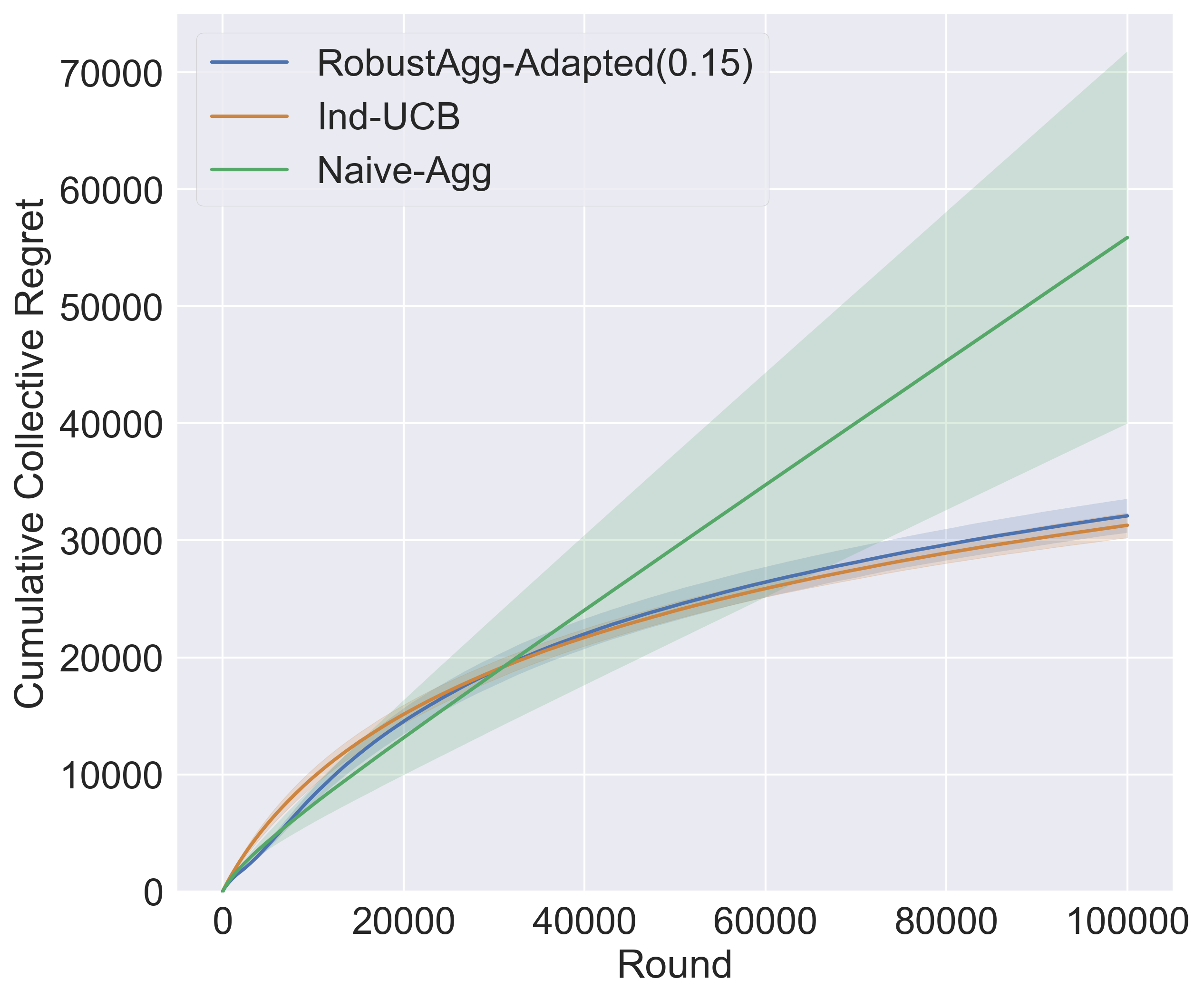}
        \caption{$|\Ical_\epsilon| = 0$}
        \label{figure:exp1c}
    \end{subfigure}
    \caption{Compares the average performance of $\adaptedrobustagg(0.15)$, \inducb, and \naiveagg in randomly generated Bernoulli $0.15$-MPMAB problem instances with $K = 10$ and $M = 20$. The $x$-axis shows a horizon of $T = 100,000$ rounds, and the $y$-axis shows the cumulative collective regret of the players.
    }
    \label{figure:experiment1}
\end{figure*}

\begin{figure*}[t]
    \centering
    \begin{subfigure}{.3\textwidth}
        \centering
        \includegraphics[height=0.75\linewidth]{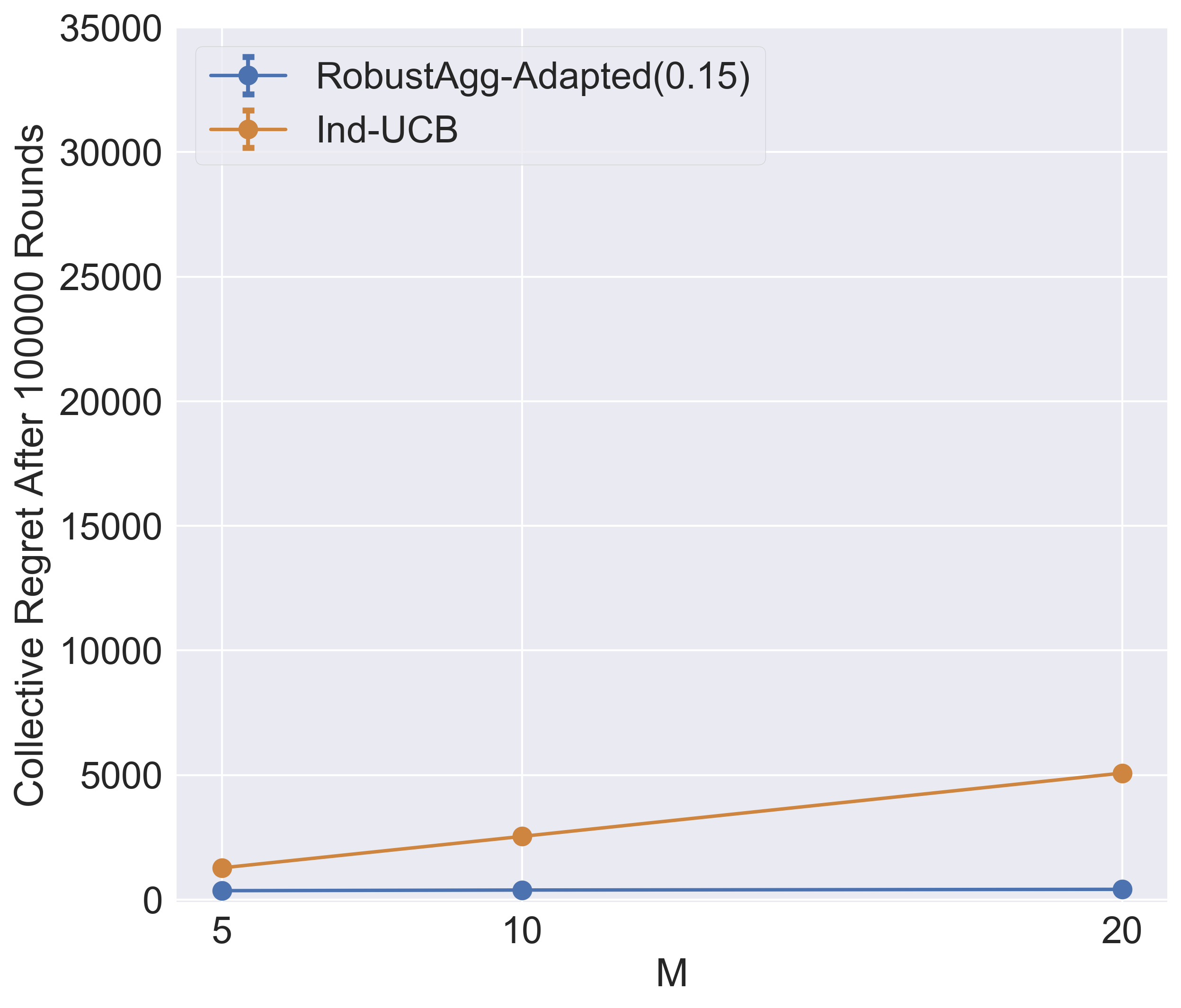}
        \caption{$|\Ical_\epsilon| = 9$}
        \label{figure:exp2a}
    \end{subfigure}
    \begin{subfigure}{0.3\textwidth}
        \centering
        \includegraphics[height=0.75\linewidth]{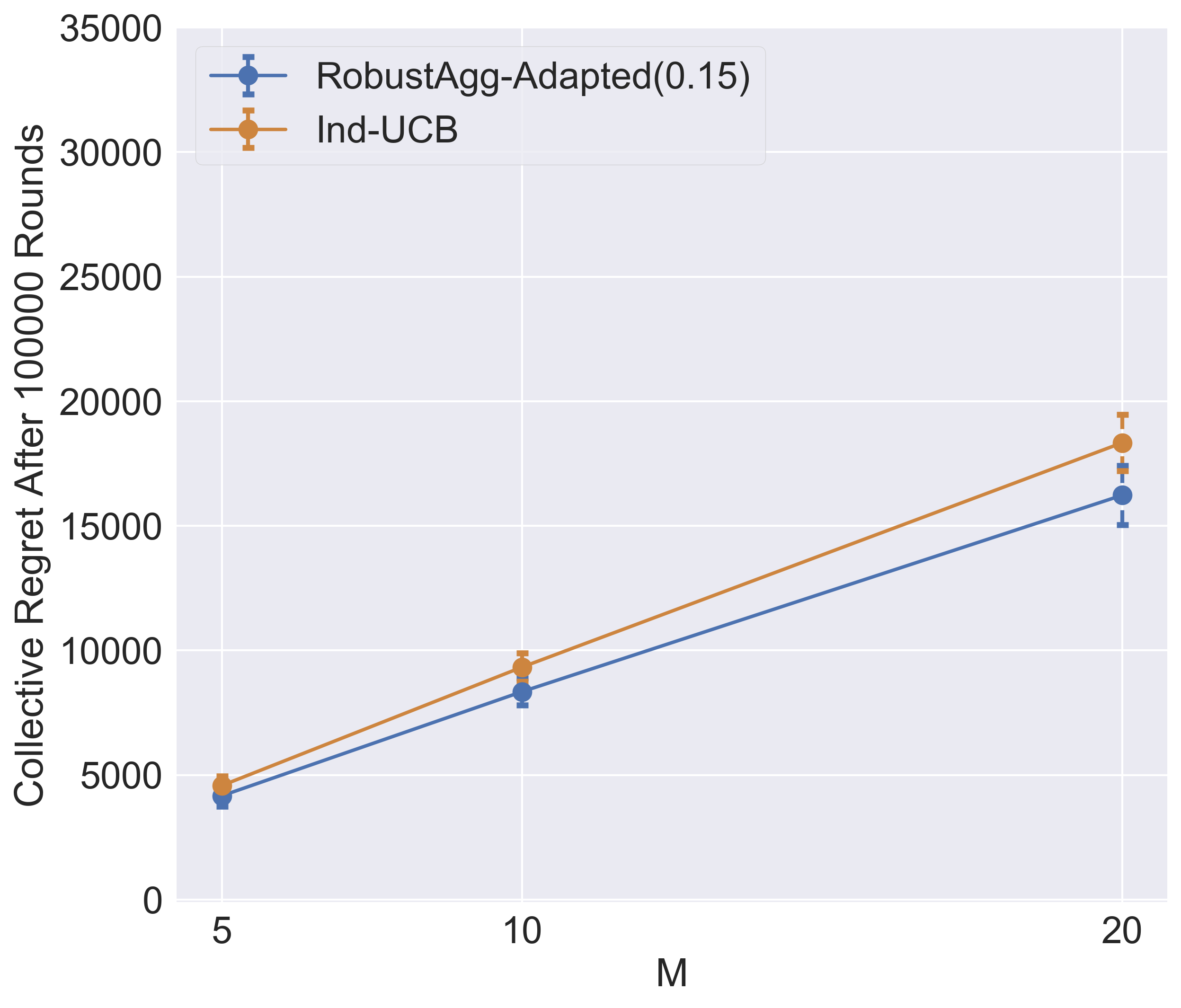}
        \caption{$|\Ical_\epsilon| = 5$}
        \label{figure:exp2b}
    \end{subfigure}
        \begin{subfigure}{0.3\textwidth}
        \centering
        \includegraphics[height=0.75\linewidth]{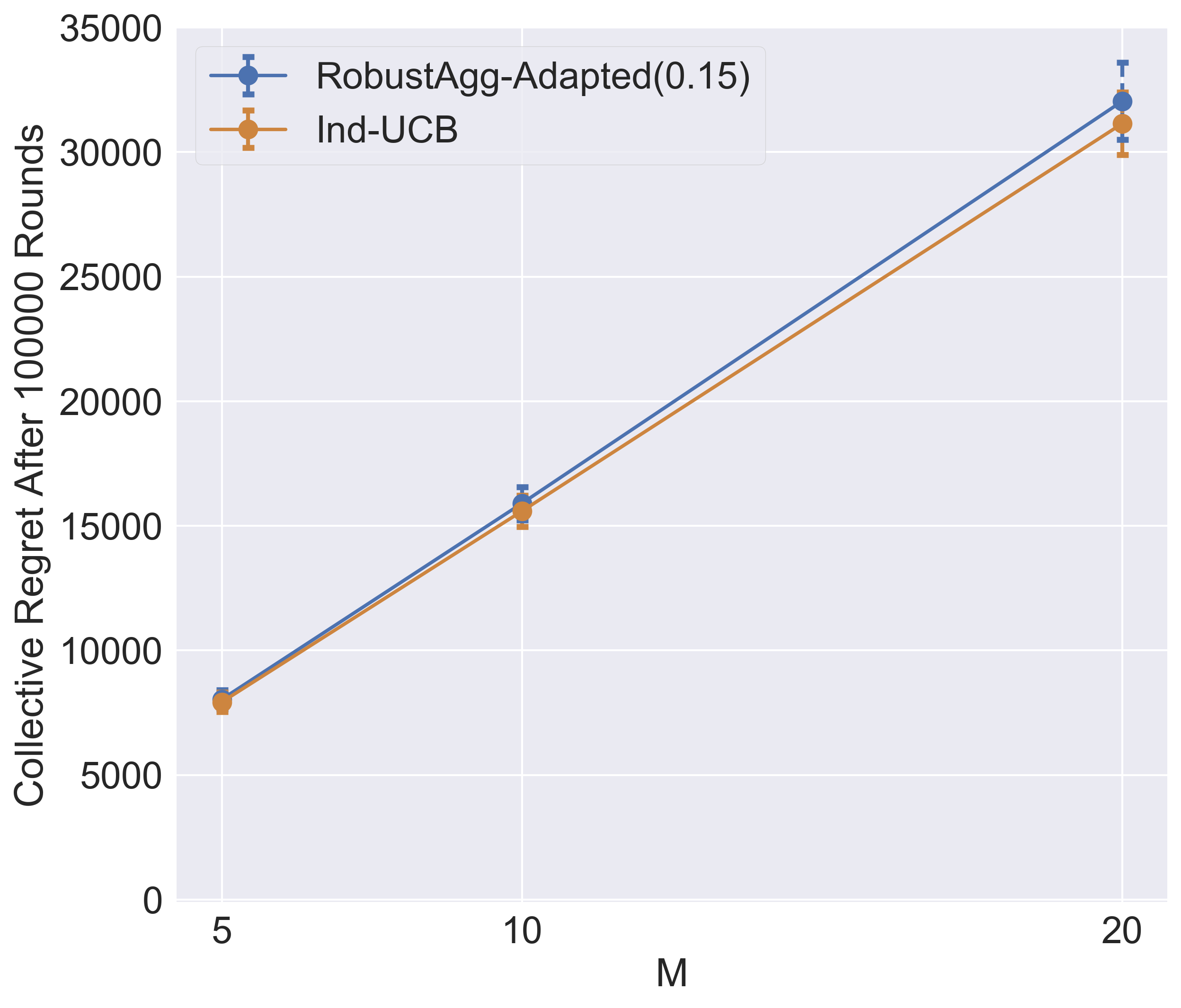}
        \caption{$|\Ical_\epsilon| = 0$}
        \label{figure:exp2c}
    \end{subfigure}
    \caption{Compares the average performance of $\adaptedrobustagg(0.15)$ and \inducb in randomly generated Bernoulli $0.15$-MPMAB problem instances with $K = 10$. 
    The $x$-axis shows different values of $M$, and the $y$-axis shows the cumulative collective regret of the players after $100,000$ rounds.}
    \label{figure:experiment2}
\end{figure*}

\section{Empirical Validation}
\label{sec:experiments}
We now validate our theoretical results with some empirical simulations using synthetic data\footnote{Our code is available at \url{https://github.com/zhiwang123/eps-MPMAB}.}. Specifically, we seek to answer the following questions:
\begin{enumerate}
    \item In practice, how does our proposed algorithm compare with algorithms that either do not take advantage of adaptive data aggregation or do not execute aggregation in a robust fashion?
    \item How does the performance of our algorithm change with different numbers of subpar arms?
\end{enumerate}
We note that these questions are considered in the setting where the dissimilarity parameter $\epsilon$ is known to the algorithms. 

\subsection{Experimental setup}
We first describe the algorithms compared in the simulations. We then discuss the procedure we used for generating synthetic data.

\paragraph{$\adaptedrobustagg(\epsilon)$.} 
Since standard concentration bounds are loose in practice, we performed simulations on a more practical and aggressive variant of $\robustagg(\epsilon)$, which we call $\adaptedrobustagg(\epsilon)$. Specifically, we changed the constant coefficient $8\sqrt{13}$ to $\sqrt{2}$ in the UCBs; this constant was taken from the original UCB-$1$ algorithm \citep{acf02}, which is an ingredient of the baseline \inducb, and we simply kept the default value.

\paragraph{Baselines.}
We evaluate the following two algorithms as baselines: (a) \inducb, described in \Cref{sec:formulation}; and (b) \naiveagg, in which the players \textit{naively} aggregate data assuming that their reward distributions are identical---in other words, \naiveagg is equivalent to $\adaptedrobustagg(0)$. 

\paragraph{Instance generation.} %
We generated problem instances using the following \textit{randomized} procedure. We first set $\epsilon = 0.15$. Then, given the number of players $M$, the number of arms $K$, and the number of subpar arms $|\Ical_\epsilon| \in \{0,1,\ldots,K-1\}$, we first sampled the means of the reward distributions for player $1$:

Let $c = K - |\Ical_\epsilon|$. For $i \in \{1, 2, \ldots, c\}$, we sampled $\mu^1_i \overset{\mathrm{i.i.d.}}{\sim} \mathcal{U}[0.8,0.8 + \epsilon)$, where $\mathcal{U}[a,b)$ is the uniform distribution with support $[a,b]$. Let $d = \max_{i \in [c]} \mu_i^1$.
Then, for $i \in \{c+1, \ldots, K\}$, we sampled $\mu^1_i \overset{\mathrm{i.i.d.}}{\sim} \mathcal{U}[0, d - 5\epsilon)$.

We then sampled the means of the reward distributions for players $p \in \{2, \ldots, M\}$: For each $i \in [K]$, we sampled $\mu^p_i \overset{\mathrm{i.i.d.}}{\sim} \mathcal{U}\big[\max(0,\mu^1_i - \frac{\epsilon}{2}),\min(\mu^1_i + \frac{\epsilon}{2},1)\big)$.

\begin{fact}
The above construction gives a Bernoulli $0.15$-MPMAB problem instance that has exactly $(K - c)$ subpar arms, namely, $\Ical_\epsilon = \{i: c+1 \le i \le K\}$.
\label{fact:synthetic}
\end{fact}

\subsection{Simulations and results}
We ran two sets of simulations, and the results are shown in \Cref{figure:experiment1} and \Cref{figure:experiment2}. More detailed results are deferred to Appendix~\ref{appendix:experiments}.

\paragraph{Experiment 1.} We compare the cumulative collective regrets of the three algorithms in problem instances with \textit{different numbers of subpar arms}. We set $M = 20$, $K = 10$ and $\epsilon = 0.15$. For each $v \in \{0, 1, 2, \ldots, 9\}$, we generated $30$ Bernoulli $0.15$-MPMAB problem instances, each of which has exactly $v$ subpar arms, i.e., we generated instances with $|\Ical_\epsilon| = v$. Figures \ref{figure:exp1a}, \ref{figure:exp1b} and \ref{figure:exp1c} show the average regrets in a horizon of $100,000$ rounds over these generated instances, in which $|\Ical_\epsilon| = 8, 6$ and $0$, respectively. In the interest of space, figures in which $|\Ical_\epsilon|$ takes other values are deferred to Appendix~\ref{app:more_exp_results}.

Notice that $\adaptedrobustagg(0.15)$ outperforms both baseline algorithms in Figures \ref{figure:exp1a} and \ref{figure:exp1b} when $|\Ical_\epsilon| = 8$ and $6$. \Cref{figure:exp1c} demonstrates that when $|\Ical_\epsilon| = 0$, i.e., when there is no arm that is amenable to data aggregation, the performance of $\adaptedrobustagg(0.15)$ is still on par with that of \inducb. Also, as shown in \Cref{figure:exp1a}, even when $|\Ical_\epsilon^C| = 2$, i.e., when there are only two ``competitive'' (not subpar) arms, the collective regret of \naiveagg can still easily be nearly linear in the number of rounds.

\paragraph{Experiment 2.}
We study how the collective regrets of $\adaptedrobustagg(0.15)$ and \inducb scale with the \textit{number of players} in problem instances with different numbers of subpar arms.
We set $K = 10$ and $\epsilon = 0.15$.
For each combination of $M \in \{5, 10, 20\}$ and $v \in \{0, 1, 2, \ldots, 9\}$, we generated $30$ Bernoulli $0.15$-MPMAB problem instances with $M$ players and exactly $v$ subpar arms, that is, for each instance, $|\Ical_\epsilon| = v$.
Figures \ref{figure:exp2a}, \ref{figure:exp2b} and \ref{figure:exp2c} compare the average regrets after $100,000$ rounds in instances with different numbers of players $M$, in which $|\Ical_\epsilon|$ are set to be $9, 5$ and $0$, respectively. Again, figures in which $|\Ical_\epsilon|$ takes other values are deferred to Appendix~\ref{app:more_exp_results}.

Observe that when $|\Ical_\epsilon|$ is large, the collective regret of $\adaptedrobustagg(0.15)$ is less sensitive to the number of players. In the extreme case when $|\Ical_\epsilon| = 9$, all suboptimal arms are subpar arms, and \Cref{figure:exp2a} shows that the collective regret of $\adaptedrobustagg(0.15)$ has negligible dependence on the number of players $M$. 

\subsection{Discussion}
Back to the questions we raised earlier, our simulations show that $\adaptedrobustagg(\epsilon)$, in general, outperforms the baseline algorithms \inducb and \naiveagg.
When the set of subpar arms $\Ical_\epsilon$ is large, we showed that properly managing data aggregation can substantially improve the players' collective performance in an $\epsilon$-MPMAB problem instance. When there is no subpar arm, we demonstrated the robustness of $\adaptedrobustagg(\epsilon)$, that is, its performance is comparable with \inducb, in which the players do not share information. These empirical results validate our theoretical analyses in Section~\ref{sec:known_eps}.

\section{Conclusion and Future Work}
In this paper, we studied multitask bandit learning from heterogeneous feedback. 
We formulated the $\epsilon$-MPMAB problem and showed that whether inter-player information sharing can boost the players' performance depends on the dissimilarity parameter $\epsilon$ as well as the intrinsic difficulty of each individual bandit problem that the players face. 
In particular, in the setting where $\epsilon$ is known, we presented a UCB-based data aggregation algorithm which has near-optimal instance-dependent regret guarantees.
We also provided upper and lower bounds in the setting where $\epsilon$ is unknown.

There are many avenues for future work. 
For example, we are interested in extending our results to contextual bandits and Markov decision processes. 
Another direction is to study multitask bandit learning under other interaction protocols (e.g., only a subset of players take actions in each round).
In the future, we would also like to evaluate our algorithms in real-world applications such as healthcare robotics \citep{riek-cacm}.

\section{Acknowledgments}
We thank Geelon So and Gaurav Mahajan for insightful discussions. We also thank the National Science Foundation under IIS 1915734 and CCF 1719133 for research support. Chicheng Zhang acknowledges startup funding support from the University of Arizona.

\newpage
\onecolumn
\appendix

\section{Related Work and Comparisons}
\label{appendix:related_work}

We review the literature on multi-player bandit problems (see also \citet[Section 1.3.2]{l19} for a survey), and we comment on how existing problem formulations/approaches compare with ours studied in this paper.

\paragraph{Identical reward distributions.} A large portion of prior studies focuses on the setting where a group of players collaboratively work on one bandit learning problem instance, i.e., for each arm/action, the reward distribution is identical for every player. 

For example, \citet{kpc11} study a networked bandit problem, in which only one agent observes rewards, and the other agents only have access to its sampling pattern.
Peer-to-peer networks are explored by \citet{sbhojk13}, in which limited communication is allowed based on an overlay network.
\citet{lsl16} apply running consensus
algorithms to study a distributed cooperative multi-armed bandit problem. %
\citet{kjg18} study collaborative stochastic bandits over different structures of social networks that connect a group of agents. 
\citet{whcw19} study communication cost minimization in multi-agent multi-armed bandits.
Multi-agent bandit with a gossip-style protocol that has a communication budget is investigated in \citep{sgs19,csgs20}.
\citet{dp20a} investigate multi-agent bandits with heavy-tailed rewards.
\citet{wpajr20} present an approach with a ``parsimonious exploration principle'' to minimize regret and communication cost.
We note that, in contrast, we study multi-player bandit learning where the reward distributions can be different across players .

\paragraph{Player-dependent reward distributions.}
Multi-agent bandit learning with \textit{heterogeneous feedback} has also been covered by previous studies. 

\begin{itemize}
    \item \citet{cgz13} study a network of linear contextual bandit players with heterogeneous rewards, where the players can take advantage of reward similarities hinted by a graph.
    In \citep{wwgw16,www17,whle17}, reward distributions of each player are \textit{generated} based on social influence, which is modeled using preferences of the player's neighbors in a graph.
    These papers use regularization-based methods that take advantage of graph structures; in contrast, we study \textit{when and how} to use information from other players based on a dissimilarity parameter.
    
    \item \citet{glz14,bcs14,stv14,lkg16,ksl16,lcll19}, among others, assume that the players' reward distributions have a cluster structure and players that belong to one cluster share a \textit{common} reward distribution; our paper does not assume such cluster structure. %
    
    \item \citet{nl14} investigate dynamic clustering of players with independent reward distributions and provides an empirical validation of their algorithm; \citet{zhx20} present an algorithm that combines dynamic clustering and Thompson sampling. In contrast, in this paper, we develop a UCB-based approach that has a fallback guarantee\footnote{In \citet{zhx20}, it is unclear how to tune the hyper-parameter $\beta$ apriori to ensure a sublinear fall-back regret guarantee, even if the ``similarity'' parameter $\gamma$ is known.}. 
    
    \item In the work of \citet{srj17}, a group of players seek to find the arm with the largest average reward over all players; and, in each round, the players have to reach a consensus and choose the same arm.
    
    \item \citet{dp20b} assume access to some side information for every player, and learns a reward predictor that takes both player's side information models and action as input. In comparison, our work do not assume access to such side information.
    
    \item Further, similarities in reward distributions are explored in the work of \citet{zailn19}, which studies a \textit{warm-start} scenario, in which data are provided as history \citep{sj12} for an learning agent to explore faster. 
    \citet{glb13,soaremulti} investigate multitask learning in bandits through \textit{sequential transfer} between tasks that have similar reward distributions. In contrast, we study the multi-player setting, where all players learn continually and concurrently.

\end{itemize}

\paragraph{Collisions in multi-player bandits.} Multi-player bandit problems with collisions \citep[e.g.,][]{lz10,knj14,bp20,bb20,sxsy20,blps20,wpajr20} are also well-studied. In such models, two players pulling the same arm in the same round \textit{collide} and receive zero reward. These models have a wide range of practical applications (e.g., cognitive radio), and some assume \textit{player-dependent heterogeneous} reward distributions~\citep{bblb20,bkmp20}; in comparison, collision is not modeled in our paper.

\paragraph{Side information.}
Models in which learning agents observe side information have also been studied in prior works---one can consider data collected by other players in multi-player bandits as side observations \citep{l19}. In some models, a player observes side information for some arms that are not chosen in the current round: stochastic models with such side information are studied in \citep{cklb12,bes14,wgs15}, and adversarial models in \citep{ms11,acgmms17}; 
Similarities/closeness among arms in one bandit problem are studied in \citep{dds17,xvzs17,wzzlisg18}.
We note that our problem formulation is different, because in these models, auxiliary data are from arms in the same bandit problem instance instead of from other players.

Upper and lower bounds on the means of reward distributions are used as side information in \citep{sbss20}. Loss predictors~\citep{wla20} can also be considered as side information. In contrast, we do not leverage such information.
Further, side information can also refer to ``context'' in contextual bandits \citep{s14}. In comparison, we assume a multi-armed setting and their results do not imply ours.

\paragraph{Other multi-player bandit learning topics.} Many other multi-player bandit learning topics have also been explored. 
For example, \citet{ak05,vss20} study multi-player models in which some of the players are malicious.
\citet{cb18} study collaborative bandits with applications such as top-$K$ recommendations. 
Nonstochastic multi-armed bandit models with communicating agents are studied in \citep{bm19,cgm19}. 
Privacy protection in decentralized exploration is investigated in \citep{fal19}.
We note that, in this paper, our goal does not align closely with these topics.

\section{Proof of Claim~\ref{claim:delta-small}}
\label{appendix:claim2}

We first restate Example~\ref{example:delta-small} and Claim~\ref{claim:delta-small}.

\newtheorem*{E1}{Example~\ref{example:delta-small}}
\begin{E1}
For a fixed $\epsilon \in (0,\frac1{8})$ and $\delta \le \epsilon/4$, consider the following Bernoulli MPMAB problem instance: for each $p \in [M]$, $\mu_1^p = \frac12 + \delta$, $\mu_2^p = \frac12$. 
This is a $0$-MPMAB instance, hence an $\epsilon$-MPMAB problem instance.
Also, note that $\epsilon$ is at least  four times larger than the gaps $\Delta^p_2 = \delta$.
\end{E1}

\newtheorem*{C1}{Claim~\ref{claim:delta-small}}
\begin{C1}
For the above example, any sublinear regret algorithm for the $\epsilon$-MPMAB problem must have $\Omega(\frac{M \ln T}{\delta})$ regret on this instance, 
matching the \inducb regret upper bound.
\end{C1}

\begin{proof}[Proof of Claim~\ref{claim:delta-small}]
Suppose $\Acal$ is a sublinear-regret algorithm for the $\epsilon$-MPMAB problem; i.e., there exist $C>0$ and $\alpha>0$ such that $\Acal$ has $CT^{1-\alpha}$ regret in all $\epsilon$-MPMAB instances. 

Recall that we consider the Bernoulli $\epsilon$-MPMAB instance $\mu = (\mu_i^p)_{i \in [2], p \in [M]}$ such that $\mu_1^p = \frac12 +\delta$ and $\mu_2^p = \frac12$ for all $p$. As $\epsilon \in (0, \frac18)$ and $\delta \leq \frac{\epsilon}{4}$, it can be directly verified that all $\mu_i^p$'s are in $[\frac{15}{32}, \frac{17}{32}]$. In addition, since for all $p$, $\Delta_2^p = \delta \leq \frac{\epsilon}{4} = 5 \cdot \frac{\epsilon}{20}$, we have $\Ical_{\epsilon/20} = \O$, i.e., $\Ical^C_{\epsilon/20} = \cbr{1,2}$. 

From Theorem~\ref{thm:gap_dep_lb}, we conclude that for this MPMAB instance $\mu$, $\Acal$ has regret lower bounded as follows:
\[
\EE_{\mu} \sbr{ \Rcal(T) }
\geq
\Omega\del{\frac{M \ln(T^\alpha \Delta_2^p / C)}{\Delta_2^p}} 
=
\Omega\del{\frac{M \ln(T^\alpha \delta / C)}{\delta}}
=
\Omega\del{\frac{M \ln T}{\delta}},
\]
for sufficiently large $T$.
\end{proof}

\section{Basic properties of $\Ical_\epsilon$ for $\epsilon$-MPMAB instances}
\label{appendix:proof_of_fact}

In Section~\ref{sec:known_eps} of the paper, we presented the following two facts about properties of $\Ical_\epsilon$ for $\epsilon$-MPMAB problem instances:
\newtheorem*{F1}{Fact~\ref{fact:optimal_arm}}
\begin{F1}
$\abs{\Ical_\epsilon} \leq K-1$. In addition, for each arm $i \in \Ical_\epsilon$, $\Delta_i^{\min} > 3\epsilon$; in other words,  for all players $p$ in $[M]$, $\Delta_i^p = \mu_*^p - \mu_i^p > 3\epsilon$; consequently, arm $i$ is suboptimal for all players $p$ in $[M]$. 
\end{F1}

\newtheorem*{F2}{Fact~\ref{fact:i-eps-inv-gap-short}}
\begin{F2}
For any $i \in \Ical_{\epsilon}$, 
$\frac{1}{\Delta_i^{\min}} \leq \frac{2}{M} \sum_{p \in [M]} \frac{1}{\Delta_i^p}$.
\end{F2}

Here, we will present and prove a more complete collection of facts about the properties of $\Ical_\epsilon$ which covers every statement in Fact~\ref{fact:optimal_arm} and Fact~\ref{fact:i-eps-inv-gap-short}.
Before that, we first prove the following fact.
\begin{fact}
\label{fact:delta_difference}
For an $\epsilon$-MPMAB problem instance, for any $i \in [K]$, and $p, q \in [M]$,  $|\Delta_i^p - \Delta_i^q| \le 2\epsilon$.
\end{fact}

\begin{proof}
Fix any player $p \in [M]$, let $j \in [K]$ be an optimal arm for $p$ such that $\mu_j^p = \mu_*^p$. We first show that, for any player $q \in [M]$, $|\mu_*^q - \mu_j^p| \le \epsilon$.

\begin{itemize}
    \item $\mu_*^q \ge \mu_j^p - \epsilon$ is trivially true because $\mu_j^q \ge \mu_j^p - \epsilon$ by Definition~\ref{assumption:epsilon} and $\mu_*^q \ge \mu_j^q$ by the definition of $\mu_*^q$;
    \item $\mu_*^q \le \mu_j^p + \epsilon$ is true because if there exists an arm $k \in [K]$ such that $\mu_k^q > \mu_j^p + \epsilon$, then by Definition~\ref{assumption:epsilon} we must have $\mu_k^p \ge \mu_k^q - \epsilon > \mu_j^p$ which contradicts with the premise that $j$ is an optimal arm for player $p$.
\end{itemize}

We have shown that $|\mu_*^q - \mu_*^p| \le \epsilon$. Since $|\mu_i^q - \mu_i^p| \le \epsilon$ by Definition~\ref{assumption:epsilon}, it follows from the triangle inequality that $|\Delta_i^p - \Delta_i^q| \le 2\epsilon$.
\end{proof}

We now present a set of basic properties of $\Ical_\epsilon$.

\begin{fact}[Basic properties of $\Ical_\epsilon$]
\label{fact:Ical_properties}
Let $\Delta_i^{\max} = \max_{p \in [M]} \Delta_i^p$.
For an $\epsilon$-MPMAB problem instance,
for each arm $i \in \Ical_\epsilon$,
\begin{enumerate}[(a)]
    \item $\Delta_i^p > 3\epsilon$ for all players $p \in [M]$; in other words, $\Delta_i^{\min} > 3\epsilon$; \label{item:3_epsilon}
    \item arm $i$ is suboptimal for all players $p \in [M]$, i.e., for any player $p \in [M]$, $\mu_i^p < \mu_*^p$; \label{item:suboptimal_arm}
    \item $\frac{\Delta_i^p}{\Delta_i^q} < 2$ for any pair of players $p, q \in [M]$; consequently, $\frac{\Delta_i^{\max}}{\Delta_i^{\min}} < 2$; \label{item:delta_quotient}
    \item $\frac{1}{\Delta_i^{\min}} \leq \frac{2}{M} \sum_{p \in [M]} \frac{1}{\Delta_i^p}$;  
    \item $|\Ical_\epsilon| \le K - 1$.
\end{enumerate}
\end{fact}

\begin{proof}
We prove each item one by one.
\begin{enumerate}[(a)]
    \item For each arm $i \in \Ical_\epsilon$, by definition, there exists $p \in [M]$, $\Delta_i^p > 5\epsilon$. It follows from Fact~\ref{fact:delta_difference} that for any $q \in [M]$, $\Delta_i^q \ge \Delta_i^p - 2\epsilon > 3\epsilon$. $\Delta_i^{\min} > 3\epsilon$ then follows straightforwardly.
    
    \item For each arm $i \in \Ical_\epsilon$, it follows from item~\ref{item:3_epsilon} that for any $p \in [M]$, $\Delta_i^p > 3\epsilon \ge 0$. Therefore, $i$ is suboptimal for all player $p \in [M]$.
    
    \item By Fact~\ref{fact:delta_difference}, for any $i \in \Ical_\epsilon \subseteq [K]$ and any $p, q \in [M]$, $\Delta_i^p \le \Delta_i^q + 2\epsilon$, which implies $\frac{\Delta_i^p}{\Delta_i^q} \leq 1 + \frac{2\epsilon}{\Delta_i^q}$.
    Since by item~\ref{item:3_epsilon}, $\Delta_i^q > 3\epsilon$, it follows that
    $\frac{\Delta_i^p}{\Delta_i^q} \leq 1 + \frac{2\epsilon}{\Delta_i^q} < 2$. $\frac{\Delta_i^{\max}}{\Delta_i^{\min}} < 2$ then follows straightforwardly.
    
    \item For each arm $i \in \Ical_\epsilon$, it follows from item~\ref{item:delta_quotient} that for any $p \in [M]$, $\Delta_i^p \in [\Delta_i^{\min}, 2\Delta_i^{\min}]$. Therefore, we have $\frac{1}{\Delta_i^p} \in [\frac{1}{2\Delta_i^{\min}}, \frac{1}{\Delta_i^{\min}}]$, as $\Delta_i^p > 0$ for all $p$. It then follows that $\frac{2}{M} \sum_{p \in [M]} \frac{1}{\Delta_i^p} \ge \frac{2}{M} \sum_{p \in [M]} \frac{1}{2\Delta_i^{\min}} = \frac{1}{\Delta_i^{\min}}.$
    
    \item Pick an arm $i$ that is optimal with respect to player $1$; $i$ cannot be in $\Ical_\epsilon$ because of item~\ref{item:suboptimal_arm}. Therefore, 
    $\Ical_\epsilon \subseteq [K] \setminus \cbr{i}$, which implies that it has size at most $K-1$.
    \qedhere
\end{enumerate}
\end{proof}

\section{Proof of Upper Bounds in \Cref{sec:known_eps}}
\label{appendix:upper_bounds_known_eps}

\subsection{Proof Overview} 
In Appendix \ref{app:event_q} and \ref{app:event_e}, we focus on showing that in a ``clean'' event $\mathcal{E}$ (defined in \ref{app:event_e}), the upper confidence bound UCB$^p_i(t) = \kappa_i^p(t, \lambda) + F(\wbar{n^p_i}, \wbar{m^p_i}, \lambda, \epsilon)$ (line~\ref{line:ucb-end} of Algorithm~\ref{alg: mpmab})\footnote{Recall that $\wbar{z} = \max\{z,1\}$.} holds for every $t \in [T], i \in [K], p \in [M]$ and $\lambda \in [0,1]$; and the ``clean'' event $\mathcal{E}$ occurs with $1 - 4MK/T^4$ probability.

Then, in Appendix \ref{appendix:gap_dep_ub}, we provide a proof of the gap-dependent upper bound in Theorem~\ref{thm:gap_dep_ub}. In Appendix \ref{appendix:gap_ind_ub}, we provide a proof of the gap-independent upper bound in Theorem~\ref{thm:gap_ind_upper}.

\subsection{Event $\mathcal{Q}_i(t)$}
\label{app:event_q}
Recall that $n^p_i(t-1)$ is the number of pulls of arm $i$ by player $p$ after the first $(t-1)$ rounds. Let $m^p_i(t-1) = \sum_{q \in [M]: q \neq p} n^q_i(t-1)$.

We now define the following event.

\begin{definition}
Let $$\mathcal{Q}_i(t) = 
\cbr{ 
\forall p, 
\left|\zeta^p_i(t) - \mu^p_i\right| \le 8\sqrt{\frac{3\ln T}{\wbar{n^p_i}(t-1)}},\quad
\abs{ \eta_i^p(t) - \sum_{q \neq p} \frac{n_i^q(t-1)}{\wbar{m_i^p}(t-1)} \mu_i^q} \leq 4\sqrt{\frac{14 \ln T}{\wbar{m_i^p}(t-1)}} 
},
$$
where
\[ 
\zeta_i^p(t) = \frac{ \sum_{s=1}^{t-1} \mathds{1}\{ i_t^p = i \}r_t^p}{\wbar{n^p_i}(t-1)},  
\]
and
\[
\eta_i^p(t) = \frac{ \sum_{s=1}^{t-1} \sum_{q = 1}^M \mathds{1} \{ q \neq p,\ i_t^q = i \} r_t^q }{\wbar{m^p_i}(t-1) }.
\]
\end{definition}

\begin{lemma}
\label{lem:Q}
\[ 
\Pr( \mathcal{Q}_i(t) ) \geq 1 - 4MT^{-5}.
\]
\end{lemma}
\begin{proof}
For any fixed player $p$, we discuss the two inequalities separately. Lemma~\ref{lem:Q} then follows by a union bound over the two inequalities and over all $p \in [M]$.

We first discuss the concentration of $\zeta_i^p(t)$. We define a filtration $\cbr{\Bcal_t}_{t=1}^T$, where
$$\Bcal_t = \sigma( \cbr{ i_s^{p'}, r_s^{p'}: s \in [t], p' \in [M] } \cup \cbr{ i_{t+1}^{p'}: p' \in [M]})$$ is the $\sigma$-algebra generated by the historical interactions up to round $t$ and the arm selection of all players at round $t+1$.

Let random variable $X_t = \mathds{1}\{ i_t^p = i \}  \del{r^p_t - \mu_i^p}$. We have $\EE\sbr{X_t \mid \Bcal_{t-1}} = 0$; in addition, $\VV\sbr{X_t \mid \Bcal_{t-1}} = 
\EE\sbr{(X_t - \EE[X_t \mid \Bcal_{t-1}])^2 \mid \Bcal_{t-1}}
\leq 
\EE\sbr{ (\mathds{1}\{ i_t^p = i \} r^p_t)^2  \mid \Bcal_{t-1}} 
\leq \mathds{1} \{ i_t^p = i \}$ and $\abs{X_t} \leq 1$.

Applying Freedman's inequality~\cite[Lemma 2]{bartlett2008high} with 
$\sigma = \sqrt{\sum_{s=1}^{t-1} \VV\sbr{X_s \mid \Bcal_{s-1}}}$ and $b = 1$, and using $\sigma \leq \sqrt{\sum_{s=1}^{t-1} \ind\{i_s^p = i\}}$,
we have that with probability at least $1-2 T^{-5}$,
\begin{equation}
\abs{\sum_{s=1}^{t-1} X_s}
\leq
4 \sqrt{  \sum_{s=1}^{t-1} \mathds{1} \{ i_s^p = i \} \cdot  \ln(T^5\log_2 T) }
+
2 \ln(T^5 \log_2 T).
\label{eqn:freedman-p-i-1}
\end{equation}

We consider two cases:
\begin{enumerate}
\item If $n^p_i(t-1) = \sum_{s=1}^{t-1} \mathds{1} \{ i_s^p = i \} = 0$, we have $\wbar{n^p_i}(t-1) = 1$ and $\zeta_i^p(t) = 0$. In this case, we trivially have
\[ 
\abs{ \zeta_i^p(t) - \mu_i^p } 
\leq 1 
\leq 8 \sqrt{\frac{3\ln T}{\wbar{n^p_i}(t-1)}}.
\]
\item Otherwise, $n^p_i(t-1) \geq 1$. In this case, we have $\wbar{n^p_i}(t-1) = n^p_i(t-1)$. Divide both sides of Eq.~\eqref{eqn:freedman-p-i-1} by $n^p_i(t-1)$, and use the fact that $\log T \leq T$, we have
\[
\abs{  \frac{\sum_{s=1}^{t-1} \mathds{1} \{ i_s^p = i \} r_s^p}{n^p_i(t-1)} -\mu_i^p }
\leq 
4 \sqrt{ \frac{6 \ln T}{n^p_i(t-1)}  }
+
\frac{12 \ln T}{n^p_i(t-1)}.
\]
If $\frac{12 \ln T}{n^p_i(t-1)} \geq 1$, $\abs{  \frac{\sum_{s=1}^{t-1} \mathds{1} \{ i_s^p = i \} r_s^p}{n^p_i(t-1)} -\mu_i^p } \leq 8 \sqrt{ \frac{3 \ln T}{n^p_i(t-1)} }$ is trivially true. 
Otherwise, $\frac{12 \ln T}{n^p_i(t-1)} \leq 2\sqrt{\frac{3 \ln T }{n^p_i(t-1)}}$, which implies that $\abs{ \frac{\sum_{s=1}^{t-1} \mathds{1} \{ i_s^p = i \} r_s^p}{n^p_i(t-1)} -\mu_i^p } \leq (4\sqrt{6} + 2\sqrt{3}) \sqrt{\frac{\ln T}{n^p_i(t-1)}} \leq 8 \sqrt{\frac{3  \ln T}{n^p_i(t-1)}}$. 
\end{enumerate} 
In summary, in both cases,
with probability at least $1-2T^{-5}$, we have
\[\abs{ \zeta_i^p(t-1) -\mu_i^p } \leq 8 \sqrt{\frac{3 \ln T}{\wbar{n^p_i}(t-1)}}. \]

A similar application of Freedman's inequality also shows the concentration of $\eta_i^p(t)$.
Similarly, we define a filtration $\{\mathcal{G}_{t,q} \}_{t \in [T], q \in [M]}$, where
$$\mathcal{G}_{t,q} = \sigma( \cbr{ i_s^{p'}, r_s^{p'}: s \in [t], p' \in [M] } \cup \cbr{ i_{t+1}^{p'}: p' \in [M], p' \le q})$$ 
is the $\sigma$-algebra generated by the historical interactions up to round $t$ and the arm selection of players $1,2,\ldots,q$ in round $t+1$. 
We have
\[
\Gcal_{1,1} \subset \Gcal_{1,2} \subset \ldots \subset \Gcal_{1,M} \subset \Gcal_{2,1} \subset \ldots \subset \Gcal_{2,M} \subset \ldots \subset \Gcal_{T,M}.
\]
By convention, let $\Gcal_{t,0} = \Gcal_{t-1,M}$.

Now, let random variable $Y_{t,q} = \mathds{1} \{ q \neq p,\ i_t^q = i \}  \del{r^q_t - \mu_i^q}$. We have $\EE\sbr{Y_{t,q} \mid \Gcal_{t-1,q-1}} = 0$;
in addition, $\VV\sbr{Y_{t,q} \mid \Gcal_{t-1,q-1}} = %
\EE\sbr{Y_{t,q}^2 \mid \Gcal_{t-1,q-1}} \leq \mathds{1} \{ q \neq p,\ i_t^q = i \},$ and $\abs{Y_{t,q}} \leq 1$.

Similarly, applying Freedman's inequality~\cite[Lemma 2]{bartlett2008high}
with 
$\sigma = \sqrt{\sum_{s=1}^{t-1} \sum_{q=1}^M \VV\sbr{Y_{s,q} \mid \Gcal_{s-1,q-1}}}$ and $b = 1$, and using $\sigma \leq \sqrt{\sum_{s=1}^{t-1} \sum_{q=1}^M \ind\{q \neq p, i_s^q = i}\}$,
we have that with probability at least $1-2T^{-5}$,
\begin{equation}
\abs{\sum_{s=1}^{t-1} \sum_{q=1}^{M} Y_{s,q}}
\leq
4 \sqrt{  \sum_{s=1}^{t-1} \sum_{q=1}^M \mathds{1}\{ q \neq p,\ i_s^q = i \} \cdot  \ln(T^5\log_2 (TM)) }
+
2 \ln(T^5 \log_2 (TM)).
\label{eqn:freedman-p-i-2}
\end{equation}

Again, we consider two cases. If $m^p_i(t-1) = 0$, then we have $\eta_i^p(t-1) = 0$ and $$\abs{ \eta_i^p(t-1) - \sum_{q \neq p} \frac{n^q_i(t-1)}{\wbar{m^p_i}(t-1) }\mu_i^q } = 0 \le 4 \sqrt{\frac{14 \ln T}{\wbar{m^p_i}(t-1)}}.$$

Otherwise, we have $\wbar{m^p_i}(t-1) = m^p_i(t-1)$. Divide both sides of Eq.~\eqref{eqn:freedman-p-i-2} by $m^p_i(t-1)$, and use the fact that $\log_2 (TM) \leq T^2$,
we have
\[
\abs{  \frac{\sum_{s=1}^{t-1} \sum_{q=1}^{M} \mathds{1}\{ q \neq p,\ i_s^q = i \} r_s^q}{m^p_i(t-1)} -  \sum_{q \neq p} \frac{n^q_i(t-1)}{m^p_i(t-1) }\mu_i^q }
\leq 
4 \sqrt{ \frac{7 \ln T}{m^p_i(t-1)}  }
+
\frac{14 \ln T}{m^p_i(t-1)}.
\]
If $\frac{14 \ln T}{m^p_i(t-1)} \geq 1$, $\abs{  \frac{\sum_{s=1}^{t-1} \sum_{q=1}^{M} \mathds{1}\{ q \neq p,\ i_s^q = i \} r_s^q}{m^p_i(t-1)} -  \sum_{q \neq p}\frac{n^q_i(t-1)}{m^p_i(t-1) }\mu_i^q } \leq 4 \sqrt{\frac{14 \ln T}{m^p_i(t-1)}}$ is trivially true. 
Otherwise, $\frac{14 \ln T}{m^p_i(t-1)} \leq \sqrt{\frac{14 \ln T }{m^p_i(t-1)}}$, which implies that 
\begin{align}
\abs{  \frac{\sum_{s=1}^{t-1} \sum_{q=1}^{M} \mathds{1} \{ q \neq p,\ i_s^q = i\} r_s^q}{m^p_i(t-1)} -  \sum_{q \neq p}\frac{n^q_i(t-1)}{m^p_i(t-1) }\mu_i^q } &\leq (4\sqrt{7} + \sqrt{14}) \sqrt{\frac{\ln T}{m^p_i(t-1)}} \nonumber \\
&\leq 4 \sqrt{\frac{14  \ln T}{m^p_i(t-1)}}. \nonumber
\end{align}

In summary, in both cases, with probability at least $1-2T^{-5}$, we have
\[\abs{ \eta_i^p(t-1) -\sum_{q \neq p}\frac{n^q_i(t-1)}{\wbar{m^p_i}(t-1) }\mu_i^q } \leq 4 \sqrt{\frac{14 \ln T}{\wbar{m^p_i}(t-1)}}. \]

The lemma follows by taking a union bound over these two inequalities for each fixed $p$, and over all $p \in [M]$.
\end{proof}

\subsection{Event $\mathcal{E}$}
\label{app:event_e}
Let $\mathcal{E} = \cap_{t=1}^{T} \cap_{i=1}^K \Qcal_i(t)$. We present the following corollary and lemma regarding event $\mathcal{E}$.

\begin{corollary}
It follows from Lemma~\ref{lem:Q} that $\Pr[\mathcal{E}] \ge 1 - \frac{4MK}{T^4}$.
\end{corollary}

\begin{lemma}
\label{lem:kappa}
If $\mathcal{E}$ occurs, we have that for every $t \in [T]$, $i \in [K]$, $p \in [M]$, for all $\lambda \in [0,1]$,
\[
\abs{ \kappa_i^p(t,\lambda) - \mu_i^p }
\leq
8 \sqrt{13\ln T \del{ \frac{\lambda^{2}}{ \wbar{n^p_i}(t-1) } + \frac{(1-\lambda)^2}{ \wbar{m_i^p}(t-1) }} } 
+ 
(1-\lambda)\epsilon,
\]
where $\kappa_i^p(t,\lambda) = \lambda \zeta^p_i(t) + (1-\lambda) \eta^p_i(t)$.
\end{lemma}
\begin{proof}
If $\mathcal{E}$ occurs, for every $t \in [T]$ and $i \in [K]$, by the definition of event $\Qcal_i(t)$, we have
\[
\abs{ \zeta^p_i(t) - \mu^p_i } < 8 \sqrt{\frac{3\ln T}{\wbar{n^p_i}(t-1)}}, \text{and}\
\abs{ \eta_i^p(t) - \sum_{q \neq p} \frac{n_i^q(t-1)}{\wbar{m_i^p}(t-1)} \mu_i^q} \leq 4 \sqrt{\frac{14 \ln T}{\wbar{m_i^p}(t-1)}}.
\]
As $\kappa_i^p(t,\lambda) = \lambda \zeta^p_i(t) + (1-\lambda) \eta^p_i(t)$, we have:
\begin{align}
\abs{ \kappa_i^p(t,\lambda) - \Big[\lambda \mu^p_i + (1-\lambda) \sum_{q \neq p} \frac{n_i^q(t-1)}{\wbar{m_i^p}(t-1)} \mu_i^q \Big] } 
\leq &
8\lambda \sqrt{\frac{3\ln T}{\wbar{n^p_i}(t-1)}}
+
4 (1-\lambda) \sqrt{\frac{14 \ln T}{\wbar{m_i^p}(t-1)}} 
\notag \\
\leq &
8 \sqrt{13\ln T \del{ \frac{\lambda^{2}}{ \wbar{n^p_i}(t-1) } + \frac{(1-\lambda)^2}{ \wbar{m_i^p}(t-1) }}},
\label{eqn:var-wt-agg-1}
\end{align}
where the second inequality uses the elementary facts that $\sqrt{A} + \sqrt{B} \leq \sqrt{2(A+B)}$.

Furthermore, from Definition~\ref{assumption:epsilon}, we have
\[ 
\abs{ \sum_{q \neq p} \frac{n_i^q(t-1)}{\wbar{m_i^p}(t-1)} \mu_i^q - \mu_i^p }
\leq
\sum_{q \neq p} \frac{n_i^q(t-1)}{\wbar{m_i^p}(t-1)} \abs{\mu_i^q - \mu_i^p} 
\leq
\epsilon.
\]
This shows that 
\[
\abs{ \mu^p_i - (\lambda \mu^p_i + (1-\lambda) \sum_{q \neq p} \frac{n_i^q(t-1)}{\wbar{m_i^p}(t-1)} \mu_i^q) }
\leq 
(1-\lambda) \epsilon.
\]
Combining the above inequality with Eq.~\eqref{eqn:var-wt-agg-1}, we get
\[
\abs{ \kappa_i^p(t, \lambda) - \mu_i^p }
\leq
8 \sqrt{13\ln T \del{ \frac{\lambda^{2}}{ \wbar{n^p_i}(t-1) } + \frac{(1-\lambda)^2}{ \wbar{m_i^p}(t-1) }} } 
+ 
(1-\lambda)\epsilon. 
\]
This completes the proof.
\end{proof}

\subsection{Proof of Theorem~\ref{thm:gap_dep_ub}}
\label{appendix:gap_dep_ub}

We first restate Theorem~\ref{thm:gap_dep_ub}.
\newtheorem*{T1}{Theorem~\ref{thm:gap_dep_ub}}
\begin{T1}
Let $\robustagg(\epsilon)$ run on an $\epsilon$-MPMAB problem instance for $T$ rounds. 
Then, its expected collective regret satisfies
\begin{align*}
\mathbb{E}[\mathcal{R}(T)] \le \order \Bigg(
\sum_{i \in \Ical_\epsilon} \del{ \frac{\ln T}{\Delta_i^{\min}} + M \Delta_i^{\min}} +
\sum_{i \in \Ical^C_\epsilon}
\sum_{p \in [M]: \Delta_i^p > 0}
\frac{\ln T}{\Delta_i^p}
 \Bigg).
\end{align*}
\end{T1}

Recall that the expected collective regret is defined as $\EE[\mathcal{R}(T)] = \sum_{i\in [K]} \sum_{p \in [M]} \Delta_i^p \cdot \mathbb{E} [n^p_i(T)]$. %
Before we prove \Cref{thm:gap_dep_ub}, 
we first present the following two lemmas, which provides an upper bound for (1) the total number of arm pulls for arm $i$, for $i$ in $\Ical_{\epsilon}$ and (2) the individual number of arm pulls for arm $i$ and player $p$, for $i$ in $\Ical_{\epsilon}^C$, conditioned on $\Ecal$ happening.

\begin{lemma}
\label{lem:arm_pulls_I}
Denote $n_i(T) = \sum_{p \in [M]} n^p_i(T)$ as the total number of pulls of arm $i$ by all the players after $T$ rounds. Let $\robustagg(\epsilon)$ run on an $\epsilon$-MPMAB problem instance for $T$ rounds.
Then, for each $i \in \Ical_\epsilon$, we have
$$\mathbb{E} [n_i(T) \vert \mathcal{E}] \le \order \left(\frac{\ln T}{(\Delta_i^{\min})^2}  + M\right).$$
\end{lemma}

\begin{lemma}
\label{lem:arm_pulls_IC}
Let $\robustagg(\epsilon)$ run on an $\epsilon$-MPMAB problem instance for $T$ rounds. Then, for each $i \in \Ical^C_\epsilon$ and player $p \in [M]$ such that $\Delta_i^p > 0$, we have
$$
\mathbb{E} [n^p_i(T) \vert \mathcal{E}] \le \order \left( \frac{\ln T}{(\Delta^p_i)^2} \right).
$$
\end{lemma}

\begin{proof}[Proof of Lemma~\ref{lem:arm_pulls_I}]

We first note that it follows from item~\ref{item:suboptimal_arm} of Fact~\ref{fact:Ical_properties} that every arm $i \in \Ical_\epsilon$ is suboptimal for all players $p \in [M]$.

We have
\begin{flalign}
n_i(T) & = \sum_{t = 1}^T\ \sum_{p = 1}^M \mathds{1}\big\{ i^p_t = i\big\} \nonumber \\
& \le M + \tau + \sum_{t = 1}^T\ \sum_{p = 1}^M \mathds{1} \big\{ i^p_t = i, n_i(t-1) > \tau \big\}. \label{eq-counting}
\end{flalign}
Here, $\tau \ge 1$ is an arbitrary integer. The term $M$ is due to parallel arm pulls in the $\epsilon$-MPMAB problem: Let $s$ be the first round such that after round $s$, the total number of pulls $n_i(s) > \tau$. This implies that $n_i(s-1) \le \tau$. Then in round $s$, there can be up to $M$ pulls of arm $i$ by all the players, which means that in round $(s+1)$ when the third term in Eq.~\eqref{eq-counting} can first start counting, there could have been up to $\tau + M$ pulls of the arm $i$.

It then follows that
\begin{flalign}
n_i(T) & \le M + \tau + \sum_{t = 1}^T\ \sum_{p = 1}^M \mathds{1} \big\{\text{UCB}^p_{i_*^p}(t) \le \text{UCB}^p_i(t),\ n_i(t-1) > \tau \big\}. \label{eqn:armpull-batch}
\end{flalign}

Recall that $\Delta_i^{\min} = \min_{p} \Delta_i^p$, and for each $i \in \Ical_\epsilon$, we have $\Delta_i^p \ge \Delta_i^{\min} > 3\epsilon$ by item~\ref{item:3_epsilon} of Fact~\ref{fact:Ical_properties}.

With foresight, we choose $\tau = \lceil \frac{3328\ln T}{(\Delta^{\min}_i - 2\epsilon)^2} \rceil$. 
Conditional on $\mathcal{E}$, we show that, \textit{for any arm $i \in \Ical_\epsilon$}, the event $\big\{ \text{UCB}^p_{i_*^p}(t) \le \text{UCB}^p_i(t), n_i(t-1) > \tau \big\}$ 
never happens. It suffices to show that if $n_i(t-1) > \tau$,
\begin{equation} 
\UCB^p_{i_*^p}(t) \geq \mu_*^p, 
\label{eqn:i-star-large-eps}
\end{equation}
and
\begin{equation} \UCB^p_i(t) < \mu^p_*
\label{eqn:i-small-eps}
\end{equation}
happen simultaneously. 

Eq.~\eqref{eqn:i-star-large-eps} follows straightforwardly from the definition of $\Ecal$ along with Lemma~\ref{lem:kappa}. For Eq.~\eqref{eqn:i-small-eps}, we have the following upper bound on $\UCB_i^p(t)$:
\begin{align}
    \UCB_i^p(t)
    &= \kappa_i^p(t, \lambda^*) + F(\wbar{n^p_i}, \wbar{m^p_i}, \lambda^*, \epsilon) \notag \\
    & \leq \mu_i^p + 2 F(\wbar{n^p_i}, \wbar{m^p_i}, \lambda^*, \epsilon) \notag \\
    & = \mu_i^p + 2\Big[\min_{\lambda \in [0,1]} 8\sqrt{13\ln T[\frac{\lambda^2}{\wbar{n^p_i}(t-1)} + \frac{(1-\lambda)^2}{\wbar{m^p_i}(t-1)}]} + (1-\lambda)\epsilon\Big] \nonumber \\
    & \le \mu_i^p + 2\Big[8\sqrt{\frac{13\ln T}{\wbar{n^p_i}(t-1) + \wbar{m^p_i}(t-1)}} + \epsilon\Big] \nonumber \\
    & \le \mu_i^p + 2\Big[8\sqrt{\frac{13\ln T}{n_i(t-1)}} + \epsilon\Big] \nonumber \\
    & < \mu_i^p + 2\Big[8\sqrt{\frac{13\ln T (\Delta^p_i - 2\epsilon)^2}{3328\ln T}} +\epsilon \Big] = \mu_i^p + \Delta^p_i = \mu^p_*, \nonumber
\end{align}
where the first inequality is from the definition of $\Ecal$ and Lemma~\ref{lem:kappa}; the second inequality is from choosing $\lambda = \frac{\wbar{n^p_i}(t-1)}{\wbar{n^p_i}(t-1) + \wbar{m^p_i}(t-1)}$; the third inequality is from the simple facts that $n^p_i(t-1) \le \wbar{n^p_i}(t-1)$, $m^p_i(t-1) \le \wbar{m^p_i}(t-1)$, and $n_i(t-1) = n^p_i(t-1) + m^p_i(t-1)$;
the last inequality is from the premise that $n_i(t-1) > \tau \geq \frac{3328\ln T}{(\Delta^{\min}_i - 2\epsilon)^2} \geq \frac{3328\ln T}{(\Delta^p_i - 2\epsilon)^2}$.

Continuing Eq.~\eqref{eqn:armpull-batch}, it then follows that, for each $i \in \Ical_\epsilon$,
\begin{align}
    \mathbb{E} [n_i(T) \vert \mathcal{E}] \le \lceil \frac{3328\ln T}{(\Delta^{\min}_i - 2\epsilon)^2} \rceil + M \le  \frac{3328\ln T}{(\Delta^{\min}_i - 2\epsilon)^2}  + (M + 1).
    \label{eq:ub_pulls_I_before}
\end{align}

Now, by item~\ref{item:3_epsilon} of Fact~\ref{fact:Ical_properties}, for each $i \in \Ical_\epsilon$, $\Delta_i^{\min} > 3\epsilon$. We then have $\frac{\Delta_i^{\min}}{\Delta_i^{\min} - 2\epsilon} = \frac{\Delta_i^{\min} - 2\epsilon + 2\epsilon}{\Delta_i^{\min} - 2\epsilon} = 1 + \frac{2\epsilon}{\Delta_i^{\min} - 2\epsilon} < 3$. 
It follows that
$$\frac{3328\ln T}{(\Delta^{\min}_i - 2\epsilon)^2} = \frac{3328\ln T}{(\Delta_i^{\min})^2} \cdot \left(\frac{\Delta_i^{\min}}{\Delta^{\min}_i - 2\epsilon}\right)^2 < \frac{29952\ln T}{(\Delta_i^{\min})^2}.$$

Therefore, continuing Eq.~\eqref{eq:ub_pulls_I_before}, for each $i \in \Ical_\epsilon$, we have
\begin{align*}
    \mathbb{E} [n_i(T) \vert \mathcal{E}] < \frac{29952\ln T}{(\Delta_i^{\min})^2}  + (M + 1) \le \frac{29952\ln T}{(\Delta_i^{\min})^2}  + 2M.
\end{align*}
where the second inequality follows from the fact that $M \ge 1$.

It then follows that
\begin{align*}
  \mathbb{E} [n_i(T) \vert \mathcal{E}] \le \order \left(\frac{\ln T}{(\Delta_i^{\min})^2}  + M\right).
\end{align*}

This completes the proof of Lemma~\ref{lem:arm_pulls_I}.
\end{proof}

\begin{proof}[Proof of Lemma~\ref{lem:arm_pulls_IC}]
Let's now turn our attention to arms in $\Ical^C_\epsilon = [K] \setminus \Ical_\epsilon$. 
For each arm $i \in \Ical^C_\epsilon$ and for each player $p \in [M]$ such that $\mu_i^p < \mu_{*}^p$, we seek to bound the expected number of pulls of arm $i$ by $p$ in $T$ rounds, under the assumption that the event $\Ecal$ occurs.
Since the optimal arm(s) may be different for different players, we treat each player separately.

Fix a player $p \in [M]$ and a suboptimal arm $i \in \Ical^C_\epsilon$ such that $\Delta^p_i > 0$. Recall that $n^p_i(t-1)$ is the number of pulls of arm $i$ by player $p$ after $(t-1)$ rounds. We have
\begin{align}
    n^p_i(T) & = \sum_{t=1}^T \mathds{1}\big\{ i^p_t = i\big\} \nonumber \\
    & \le \tau + \sum_{t = \tau + 1}^T \mathds{1} \big\{ i^p_t = i, n^p_i(t-1) > \tau \big\}, \label{eq:lem_2_counting}
\end{align}
where $\tau \ge 1$ is an arbitrary integer. It then follows that
\[
n^p_i(T) \le \tau + \sum_{t = \tau + 1}^T \mathds{1} \big\{ \text{UCB}^p_{i_*^p}(t) \le \text{UCB}^p_i(t), n^p_i(t-1) > \tau \big\}.
\]

With foresight, let $\tau = \lceil \frac{3328\ln T}{(\Delta^p_i)^2} \rceil$. Conditional on $\mathcal{E}$, we show that, \textit{for any $i \in \Ical^C_\epsilon$ such that $\Delta_i^p > 0$}, the event $\big\{ \text{UCB}^p_{i_*^p}(t) \le \text{UCB}^p_i(t), n^p_i(t-1) > \tau \big\}$ 
never happens. It suffices to show that if $n^p_i(t-1) > \tau$,
\begin{equation} 
\UCB^p_{i_*^p}(t) \geq \mu_*^p, 
\label{eqn:i-star-large}
\end{equation}
and
\begin{equation} \UCB^p_i(t) < \mu_*^p
\label{eqn:i-small}
\end{equation}
happen simultaneously. 

Eq.~\eqref{eqn:i-star-large} follows straightforwardly from the definition of $\Ecal$ along with Lemma~\ref{lem:kappa}. For Eq.~\eqref{eqn:i-small}, we have the following upper bound on $\UCB_i^p(t)$:
\begin{align}
    \UCB_i^p(t)
    &= \kappa_i^p(t, \lambda^*) + F(\wbar{n^p_i}, \wbar{m^p_i}, \lambda^*, \epsilon) \notag \\
    & \leq \mu_i^p + 2 F(\wbar{n^p_i}, \wbar{m^p_i}, \lambda^*, \epsilon) \notag \\
    &= \mu_i^p + 2\Big[\min_{\lambda \in [0,1]} 8\sqrt{13\ln T[\frac{\lambda^2}{\wbar{n^p_i}(t-1)} + \frac{(1-\lambda)^2}{\wbar{m^p_i}(t-1)}]} + (1-\lambda)\epsilon\Big] \nonumber \\
    & \le \mu_i^p +2\Big[8\sqrt{\frac{13\ln T}{\wbar{n^p_i}(t-1)}}\Big] \nonumber \\
    & \le \mu_i^p +2\Big[8\sqrt{\frac{13\ln T}{n^p_i(t-1)}}\Big] \nonumber \\
    & < \mu_i^p + 2\Big[8\sqrt{\frac{13\ln T (\Delta^p_i)^2}{3328\ln T}} \Big] = \mu_i^p + \Delta^p_i = \mu^p_*, \nonumber
\end{align}
where the first inequality is from the definition of event $\mathcal{E}$ and Lemma~\ref{lem:kappa}; the second inequality is from choosing $\lambda = 1$; 
the third inequality uses the basic fact that $n^p_i(t-1) \le \wbar{n^p_i}(t-1)$;
the fourth inequality is by our premise that $n^p_i(t-1) > \tau \geq \frac{3328\ln T}{(\Delta^p_i)^2}$.

It follows that conditional on $\mathcal{E}$, the second term in Eq.~\eqref{eq:lem_2_counting} is always zero, i.e., player $p$ would not pull arm $i$ again. Therefore, for any $i \in \Ical^C_\epsilon$ such that $\Delta_i^p > 0$, we have
\begin{align}
    \mathbb{E} [n^p_i(T) \vert \mathcal{E}] \le \lceil \frac{3328\ln T}{(\Delta^p_i)^2} \rceil \le  \frac{3328\ln T}{(\Delta^p_i)^2} + 1 \le \frac{3328\ln T}{(\Delta^p_i)^2} \cdot 2 = \frac{6656\ln T}{(\Delta^p_i)^2}. %
\end{align}

It then follows that
\begin{align*}
  \mathbb{E} [n^p_i(T) \vert \mathcal{E}] \le \order \left( \frac{\ln T}{(\Delta^p_i)^2} \right).
\end{align*}

This completes the proof of Lemma~\ref{lem:arm_pulls_IC}.
\end{proof}

\paragraph{Proof of Theorem~\ref{thm:gap_dep_ub}.}
We now prove Theorem~\ref{thm:gap_dep_ub}.
\begin{proof}
We have
\begin{align}
    \EE[\mathcal{R}(T)] & \le \EE[\mathcal{R}(T) \vert \mathcal{E}] + \EE[\mathcal{R}(T) \vert \overline{\mathcal{E}}]\Pr[\overline{\mathcal{E}}] \nonumber \\
    & \le \EE[\mathcal{R}(T) \vert \mathcal{E}] + (TM)\frac{4MK}{T^4} \nonumber \\
    & \le \EE[\mathcal{R}(T) \vert \mathcal{E}] + O(1) \label{eq:expected_regret_given_e}
\end{align}
where the second inequality uses the fact that $\EE[\mathcal{R}(T) \vert \overline{\mathcal{E}}] \le TM$, as the instantaneous regret for each player in each round is bounded by $1$; and the last inequality follows under the premise that $T > \max(M,K)$.

Let $\Delta_i^{\max} = \max_p \Delta_i^p$. We have
\begin{align}
    \EE[\mathcal{R}(T) | \Ecal] &= 
    \sum_{i \in [K]} \sum_{p \in [M]} \mathbb{E} [n^p_i(T) | \Ecal] \cdot \Delta^p_i \nonumber \\
    & \le  \sum_{i \in \Ical_\epsilon} \mathbb{E} [n_i(T) | \Ecal] \cdot \Delta_i^{\max} 
    +  
    \sum_{i \in \Ical^C_\epsilon} \sum_{p \in [M]: \Delta_i^p > 0} \mathbb{E} [n^p_i(T) | \Ecal] \cdot \Delta_i^p, \label{eq:break_down_I_IC}
\end{align}
where the inequality holds because the instantaneous regret for any arm $i$ and any player $p$ is bounded by $\Delta_i^{\max}$.

Now, it follows from Lemma~\ref{lem:arm_pulls_I} that there exists some constant $C_1 > 0$ such that for each $i \in \Ical_\epsilon$,
$$\mathbb{E} [n_i(T) \vert \mathcal{E}] \le  C_1 \left(\frac{\ln T}{(\Delta_i^{\min})^2}  + M\right),
$$
and it follows from Lemma~\ref{lem:arm_pulls_IC} that there exists some constant $C_2 > 0$ such that for each $i \in \Ical^C_\epsilon$ and $p \in [M]$ with $\Delta_i^p > 0$,
$$
\mathbb{E} [n^p_i(T) \vert \mathcal{E}] \le C_2 \left( \frac{\ln T}{(\Delta^p_i)^2} \right).
$$

Then, continuing Eq.~\eqref{eq:break_down_I_IC}, we have
\begin{align}
    \EE[\mathcal{R}(T) | \Ecal] 
    & \le \sum_{i \in \Ical_\epsilon} C_1 \left(\frac{\ln T}{(\Delta_i^{\min})^2}  + M\right) \cdot \Delta_i^{\max} 
    + 
    \sum_{i \in \Ical^C_\epsilon} \sum_{p \in [M]: \Delta_i^p > 0} C_2 \left(\frac{\ln T}{(\Delta^p_i)^2}\right) \cdot \Delta_i^p \nonumber \\
    & \le 2C_1 \sum_{i \in \Ical_\epsilon} \left(\frac{\ln T}{\Delta_i^{\min}}  + M\Delta_i^{\min}\right)
    + 
    C_2 \sum_{i \in \Ical^C_\epsilon} \sum_{p \in [M]: \Delta_i^p > 0} \frac{\ln T}{\Delta^p_i}, \nonumber
\end{align}
where the second inequality follows from item~\ref{item:delta_quotient} of Fact~\ref{fact:Ical_properties} which states that $\forall i \in \Ical_\epsilon, \Delta_i^{\max} < 2\Delta_i^{\min}$.

It then follows from Eq.~\eqref{eq:expected_regret_given_e} that
\begin{align}
    \mathbb{E} [\mathcal{R}(T)] & \le \mathbb{E} [\mathcal{R}(T) | \Ecal] + O(1) \nonumber \\
    & \le O\left(\sum_{i \in \Ical_\epsilon} \left(\frac{\ln T}{\Delta^{\min}_i} +
    M \Delta_i^{\min}\right)
    + 
    \sum_{i \in \Ical^C_\epsilon} \sum_{p \in [M]: \Delta_i^p > 0} \frac{\ln T}{\Delta^p_i}
     \right), \nonumber
\end{align}

This completes the proof of Theorem~\ref{thm:gap_dep_ub}.
\end{proof}

\subsection{Proof of Theorem~\ref{thm:gap_ind_upper}}
\label{appendix:gap_ind_ub}

We first restate \Cref{thm:gap_ind_upper}.
\newtheorem*{T2}{Theorem~\ref{thm:gap_ind_upper}}
\begin{T2}
Let $\robustagg(\epsilon)$ run on an $\epsilon$-MPMAB problem instance for $T$ rounds. Then its expected collective regret satisfies
\[ 
\EE[\mathcal{R}(T)] \le
\tilde{\order}\rbr{
\sqrt{\abr{\Ical_\epsilon} M T}
+ 
M \sqrt{(\abr{\Ical^C_\epsilon} - 1) T}
+ 
M \abr{\Ical_\epsilon}
}.
\]
\end{T2}

\begin{proof}

From the earlier proof of Theorem~\ref{thm:gap_dep_ub}, we have
\begin{align}
    \EE[\mathcal{R}(T)] \le \EE[\mathcal{R}(T) \vert \mathcal{E}] + 
    O(1). \label{eq:expected_regret_given_e_2}
\end{align}

Recall that $\Delta_i^{\max} = \max_p \Delta_i^p$. We also have
\begin{align}
    \EE[\mathcal{R}(T) \vert \Ecal] &= 
    \sum_{i \in [K]} \sum_{p \in [M]} \mathbb{E} [n^p_i(T) | \Ecal] \cdot \Delta^p_i \nonumber \\
    & \le  \sum_{i \in \Ical_\epsilon} \mathbb{E} [n_i(T) | \Ecal] \cdot \Delta_i^{\max} 
    +  
    \sum_{i \in \Ical^C_\epsilon} \sum_{p \in [M]: \Delta_i^p > 0} \mathbb{E} [n^p_i(T) | \Ecal] \cdot \Delta_i^p  \label{eq:gap_dep_ub_two_terms}
\end{align}

Again, it follows from Lemma~\ref{lem:arm_pulls_I} that there exists some constant $C_1 > 0$ such that for each $i \in \Ical_\epsilon$,
\begin{align}
    \mathbb{E} [n_i(T) \vert \mathcal{E}] \le  C_1 \left(\frac{\ln T}{(\Delta_i^{\min})^2}  + M\right), \label{eq:C_1_thm_gap_ind}
\end{align}
and it follows from Lemma~\ref{lem:arm_pulls_IC} that there exists some constant $C_2 > 0$ such that for each $i \in \Ical^C_\epsilon$ and $p \in [M]$ with $\Delta_i^p > 0$, 
\begin{align}
    \mathbb{E} [n^p_i(T) \vert \mathcal{E}] \le C_2 \left( \frac{\ln T}{(\Delta^p_i)^2} \right). \label{eq:C_2_thm_gap_ind}
\end{align}

Now let us bound the two terms in Eq.~\eqref{eq:gap_dep_ub_two_terms} separately, using the technique from \citep[Theorem 7.2]{lattimore2020bandit}.

For the first term, with foresight, let us set $\delta_1 = \sqrt{\frac{C_1|\Ical_\epsilon| \ln T}{MT}}$. If $|\Ical_\epsilon| = 0$, we have $\sum_{i \in \Ical_\epsilon} \mathbb{E} [n_i(T) \vert \Ecal] \cdot \Delta_i^{\max} = 0$ trivially. Otherwise, $\delta_1 > 0$ because $T > \max(M,K) \ge 1$. Then, we have
\begin{align}
    & \sum_{i \in \Ical_\epsilon} \mathbb{E} [n_i(T) \vert \Ecal] \cdot \Delta_i^{\max} \nonumber \\
    \le & \ 2 \sum_{i \in \Ical_\epsilon} \mathbb{E} [n_i(T) \vert \Ecal] \cdot \Delta_i^{\min} \nonumber \\
    \le & \ 2 \left( 
    \sum_{i \in \Ical_\epsilon: \Delta_i^{\min} \in (0, \delta_1)} 
    \mathbb{E} [n_i(T) | \Ecal] \cdot \Delta_i^{\min} 
    + 
    \sum_{i \in \Ical_\epsilon: \Delta_i^{\min} \in [\delta_1, 1]} 
    \mathbb{E} [n_i(T) | \Ecal] \cdot \Delta_i^{\min}
    \right) \nonumber \\
    \le & \ 2 \left(MT\delta_1 
    + 
    \sum_{i \in \Ical_\epsilon: \Delta_i^{\min} \in [\delta_1, 1]} 
    C_1 \left(\frac{\ln T}{(\Delta_i^{\min})^2}  + M \right) \Delta_i^{\min} 
    \right)\nonumber \\
    \le & \ 2
    \left(MT\delta_1 
    + 
    \sum_{i \in \Ical_\epsilon: \Delta_i^{\min} \in [\delta_1, 1]} 
    \frac{C_1 \ln T}{\Delta_i^{\min}}  
    +
    C_1 \sum_{i \in \Ical_\epsilon: \Delta_i^{\min} \in [\delta_1, 1]} 
    M\Delta_i^{\min} 
    \right) \nonumber \\
    \le & \ 2
    \left(MT\delta_1 + \frac{C_1|\Ical_\epsilon|\ln T}{\delta_1} + C_1 \sum_{i \in \Ical_\epsilon} M\Delta_i^{\min} 
    \right) \nonumber \\
    \le & \ 4\sqrt{C_1 |\Ical_\epsilon|MT\ln T} + 2C_1M |\Ical_\epsilon|, \label{eq:gap_dep_term_1}
\end{align}
where the first inequality follows from item~\ref{item:delta_quotient} of Fact~\ref{fact:Ical_properties}; the third inequality follows from Eq.~\eqref{eq:C_1_thm_gap_ind} and the fact that $\sum_{i \in \Ical_\epsilon} \mathbb{E} [n_i(T) \vert \Ecal] \le MT$ as $M$ players each pulls one arm in each of $T$ rounds; and the last inequality follows from our premise that $\delta_1 = \sqrt{\frac{C_1|\Ical_\epsilon| \ln T}{MT}}$.

For the second term, we consider two cases:
\paragraph{Case 1: $|\Ical^C_\epsilon| = 1$.} In this case, as we have discussed in the paper, $\Ical^C_\epsilon$ is a singleton set $\cbr{i_*}$ where arm $i_*$ is optimal for all players $p$; that is, $\Delta_{i_*}^p = 0$ for all $p \in [M]$. We therefore have
\begin{equation}
\sum_{i \in \Ical^C_\epsilon} \sum_{p \in [M]: \Delta_i^p > 0} \mathbb{E} [n^p_i(T) | \Ecal] \cdot \Delta_i^p 
= 
0 
=
4M \sqrt{C_2(\abs{\Ical_\epsilon^C}-1) T \ln T}.
\label{eq:gap_dep_term_2_1}
\end{equation}

\paragraph{Case 2: $|\Ical^C_\epsilon| \geq  2$.} With foresight, let us set $\delta_2 = \sqrt{\frac{C_2|\Ical^C_\epsilon| \ln T}{T}}$.
\begin{align}
    & \sum_{i \in \Ical^C_\epsilon} \sum_{p \in [M]: \Delta_i^p > 0} \mathbb{E} [n^p_i(T) | \Ecal] \cdot \Delta_i^p \nonumber \\
    \le &  \sum_{i \in \Ical^C_\epsilon} \sum_{p \in [M]: \Delta_i^p \in (0, \delta_2)} \mathbb{E} [n^p_i(T) | \Ecal] \cdot \Delta_i^p 
    + 
    \sum_{i \in \Ical^C_\epsilon} \sum_{p \in [M]: \Delta_i^p \in [\delta_2, 1]} \mathbb{E} [n^p_i(T) | \Ecal] \cdot \Delta_i^p\nonumber \\
    \le & MT\delta_2 
    + 
    \sum_{i \in \Ical^C_\epsilon} \sum_{p \in [M]: \Delta_i^p \in [\delta_2, 1]} C_2 \left(\frac{\ln T}{(\Delta^p_i)^2} \right) \Delta_i^p \nonumber \\
     \le & MT\delta_2 
    + 
    \sum_{i \in \Ical^C_\epsilon} \sum_{p \in [M]: \Delta_i^p \in [\delta_2, 1]} \left(\frac{C_2 \ln T}{\Delta^p_i} \right) \nonumber \\
    \le & MT\delta_2 + \frac{C_2 M |\Ical^C_\epsilon|\ln T}{\delta_2} \nonumber \\
    \le & 4M  \sqrt{C_2(|\Ical^C_\epsilon|-1) T\ln T}, \label{eq:gap_dep_term_2_2}
\end{align}
where the second inequality follows from Eq.~\eqref{eq:C_2_thm_gap_ind} and the fact that $\sum_{i \in \Ical^C_\epsilon} \mathbb{E} [n_i(T) \vert \Ecal] \le MT$ as $M$ players each pulls one arm in each of $T$ rounds; and the last inequality follows from our premise that $\delta_2 = \sqrt{\frac{C_2|\Ical^C_\epsilon| \ln T}{T}}$ and $|\Ical^C_\epsilon| \leq 2(|\Ical^C_\epsilon| - 1)$.

In summary, from Eqs.~\eqref{eq:gap_dep_term_2_1} and~\eqref{eq:gap_dep_term_2_2}, we have in both cases, 
\begin{equation}
    \sum_{i \in \Ical^C_\epsilon} \sum_{p \in [M]: \Delta_i^p > 0} \mathbb{E} [n^p_i(T) | \Ecal] \cdot \Delta_i^p
    \leq
    4M  \sqrt{C_2(|\Ical^C_\epsilon|-1) T\ln T}.
    \label{eq:gap_dep_term_2}
\end{equation}
Combining Eq.~\eqref{eq:gap_dep_term_1} and Eq.~\eqref{eq:gap_dep_term_2}, we have
\begin{align}
    \EE[\mathcal{R}(T) \vert \Ecal] & \le 4\sqrt{C_1 |\Ical_\epsilon|MT\ln T} + 2C_1 M |\Ical_\epsilon| + 4M  \sqrt{C_2(|\Ical^C_\epsilon| - 1)T\ln T} \nonumber
\end{align}

It then follows from Eq.~\eqref{eq:expected_regret_given_e_2} that
\begin{align*}
    \EE[\mathcal{R}(T)] & 
    \le 4\sqrt{C_1 |\Ical_\epsilon|MT\ln T} 
    + 
    4M  \sqrt{C_2(|\Ical^C_\epsilon| - 1)T\ln T} 
    +
    2C_1 M |\Ical_\epsilon| 
    + 
    O(1) \\
    & \le \order\rbr{
    \sqrt{\abr{\Ical_\epsilon} M T \ln T}
    + 
    M \sqrt{(|\Ical^C_\epsilon| - 1) T \ln T}
    + 
    M \abr{\Ical_\epsilon}
    }\\
    & \le \tilde{\order}\rbr{
    \sqrt{\abr{\Ical_\epsilon} M T}
    + 
    M \sqrt{(|\Ical^C_\epsilon| - 1) T}
    + 
    M \abr{\Ical_\epsilon}
    }.
\end{align*}
This completes the proof of Theorem~\ref{thm:gap_ind_upper}.
\end{proof}

\section{Proof of the lower bounds}
\label{appendix:lb}

\subsection{Gap-independent lower bound with known $\epsilon$}

We first restate Theorem~\ref{thm:gap_ind_lower}.
\newtheorem*{T3}{Theorem~\ref{thm:gap_ind_lower}}
\begin{T3}
For any $K \geq 2, M, T \in \NN$ such that $T \geq K$, and $l, l^C$ in $\NN$ such that $l \leq K-1, l + l^C = K$, there exists some $\epsilon > 0$, such that for any algorithm $\Acal$, there exists an $\epsilon$-MPMAB problem instance, in which 
$\abs{\Ical_\epsilon} \geq l$, and $\Acal$ has a collective regret at least
$\Omega(M\sqrt{(l^C-1) T} + \sqrt{M l T})$.
\end{T3}

\begin{proof}
Fix algorithm $\Acal$.
We consider two cases regarding the comparison between $l$ and $M (\lc-1)$.
\paragraph{Case 1: $l > M (\lc-1)$.}

To simplify notations, define $\Delta = \sqrt{\frac{l+1}{24 M T}}$. 
Observe that $\Delta \leq \frac14$ as $T \geq K \geq l+1$.
We will set $\epsilon = 0$.

We will now define $l$ different Bernoulli $\epsilon$-MPMAB instances, and show that under at least one of them, $\Acal$ will have a collective regret at least $\frac1{96} \sqrt{M l T} \geq \frac1{192} (M\sqrt{(\lc-1) T} + \sqrt{MlT})$.

For $j$ in $[l+1]$, define a Bernoulli MPMAB instance $E_{j}$ to be such that for all players $p$ in $[M]$, the expected reward of arm $i$,
\[
\mu_i^p 
= 
\begin{cases} 
\frac12 + \Delta & i = j \\
\frac12 & i \in [l+1] \setminus \cbr{j} \\
0  & i \notin [l+1]
\end{cases}
.
\]

We first verify that for every instance $E_{j}$, it (1) is an $\epsilon$-MPMAB instance, and (2) $[l+1] \setminus \cbr{j} \subseteq \Ical_\epsilon$ and therefore $\Ical_\epsilon$ has size $\geq l$:
\begin{enumerate}
    \item For item (1), observe that for any fixed $i$, we have $\mu_i^p$ share the same value across all player $p$'s. Therefore, the is trivially $\epsilon$-dissimilar.
    
    \item For item (2), for all $i$ in $[l+1] \setminus \cbr{j}$, we have $\Delta_i^p = \Delta > 5\epsilon = 0$ for all $p$; this implies that $[l+1] \setminus \cbr{j}$ is a subset of $\Ical_{\epsilon}$.

\end{enumerate}

We will now argue that 
\begin{equation} 
\EE_{j \sim \Unif([l+1])} \EE_{E_j} \sbr{ \Rcal(T) } \geq \frac{1}{96} \sqrt{M l T}.
\label{eqn:reg-collective-1}
\end{equation}

To this end, it suffices to show
\begin{equation}
\EE_{j \sim \Unif([l+1])} \EE_{E_j} \sbr{ MT - n_j(T) } \geq \frac{MT} 4.
\label{eqn:armpull-1}
\end{equation}

To see why Eq.~\eqref{eqn:armpull-1} implies Eq.~\eqref{eqn:reg-collective-1}, recall that under instance $E_j$,
$j$ is the optimal arm for all players.
In this instance, $\Rcal(T) = \sum_{i \neq j} n_i(T) \Delta_i^1$.
As under $E_j$, for all $i \neq j$ and all $p$, $\Delta_i^1 \geq \frac\Delta4$, we have $\Rcal(T) \geq \frac\Delta4 \cdot (MT - n_j(T))$. Eq.~\eqref{eqn:reg-collective-1} follows from combining this inequality with Eq.~\eqref{eqn:armpull-1}, along with some algebra.

We now come back to the proof of Eq.~\eqref{eqn:armpull-1}. 
First, we define a helper instance $E_0$, such that for all players $p$ in $[M]$, the expected reward of arm $i$ is defined as:
\[
\mu_i^p 
= 
\begin{cases} 
\frac12 & i \in [l+1] \\
0 & i \notin [l+1]
\end{cases}
\]

In addition, for all $i$ in $\cbr{0} \cup [l+1]$, define $\PP_i$ as the joint distribution of the interaction logs (arm pulls and rewards) for all $M$ players over a horizon of $T$; furthermore,  denote by $\EE_i$ expectation with respect to $\PP_i$. 

For every $i$ in $[l+1]$, we have
\begin{align*}
d_{\TV}(\PP_{0}, \PP_{i}) 
& =
\frac12 \| \PP_{0} - \PP_{i} \|_1 \nonumber \\
& \leq
\frac12 \sqrt{2\KL(\PP_0, \PP_i)} \\
& \leq 
\frac12 \sqrt{2 \KL(\Ber(0.5, 0.5+\Delta)) \EE_0[n_i(T)]} \\
&\leq
\sqrt{\frac{3}{2} \EE_0[n_i(T)] \Delta^2} \\
&=
\frac14 \sqrt{ \frac{l+1}{MT} \EE_0[n_i(T)] }
\end{align*}
where the first equality is from  $d_{\TV}(\PP, \QQ) = \frac12 \| \PP - \QQ \|_1$ for any two distributions $\PP$, $\QQ$; 
the first inequality uses Pinsker's inequality; the second inequality is from the well-known divergence decomposition lemma (e.g.~\cite{lattimore2020bandit}, Lemma 15.1);
the third inequality uses Lemma~\ref{lem:kl-ber}; and the last equality is by recalling that $\Delta \in [0,\frac14]$ and algebra.

Now, applying Lemma~\ref{lem:kl-pull} with $m = l+1 \geq 2$, $N_i = n_i(T)$ for all $i$ in $[l+1]$, and $B = MT$, Eq.~\eqref{eqn:armpull-1} is proved. This in turn finishes the proof of the regret lower bound.

\paragraph{Case 2: $M (\lc - 1) \geq l$.} 

To simplify notations, define $\Delta = \sqrt{\frac{\lc}{24 T}}$. 
Observe that $\Delta < \frac14$ as $T \geq K \geq \lc$. In addition, we must have $\lc \geq 2$ in this case, as if $\lc = 1$, $M(\lc - 1) = 0 < K = l$.
We set $\epsilon = \frac{\Delta}{2}$. 

We will now define $[l^C]^M$ different Bernoulli $\epsilon$-MPMAB instances, and show that under at least one of them, $\Acal$ will have a collective regret at least $\frac1{24} M\sqrt{\lc T} \geq \frac1{192} (M\sqrt{(\lc-1) T} + \sqrt{MlT})$.

For $i_1, \ldots, i_M \in [l^C]^M$, define Bernoulli MPMAB instance $E_{i_1,\ldots,i_M}$ to be such that for $p$ in $[M]$ and $i$ in $[K]$, the expected reward of player $p$ on pulling arm $i$ is
\[
\mu_i^p 
= 
\begin{cases} 
\frac12 + \Delta & i = i_p \\
\frac12 & i \in [l^C] \setminus \cbr{i_p} \\
0 & i \notin [l^C]
\end{cases}
\]

We first verify that for every $i_1, \ldots, i_M$,
instance $E_{i_1,\ldots,i_M}$ (1) is an $\epsilon$-MPMAB instance, and (2)  $[K] \setminus [l^C] \subset \Ical_{\epsilon}$, and therefore, $\Ical_{\epsilon}$ has size $\geq l$:
\begin{enumerate}
\item For item (1), observe that for all $i$ in $[l^C]$ and all $p$ in $[M]$, $\mu_i^p \in \cbr{\frac12, \frac12 + \Delta}$; therefore, for every $p, q$, $\abs{\mu_i^p - \mu_i^q} \leq \Delta = \epsilon$. Meanwhile, for all $i$ in $[K] \setminus [l^C]$ and all $p$ in $[M]$, $\mu_i^p = 0$, implying that for every $p,q$, $\abs{\mu_i^p - \mu_i^q} = 0 \leq \epsilon$. Therefore $E_{i_1,\ldots,i_M}$ is $\epsilon$-dissimilar. 
\item For item (2), for all $i$ in $[K] \setminus [l^C]$ and all $p$, $\Delta_i^p = \frac12 + \Delta > \frac 5 2 \Delta = 5\epsilon$. This implies that all elements of $[K] \setminus [l^C]$ are in $\Ical_{\epsilon}$. %
\end{enumerate}

We will now argue that 
\[ 
\EE_{(i_1,\ldots,i_M) \sim \Unif([l^C]^M)} \EE_{E_{i_1,\ldots,i_M}} \sbr{\Rcal(T)} \geq  \frac{M \sqrt{l^C T}}{24}.
\]
As the roles of all $M$ players are the same, by symmetry, it suffices to show that the expected regret of player 1 satisfies
\begin{equation} 
\EE_{(i_1,\ldots,i_M) \sim \Unif([l^C]^M)} \EE_{E_{i_1,\ldots,i_M}} \sbr{\Rcal^1(T)} \geq \frac{\sqrt{l^C T}}{24}.
\label{eqn:reg-1-2}
\end{equation}

It therefore suffices to show, 
\begin{equation}
\EE_{(i_1,\ldots,i_M) \sim \Unif([l^C]^M)} \EE_{E_{i_1,\ldots,i_M}} \sbr{T - n_{i_1}^1(T)} \geq \frac T 4.
\label{eqn:armpull-1-2}
\end{equation}
This is because, recall that when $i_1$ is the optimal arm for player 1,
$\Rcal(T) = \sum_{i=1}^K n_i^1(T) \Delta_i^1 = \sum_{i \neq i_1} n_i^1(T) \Delta_i^1$; in addition, for all $i \neq i_1$, $\Delta_i^1 \geq \Delta$. This implies that $\Rcal^1(T) \geq \Delta (T - n_{i_1}^1(T))$. Eq.~\eqref{eqn:reg-1-2} follows from the above inequality,  Eq.~\eqref{eqn:armpull-1-2}, and the definition of $\Delta$.

We now come back to the proof of Eq.~\eqref{eqn:armpull-1-2}.
To this end, we define the following set of  ``helper'' instances to facilitate our reasoning.
Given $i_2, \ldots, i_M \in [K]^{M-1}$,
define instance 
 $E_{0,i_2,\ldots,i_M}$ such that its reward distribution is identical to $E_{i_1,i_2,\ldots,i_M}$ except for player $1$ on arm $i_1$. Formally,
it has the following expected reward profile:
\[
\text{for }p = 1,
\mu_i^1 
= 
\begin{cases} 
\frac12 & i \in [l^C] \\
0 & i \notin [l^C]
\end{cases}
\quad 
\text{for }p \neq 1,
\mu_i^p 
= 
\begin{cases} 
\frac12 + \Delta & i = i_p \\
\frac12 & i \in [l^C] \setminus \cbr{i_p} \\
0 & i \notin [l^C]
\end{cases}
\]

In addition, for all $i_1, \ldots, i_M$ in $(\cbr{0} \cup [\lc]) \times [\lc]^{M-1}$, define $\PP_{i_1,\ldots, i_M}$ as the joint distribution of the interaction logs (arm pulls and rewards) for all $M$ players over a horizon of $T$; furthermore, for $i$ in $(\cbr{0} \cup [\lc])$, define $\PP_i = \frac{1}{{(\lc)}^{M-1}} \sum_{i_2, \ldots, i_M \in [\lc]^{M-1}} \PP_{i,i_2,\ldots,i_M}$, and denote by $\EE_i$ expectation with respect to $\PP_i$. In this notation, Eq.~\eqref{eqn:armpull-1-2}
 can be rewritten as
\[
\frac{1}{\lc} \sum_{i=1}^{\lc} \EE_i[T - n_i^1(T)] \geq \frac T 2. 
\]

For every $i$ in $[\lc]$, 
\begin{align*}
d_{\TV}(\PP_{0}, \PP_{i}) 
& =
\frac12 \| \PP_{0} - \PP_{i} \|_1 \nonumber \\
& = 
\frac12 \norm{ \frac{1}{{\lc}^{M-1}} \sum_{i_2, \ldots, i_M \in [\lc]^{M-1}} ( \PP_{0,i_2,\ldots,i_M} -  \PP_{i,i_2,\ldots,i_M} ) }_1 
\nonumber \\
& \leq  
\frac{1}{{(\lc)}^{M-1}} \cdot \sum_{i_2,\ldots,i_M \in [\lc]^{M-1}} \frac12 \| \PP_{0,i_2,\ldots,i_M} -   \PP_{i,i_2,\ldots,i_M} \|_1 \\
&\leq
\frac{1}{{(\lc)}^{M-1}} \cdot \sum_{i_2,\ldots,i_M \in [\lc]^{M-1}}
\sqrt{\frac12 \KL(\Ber(0.5, 0.5+\Delta)) \cdot \EE_{0,i_2,\ldots,i_M} [N_i^1(T)]} \\
&\leq
\frac{1}{{(\lc)}^{M-1}} \cdot \sum_{i_2,\ldots,i_M \in [\lc]^{M-1}}
\sqrt{\frac32 \Delta^2 \cdot \EE_{0,i_2,\ldots,i_M} [N_i^1(T)]} \\
& \leq
\sqrt{\frac32 \Delta^2 \cdot \EE_{0} [N_i^1(T)]}
\end{align*}
where the first equality is from  $d_{\TV}(\PP, \QQ) = \frac12 \| \PP - \QQ \|_1$ for any two distributions $\PP$, $\QQ$; the second equality is from the definition of $\PP_i$, $i \in \cbr{0} \cup [\lc]$; the first inequality is from triangle inequality of $\ell_1$ norm; the second inequality is from Pinsker's inequality, and the divergence decomposition lemma (\cite{lattimore2020bandit}, Lemma 15.1);
the third inequality is from Lemma~\ref{lem:kl-ber} and recalling that $\Delta \in [0,\frac14]$;
the last inequality is from Jensen's inequality, and the definition of $\PP_0$.

Applying Lemma~\ref{lem:kl-pull} with $m = \lc \geq 2$, $N_i = n_i^1(T)$ for all $i$ in $[\lc]$ and $B = T$,
Eq.~\eqref{eqn:armpull-1-2} is proved. This in turn finishes the proof of the regret lower bound.
\end{proof}

\subsection{Gap-dependent lower bounds with known $\epsilon$}

We restate Theorem~\ref{thm:gap_dep_lb} here with specifications of exact constants in the lower bound.

\begin{theorem}[Restatement of  Theorem~\ref{thm:gap_dep_lb}]
\label{thm:gap_dep_lb_gen}
Fix $\epsilon \geq 0$ and $\alpha, C > 0$. Let $\Acal$ be an algorithm such that $\Acal$ has at most $C T^{1-\alpha}$ 
regret in all $\epsilon$-MPMAB problem instances.
Then, for any Bernoulli $\frac\epsilon2$-MPMAB instance $\mu = (\mu_i^p)_{i \in [K], p \in [M]}$ such that $\mu_i^p \in [\frac {15}{32}, \frac{17}{32}]$ for all $i$ and $p$, we have:
\begin{align*}
\EE_\mu[\Rcal(T)]
\geq
\sum_{i \in \Ical_{\epsilon/20}^C}
\sum_{p \in [M]: \Delta_i^p > 0} \frac{\ln\del{\Delta_{i}^{p}T^{\alpha}/8C}}{12 \Delta_i^p}
+
\sum_{i \in \Ical_{\epsilon/20}: \Delta_i^{\min} > 0}
\frac{\ln\del{\Delta_{i}^{\min}T^{\alpha} / 8C}}{12 \Delta_i^{\min}}.
\end{align*}
\end{theorem}

\begin{proof}
We will first prove the following two claims:
\begin{enumerate}
    \item For any $i_0$ in $[K]$ such that $\Delta_{i_0}^{\min} > 0$,
    $\EE_\mu\sbr{n_{i_0}(T)} \geq \frac{\ln\del{\Delta_{i_0}^{\min}T^{\alpha} / 8C}}{12 (\Delta_{i_0}^{\min})^2}$.
    \label{item:generic-lb}
    \item For any $i_0$ in $\Ical_{\epsilon/20}^C$ and any $p_0$ in $[M]$ such that $\Delta_{i_0}^{p_0} > 0$,
    $\EE_\mu\sbr{n_{i_0}^{p_0}(T)} \geq \frac{\ln\del{\Delta_{i_0}^{p_0}T^{\alpha}/8C}}{12 (\Delta_{i_0}^{p_0})^2}$.
    \label{item:nearopt-lb}
\end{enumerate}

The proof of these two claims are as follows:
\begin{enumerate}
    \item Fix $i_0$ in $[K]$ such that $\Delta_{i_0}^{\min} > 0$, i.e., $\Delta_{i_0}^p > 0$ for all $p$ in $[M]$. %
    Define $p_0 = \argmin_{p \in [M]} \Delta_{i_0}^p$.
    
    We consider a new Bernoulli MPMAB instance, with mean reward defined as follows:
    \[
    \forall p \in [M], 
    \quad
    \nu_i^p
    =
    \begin{cases}
    \mu_i^p + 2 \Delta_{i_0}^{p_0}, & i = i_0, \\
    \mu_i^p & \text{otherwise}
    \end{cases}
    \]
    We have the following key observations:
    \begin{enumerate}
        \item $\nu$ is an $\epsilon$-MPMAB instance; this is because $\nu_i^p - \nu_i^q = \mu_i^p - \mu_i^q$ for any $p, q$ in $[M]$ and $i$ in $[K]$, and $\mu$ is an $\frac\epsilon2$-MPMAB instance. 
        By our assumption that $\Acal$ has $C T^{1-\alpha}$ regret on all $\epsilon$-MPMAB environments, we have
        \begin{equation} 
        \EE_{\mu} \sbr{\Rcal(T)} \leq C T^{1-\alpha}, 
        \quad
        \EE_{\nu} \sbr{\Rcal(T)} \leq C T^{1-\alpha}. 
        \label{eqn:sublinear-reg-1}
        \end{equation}
        
        \item By the divergence decomposition lemma (\cite{lattimore2020bandit}, Lemma 15.1),
        \begin{equation} 
        \KL(\PP_\mu, \PP_\nu) 
        = 
        \sum_{p=1}^M \EE_{\mu} \sbr{n_{i_0}^p(T)} \KL\del{\Ber(\mu_{i_0}^p), \Ber(\mu_{i_0}^p + 2\Delta_{i_0}^{p_0})},
        \label{eqn:div-decomp-1}
        \end{equation}
        As for all $p$, $\mu_{i_0}^p \in [\frac{15}{32}, \frac{17}{32}]$, and $\Delta_{i_0}^p \leq \frac{1}{16}$, using Lemma~\ref{lem:kl-ber}, we have that for all $p$,
        $$\KL\del{\Ber(\mu_i^p), \Ber(\mu_i^p + 2\Delta_{i_0}^{p_0})} \leq 3 \cdot (2\Delta_{i_0}^{p_0})^2 = 12 (\Delta_{i_0}^{p_0})^2.$$ Plugging into Eq.~\eqref{eqn:div-decomp-1}, we get
        \begin{equation}
           \KL(\PP_\mu, \PP_\nu)  
           \leq
           \sum_{p=1}^M ( \EE_{\mu} 
            \sbr{n_{i_0}^p(T)} 
             \cdot
             12 (\Delta_{i_0}^{p_0})^2 )
           =
           12 \EE_{\mu} 
           \sbr{n_{i_0}(T)}  (\Delta_{i_0}^{p_0})^2.
           \label{eqn:div-decomp-ineq-1}
        \end{equation}

    \item 
    Under MPMAB instance $\mu$, $\Rcal(T) \geq \Rcal^{p_0}(T) \geq \Delta_{i_0}^{p_0} n_{i_0}^{p_0}(T) \geq \frac{\Delta_{i_0}^{p_0} T}{2} \mathds{1} \{n_{i_0}^{p_0}(T) \geq \frac T 2\}$. Taking expectations, we get,
    \begin{equation}
    \EE_\mu \sbr{\Rcal(T)} \geq \frac{ \Delta_{i_0}^{p_0} T}{2} \PP_\mu( n_{i_0}^{p_0}(T) \geq \frac T 2). 
    \label{eqn:reg-lb-mu-1}
    \end{equation}
    
    Likewise, under MPMAB instance $\nu$, for player $p_0$, $\Rcal(T) \geq \Rcal^{p_0}(T) \geq \Delta_{i_0}^{p_0} (T - n_{i_0}^{p_0}(T)) \geq \frac{\Delta_{i_0}^{p_0} T}{2} \mathds{1} \{n_{i_0}^{p_0}(T) < \frac T 2\}$. Taking expectations, we get,
    \begin{equation}
    \EE_\nu \sbr{\Rcal(T)} \geq \frac{\Delta_{i_0}^{p_0} T}{2} \PP_\nu( n_{i_0}^{p_0}(T) < \frac T 2).
    \label{eqn:reg-lb-nu-1}
    \end{equation}
    \end{enumerate}
   
    Adding up Eq.~\eqref{eqn:reg-lb-mu-1} and Eq.~\eqref{eqn:reg-lb-nu-1}, we have 
    \begin{equation}
    \EE_\nu \sbr{\Rcal(T)} 
    + \EE_\mu \sbr{\Rcal(T)}
    \geq 
    \frac{\Delta_{i_0}^{p_0} T}{2} 
    \del{ 
    \PP_\nu( n_{i_0}^{p_0}(T) < \frac T 2)
    +
    \PP_\mu( n_{i_0}^{p_0}(T) \geq \frac T 2)
    }.
    \label{eqn:combined-1}
    \end{equation}
    
    From Eq.~\eqref{eqn:sublinear-reg-1}, we have the left hand side is at most $2 C T^{1-\alpha}$. 
    By Bretagnolle-Huber inequality (see Lemma~\ref{lem:bh}) and Eq.~\eqref{eqn:div-decomp-ineq-1}, we have that 
    $\PP_\nu( n_{i_0}^{p_0}(T) < \frac T 2)
    +
    \PP_\mu( n_{i_0}^{p_0}(T) \geq \frac T 2)
    \geq 
    \frac12 \exp(-\KL(\PP_\mu, \PP_{\mu})) 
    \geq 
    \frac12 \exp(-12 \EE_\mu\sbr{n_{i_0}(T)} (\Delta_{i_0}^{p_0})^2)
    $.
    Plugging these to Eq.~\eqref{eqn:combined-1}, we get
    \[ 
    2 C T^{1-\alpha}
    \geq
    \frac{\Delta_{i_0}^{p_0} T}4 \exp(-12 \EE_\mu\sbr{n_{i_0}(T)} (\Delta_{i_0}^{p_0})^2).
    \]
    Solving for $\EE_\mu\sbr{n_{i_0}(T)}$, we conclude that
    \[
    \EE_\mu\sbr{n_{i_0}(T)}
    \geq 
    \frac{1}{12 (\Delta_{i_0}^{p_0})^2} \cdot \ln\del{\frac{\Delta_{i_0}^{p_0}T^{\alpha}}{8C}}
    =
    \frac{1}{12 (\Delta_{i_0}^{\min})^2} \cdot \ln\del{\frac{\Delta_{i_0}^{\min}T^{\alpha}}{8C}}.
    \]
    
    \item Fix $i_0$ in $\Ical_{\epsilon/20}^C$ and $p_0 \in [M]$ such that $\Delta_{i_0}^{p_0} > 0$. By definition of $\Ical_{\epsilon/20}^C$, we also have $\Delta_{i_0}^{p_0} = \mu_*^{p_0} - \mu_{i_0}^{p_0} \leq \epsilon / 4$.
    
     We consider a new MPMAB environment $\nu$, with mean reward defined as follows:
    \[
    \nu_i^p 
    =
    \begin{cases}
    \mu_i^p + 2\Delta_{i_0}^{p_0} & i = i_0, p = p_0 \\
    \mu_i^p & \text{otherwise}
    \end{cases}
    \]
    Same as before, we have the following three key observations:
    \begin{enumerate}
        \item $\nu$ is an $\epsilon$-MPMAB instance; this is because $(\nu_i^p - \nu_i^q) - (\mu_i^p - \mu_i^q) \in \cbr{-\frac\epsilon2, 0, \frac\epsilon2}$ for any $p, q$ in $[M]$ and $i$ in $[K]$, and $\mu$ is an $\frac\epsilon2$-MPMAB instance. 
        By our assumption that $\Acal$ has $C T^{1-\alpha}$ regret on all $\epsilon$-MPMAB problem instances, we have
        \begin{equation} 
        \EE_{\mu} \sbr{\Rcal(T)} \leq C T^{1-\alpha}, 
        \quad
        \EE_{\nu} \sbr{\Rcal(T)} \leq C T^{1-\alpha}. 
        \label{eqn:sublinear-reg-2}
        \end{equation}
        
        \item By the divergence decomposition lemma (\cite{lattimore2020bandit}, Lemma 15.1),
        \begin{equation} 
        \KL(\PP_\mu, \PP_\nu) 
        = 
        \EE_{\mu} \sbr{n_{i_0}^{p_0}(T)} \KL\del{\Ber(\mu_{i_0}^{p_0}), \Ber(\mu_{i_0}^{p_0} + 2\Delta_{i_0}^{p_0})}
        \leq
        12 \EE_{\mu} \sbr{n_{i_0}^{p_0}(T)} (\Delta_{i_0}^{p_0})^2
        ,
        \label{eqn:div-decomp-ineq-2}
        \end{equation}
        where the second equality uses the following observation: 
        $\mu_{i_0}^{p_0} \in [\frac{15}{32}, \frac{17}{32}]$, and $\Delta_{i_0}^{p_0} \leq \frac{1}{16}$, using Lemma~\ref{lem:kl-ber},
        $\KL\del{\Ber(\mu_{i_0}^{p_0}), \Ber(\mu_{i_0}^{p_0} + 2\Delta_{i_0}^{p_0})} \leq 3 \cdot (2\Delta_{i_0}^{p_0})^2 = 12 (\Delta_{i_0}^{p_0})^2$. 
        
        \item Under MPMAB instance $\mu$, $\Rcal(T) \geq \Rcal^{p_0}(T) \geq \Delta_{i_0}^{p_0} n_{i_0}^{p_0}(T) \geq \frac{\Delta_{i_0}^{p_0} T}{2} \mathds{1} \{n_{i_0}^{p_0}(T) \geq \frac T 2\}$. Taking expectations, we get,
        \begin{equation}
        \EE_\mu \sbr{\Rcal(T)} \geq \frac{ \Delta_{i_0}^{p_0} T}{2} \PP_\mu( n_{i_0}^{p_0}(T) \geq \frac T 2). 
        \label{eqn:reg-lb-mu-2}
        \end{equation}
        
        Likewise, under MPMAB instance $\nu$, for player $p_0$, $\Rcal(T) \geq \Rcal^{p_0}(T) \geq \Delta_{i_0}^{p_0} (T - n_{i_0}^{p_0}(T)) \geq \frac{\Delta_{i_0}^{p_0} T}{2} \mathds{1} \{n_{i_0}^{p_0}(T) < \frac T 2\}$. Taking expectations, we get,
        \begin{equation}
        \EE_\nu \sbr{\Rcal(T)} \geq \frac{\Delta_{i_0}^{p_0} T}{2} \PP_\nu( n_{i_0}^{p_0}(T) < \frac T 2).
        \label{eqn:reg-lb-nu-2}
        \end{equation}
    \end{enumerate}   
    Same as the proof of item 1, combining Equations~\eqref{eqn:sublinear-reg-2},~\eqref{eqn:div-decomp-ineq-2},~\eqref{eqn:reg-lb-mu-2},~\eqref{eqn:reg-lb-nu-2}, and using Bretagnolle-Huber inequality, we get 
    \[
    \EE_\mu\sbr{n_{i_0}^{p_0}(T)}
    \geq 
    \frac{1}{12 (\Delta_{i_0}^{p_0})^2} \cdot \ln\del{\frac{\Delta_{i_0}^{p_0}T^{\alpha}}{8C}}.
    \]
\end{enumerate}
We now use the above two claims to conclude the proof. Recall that 
$\EE_\mu\sbr{\Rcal(T)} = \sum_{i \in [K]} \sum_{p \in [M]} \Delta_i^p \EE_\mu \sbr{n_i^p(T)}$.
For $i$ in $\Ical_{\epsilon/20}$ such that $\Delta_i^{\min} > 0$, item~\ref{item:generic-lb} implies: \[ 
\sum_{p \in [M]} \Delta_i^p \EE_\mu \sbr{n_i^p(T)} 
\geq
\Delta_i^{\min} \sum_{p \in [M]} \EE_\mu \sbr{n_i^p(T)} 
\geq 
\frac{1}{12 \Delta_{i}^{\min}} \cdot \ln\del{\frac{\Delta_{i}^{\min}T^{\alpha}}{8C}}. %
\]

For $i$ in $\Ical_{\epsilon/20}^C$, item~\ref{item:nearopt-lb} implies:
\[
\sum_{p \in [M]} \Delta_i^p \EE_\mu \sbr{n_i^p(T)} 
\geq
\sum_{p \in [M]: \Delta_i^p > 0} \frac{1}{12 \Delta_{i}^{p}} \cdot \ln\del{\frac{\Delta_{i}^{p} T^{\alpha}}{8C}}.
\]
Summing over all $i$ in $[K]$ on the above two inequalities, we have
\begin{align*}
\EE_\mu\sbr{\Rcal(T)} 
& = 
\sum_{i \in [K]} \sum_{p \in [M]} \Delta_i^p \EE_\mu \sbr{n_i^p(T)} \\
& \geq
\sum_{i \in \Ical_{\epsilon/20}^C} \sum_{p \in [M]: \Delta_i^p > 0} \frac{1}{12 \Delta_{i}^{p}} \cdot \ln\del{\frac{\Delta_{i}^{p} T^{\alpha}}{8C}}
+
\sum_{i \in \Ical_{\epsilon/20}: \Delta_i^{\min} > 0}
\frac{1}{12 \Delta_{i}^{\min}} \cdot \ln\del{\frac{\Delta_{i}^{\min}T^{\alpha}}{8C}} %
.
\qedhere
\end{align*}
\end{proof}
\paragraph{Remark.} The above lower bound argument aligns with our intuition that arms that are near-optimal with respect to some players (i.e., those in $\Ical_{\epsilon/20}^C$) are harder for information sharing: in addition to a lower bound on the collective number of pulls to it across all players (item~\ref{item:generic-lb} of the claim), we are able to show a stronger lower bound on the number of pulls to it from {\em each player} (item~\ref{item:nearopt-lb} of the claim).

\subsection{Gap-dependent lower bounds with unknown $\epsilon$}

We restate Theorem~\ref{thm:gap_dep_lb_unk} here with specifications of exact constants in the lower bound.

\begin{theorem}[Restatement of Theorem~\ref{thm:gap_dep_lb_unk}]
Fix $\alpha, C > 0$. Let $\Acal$ be an algorithm such that $\Acal$ has at most $C T^{1-\alpha}$ 
regret in all MPMAB problem instances. 
Then, for any Bernoulli MPMAB instance $\mu = (\mu_i^p)_{i \in [K], p \in [M]}$ such that $\mu_i^p \in [\frac {15}{32}, \frac{17}{32}]$ for all $i$ and $p$, we have:
\begin{align*}
\EE_\mu[\Rcal(T)]
\geq
\sum_{i \in [K]}
\sum_{p \in [M]: \Delta_i^p > 0} \frac{\ln\del{\Delta_{i}^{p}T^{\alpha}/8C}}{12 \Delta_i^p}.
\end{align*}
\end{theorem}
\begin{proof}
Recall that 
$\EE_\mu\sbr{\Rcal(T)} = \sum_{i=1}^K \sum_{p=1}^M \Delta_i^p \EE_\mu \sbr{N_i^p(T)}$;
it suffices to show that for any $i_0$ in $[K]$ and any $p_0$ in $[M]$ such that $\Delta_{i_0}^{p_0} > 0$,
    $\EE_\mu\sbr{n_{i_0}^{p_0}(T)} \geq \frac{\ln\del{\Delta_{i_0}^{p_0}T^{\alpha}/8C}}{12 (\Delta_{i_0}^{p_0})^2}$.

The proof of this claim is almost identical to the proof of the the second claim in the previous theorem, except that we have more flexibility to choose the ``alternative instances'' $\nu$'s, because $\Acal$ is assumed to have sublinear regret in all MPMAB instances; specifically, $\nu$ no longer needs to be an $\epsilon$-MPMAB instance.
We include the argument here for completeness. 
Fix $i_0$ in $[K]$ and $p_0 \in [M]$ such that $\Delta_{i_0}^{p_0} > 0$. We consider a new Bernoulli MPMAB instance $\nu$, with mean reward defined as follows:
    \[
    \nu_i^p 
    =
    \begin{cases}
    \mu_i^p + 2\Delta_{i_0}^{p_0} & i = i_0, p = p_0 \\
    \mu_i^p & \text{otherwise}
    \end{cases}
    \]
    We have the following three key observations:
    \begin{enumerate}[(a)]
        \item $\nu$ is still a valid Bernoulli MPMAB instance; this is because for all $p$ in $[M]$ and $i$ in $[K]$, $(\nu_i^p - \mu_i^p) \in [-\frac1{8}, \frac1{8}]$, and $\mu_i^p \in [\frac{15}{32}, \frac{17}{32}]$, implying that $\nu_i^p \in [0,1]$. 
        By our assumption that $\Acal$ has $C T^{1-\alpha}$ regret on all Bernoulli MPMAB instances, we have
        \begin{equation} 
        \EE_{\mu} \sbr{\Rcal(T)} \leq C T^{1-\alpha}, 
        \quad
        \EE_{\nu} \sbr{\Rcal(T)} \leq C T^{1-\alpha}. 
        \label{eqn:sublinear-reg-3}
        \end{equation}
        
        \item By the divergence decomposition lemma (\cite{lattimore2020bandit}, Lemma 15.1),
        \begin{equation} 
        \KL(\PP_\mu, \PP_\nu) 
        = 
        \EE_{\mu} \sbr{n_{i_0}^{p_0}(T)} \KL\del{\Ber(\mu_{i_0}^{p_0}), \Ber(\mu_{i_0}^{p_0} + 2\Delta_{i_0}^{p_0})}
        \leq
        12 \EE_{\mu} \sbr{n_{i_0}^{p_0}(T)} (\Delta_{i_0}^{p_0})^2
        ,
        \label{eqn:div-decomp-ineq-3}
        \end{equation}
        where the second equality uses the following observation: 
        $\mu_{i_0}^{p_0} \in [\frac{15}{32}, \frac{17}{32}]$, and $\Delta_{i_0}^{p_0} \leq \frac{1}{16}$, using Lemma~\ref{lem:kl-ber},
        $\KL\del{\Ber(\mu_{i_0}^{p_0}), \Ber(\mu_{i_0}^{p_0} + 2\Delta_{i_0}^{p_0})} \leq 3 \cdot (2\Delta_{i_0}^{p_0})^2 = 12 (\Delta_{i_0}^{p_0})^2$. 
        
        \item Under MPMAB instance $\mu$, $\Rcal(T) \geq \Rcal^{p_0}(T) \geq \Delta_{i_0}^{p_0} n_{i_0}^{p_0}(T) \geq \frac{\Delta_{i_0}^{p_0} T}{2} \mathds{1}\{n_{i_0}^{p_0}(T) \geq \frac T 2\}$. Taking expectations, we get,
        \begin{equation}
        \EE_\mu \sbr{\Rcal(T)} \geq \frac{ \Delta_{i_0}^{p_0} T}{2} \PP_\mu( n_{i_0}^{p_0}(T) \geq \frac T 2). 
        \label{eqn:reg-lb-mu-3}
        \end{equation}
        Likewise, under MPMAB instance $\nu$, for player $p_0$, $\Rcal(T) \geq \Rcal^{p_0}(T) \geq \Delta_{i_0}^{p_0} (T - n_{i_0}^{p_0}(T)) \geq \frac{\Delta_{i_0}^{p_0} T}{2} \mathds{1}\{n_{i_0}^{p_0}(T) < \frac T 2\}$. Taking expectations, we get,
        \begin{equation}
        \EE_\nu \sbr{\Rcal(T)} \geq \frac{\Delta_{i_0}^{p_0} T}{2} \PP_\nu( n_{i_0}^{p_0}(T) < \frac T 2).
        \label{eqn:reg-lb-nu-3}
        \end{equation}
    \end{enumerate}
    Combining Equations~\eqref{eqn:sublinear-reg-3},~\eqref{eqn:div-decomp-ineq-3},~\eqref{eqn:reg-lb-mu-3},~\eqref{eqn:reg-lb-nu-3}, and using Bretagnolle-Huber inequality, we get 
    \[
    \EE_\mu\sbr{n_{i_0}^{p_0}(T)}
    \geq 
    \frac{1}{12 (\Delta_{i_0}^{p_0})^2} \cdot \ln\del{\frac{\Delta_{i_0}^{p_0}T^{\alpha}}{8C}}.
    \]
    This concludes the proof of the claim, and in turn concludes the proof of the theorem. 
\end{proof}

\subsection{Auxiliary lemmas}
The following lemma is well known for proving gap-independent lower bounds in single player $K$-armed bandits. We will be using the following convention: for probability distribution $\PP_i$, denote by $\EE_i$ its induced expectation operator.

\begin{lemma}
\label{lem:kl-pull}
Suppose $m, B$ are positive integers and $m \geq 2$; there are $m+1$ probability distributions $\PP_0, \PP_1, \ldots, \PP_m$, and 
$m$ random variables $N_1, \ldots, N_m$, such that:
(1) Under any of the $\PP_i$'s, $N_1,\ldots,N_m$ are non-negative and $\sum_{i=1}^m N_i \leq B$ with probability 1;
(2) for all $i$ in $[m]$, $d_{\TV}(\PP_0, \PP_i) \leq \frac14 \sqrt{\frac m {B} \cdot \EE_0[N_i]}$.
Then,
\[ \frac 1 m \sum_{i=1}^m \EE_i[B - N_i] \geq \frac{B}{4}. \]
\end{lemma}
\begin{proof}
For every $i$ in $[m]$, as $N_i$ is a random variable that takes values in $[0,B]$, we have,
\[ \abs{\EE_i[N_i] - \EE_0[N_i]} \leq B \cdot d_{\TV}(\PP_0, \PP_i).
\] 
By item (2) and algebra, this implies that
\[ \EE_i[N_i] \leq \EE_0[N_i] + \frac14 \sqrt{m B \EE_0[N_i]}. \]
Averaging over $i$ in $[m]$ and using Jensen's inequality, we have
\[ 
 \frac1m \sum_{i=1}^m \EE_i[N_i]
 \leq
 \frac1m \sum_{i=1}^m \EE_0[N_i] 
 + \frac1{4m} \sum_{i=1}^m \sqrt{m B \EE_0[N_i]}
 \leq 
 \frac1m \sum_{i=1}^m \EE_0[N_i] 
 + 
 \frac14 \sqrt{m B \cdot \del{\frac1m \sum_{i=1}^m \EE_0[N_i]}}
\]
Noting that item (2) implies $\frac1m \sum_{i=1}^m \EE_0[N_i] \leq \frac B m$; plugging this into the above inequality, we have
\[
 \frac1m \sum_{i=1}^m \EE_i[N_i]
 \leq
 \frac{B}{m} + \frac14 \sqrt{m B \cdot \frac{B}{m}}
 \leq 
 \frac B 2 + \frac B 4
 =
 \frac{3B}{4},
\]
where the second inequality uses the assumption that $m \geq 2$.
The lemma is concluded by negating and adding $B$ on both sides.
\end{proof}

\begin{lemma}
\label{lem:kl-ber}
Suppose $a, b$ are both in $[\frac14, \frac34]$. Then, 
$\KL(\Ber(a), \Ber(b)) \leq 3(b-a)^2$.
\end{lemma}
\begin{proof}
Define $h(x) = x \ln\frac1x + (1-x)\ln\frac1{1-x}$. It can be easily verified that 
$\KL(\Ber(a), \Ber(b)) = a \ln\frac a b + (1-a) \ln\frac{1-a}{1-b}$, which in turn equals $h(a) - h(b) - h'(b)(a-b)$.
By Taylor's theorem, there exists some $\xi \in [a,b] \subseteq [\frac14, \frac34]$ such that 
\[
h(a) - h(b) - h'(b)(a-b)
= 
\frac{h''(\xi)}{2} (b-a)^2
=
\frac{1}{2\xi(1-\xi)} (b-a)^2.
\]
The lemma is concluded by verifying that $\frac{1}{2\xi(1-\xi)} \leq 3$ for $\xi$ in $[\frac14, \frac34]$.
\end{proof}

\begin{lemma}[Bretagnolle-Huber]
\label{lem:bh}
For any two distributions $\PP$ and $\QQ$ and an event $A$, 
\[
\PP(A) + \QQ(A^C) \geq \frac12 \exp(-\KL(\PP, \QQ)).
\]
\end{lemma}

\section{Upper bounds with unknown $\epsilon$}
\label{appendix:ub_unk_eps}

In this section, we provide a description of \arobustagg, an algorithm that has regret adaptive to $\Ical_{\epsilon}$ in all MPMAB environments with unknown $\epsilon$.

To ensure sublinear regret in all MPMAB environments, \robustagg uses the aggregation-based framework named \corral~\citep[][see also Lemma~\ref{lem:master_rho_b} below]{alns17}, which we now briefly review.
The \corral meta-algorithm allows one to combine multiple online bandit learning algorithms (called base learners) into one master algorithm that has performance competitive with all base learners'. For different environments, different base learners may stand out as the best, and therefore the master algorithm exhibits some degree of adaptivity. We refer readers to \citep{alns17} for the full description of \corral. 

In the context of MPMAB problems, recall that we have developed $\robustagg(\epsilon)$ that has good regret guarantees for all $\epsilon$-MPMAB instances. The central idea of \arobustagg is to apply the \corral algorithm over a series of baser learners, i.e., $\cbr{\robustagg(\epsilon_b)}_{b=1}^B$, where $E = \cbr{\epsilon_b}_{b=1}^B$ is a covering of the $[0,1]$ interval. With an appropriate setting of $E$, for any $\epsilon$-MPMAB instance, there exists some $b_0$ in $[B]$ such that $\epsilon_{b_0}$ is not much larger than $\epsilon$, and running $\robustagg(\epsilon_{b_0})$ would achieve regret guarantee competitive to $\robustagg(\epsilon)$. As $\corral$ achieves online performance competitive with all $\robustagg(\epsilon_b)$'s~\citep{alns17}, it must be competitive with $\robustagg(\epsilon_{b_0})$, and therefore can inherit the adaptive regret guarantee of $\robustagg(\epsilon_{b_0})$.

We now provide important technical details of $\arobustagg$:
\begin{enumerate}
    \item $B = \lceil \log(MT)\rceil+1$ is the number of base learners, and 
    $E = \cbr{ \epsilon_b = 2^{-b+1}: b \in  [B]
    }$ is the grid of $\epsilon$ to be aggregated.
    $\corral$ uses master learning rate $\eta = \frac{1}{M\sqrt{T}}$.
    \item For each base learner that runs  $\robustagg(\epsilon)$ for some $\epsilon$, we require them to take a new parameter $\rho \geq 1$ as input, to accommodate for the fact that it may not be selected by the \corral master all the time.
    Specifically, it performs bandit learning interaction with an environment whose returned rewards are unbiased but \textit{importance weighted}: at time step $t$, when player $p$ pulls arm $i$, instead of directly receiving reward drawn from $r \sim \Dcal_{i}^p$, it receives $\hat{r} = \frac{W_t}{w_t} r$, where $w_t \in [1, \frac{1}{\rho}]$ is a random number, and conditioned on $w_t$, $W_t \sim \Ber(w_t)$ is an independently-drawn Bernoulli random variable.
    Observe that
    $\hat{r}$ has conditional mean $\mu_i^p$, lies in the interval $[0, \rho]$, and has conditional variance at most $\rho$. 
    
    We call an environment that has the above analytical form a {\em $\rho$-importance weighted environment}; in the special case of $\rho = 1$, $w_t = 1$ and $W_t = 1$ with probability $1$ for all $t$, and therefore a $1$-importance weighted environment is the same as the original bandit learning environment.
    
    Under an $\rho$-importance weighted environment, the rewards are no longer bounded in $[0,1]$, therefore, the constructions of the UCB's of the mean rewards in the original $\robustagg(\epsilon)$ becomes invalid. Instead, we will rely on the following lemma (analogue of Lemma~\ref{lem:Q}) for constructing valid UCB's:

    \begin{lemma}
    \label{lemma:concentration_rho}
    With probability at least $1 - 4MT^{-5}$, we have
    \[
    \abs{
    \zeta_i^p (t-1) - \mu_i^p 
    }
    \leq
    8 \sqrt{ \frac{3 \rho \ln T}{ \wbar{n_i^p}(t-1) } },
    \]
    \[
    \abs{ \eta_i^p(t-1) -\sum_{q \neq p}\frac{n^q_i(t-1)}{\wbar{m^p_i}(t-1) }\mu_i^q } \leq 4 \sqrt{\frac{14 \rho \ln T}{\wbar{m^p_i}(t-1)}}
    \]
    holding for all $p$ in $[M]$,
    where
    $
    \zeta_i^p(t) = \frac{ \sum_{s=1}^{t-1} \mathds{1}\{ i_t^p = i \}\hat{r}_t^p}{\wbar{n^p_i}(t-1)},  
    $
    and
    $
    \eta_i^p(t) = \frac{ \sum_{s=1}^{t-1} \sum_{q = 1}^M \mathds{1} \{ q \neq p,\ i_t^q = i \} \hat{r}_t^q }{\wbar{m^p_i}(t-1) }.
    $
    \end{lemma}
    
    According to the above concentration bounds, changing the definition of confidence interval width to $F(\wbar{n^p_i}, \wbar{m^p_i}, \lambda, \epsilon ) = 8\sqrt{13\rho\ln T\left[\frac{\lambda^2}{\wbar{n^p_i}} + \frac{(1-\lambda)^2}{\wbar{m^p_i}}\right]} + (1-\lambda)\epsilon$ would maintain the validity of the UCB's in $\rho$-importance weighted environments; henceforth, we incorporate this modification in $\robustagg(\epsilon)$. 
    
    We have the following important analogue of Theorem~\ref{thm:gap_ind_upper}, which establishes a gap-independent regret upper bound when $\robustagg(\epsilon)$ is run in a $\rho$-importance weighted $\epsilon$-MPMAB environment. 
    This shows $\robustagg(\epsilon)$ enjoys stability: the regret of the algorithm degrades gracefully with increasing $\rho$.\footnote{See an elegant definition of $(R(T), \alpha)$-(weak) stability for bandit algorithms in~\citep{alns17}.
    Our guarantee on $\robustagg(\epsilon)$ in Lemma~\ref{lem:gap_ind_upper_rho} is slightly stronger than the
    $( \tilde{\order}\rbr{
    \sqrt{\abr{\Ical_\epsilon} M T}
    + 
    M \abr{\Ical_\epsilon}
    +
    \min\del{
    M \sqrt{\abr{\Ical^C_\epsilon} T}
    ,
    \epsilon M T
    }
    }, \frac12)$-weak stability,
    in that the regret bound has terms that are unaffected by $\rho$.
    \label{footnote:not-stab}
    }
    
    \begin{lemma}
    \label{lem:gap_ind_upper_rho}
    Let $\robustagg(\epsilon)$ run on a $\rho$-importance weighted
    $\epsilon$-MPMAB problem instance for $T$ rounds. Then its expected collective regret satisfies
    \[ 
    \EE[\mathcal{R}(T)] \le
    \tilde{\order}\rbr{
    \sqrt{\rho \abr{\Ical_\epsilon} M T}
    + 
    M \abr{\Ical_\epsilon}
    +
    \min\del{
    M \sqrt{\rho \abr{\Ical^C_\epsilon} T}
    ,
    \epsilon M T
    }
    }.
    \]
    \end{lemma}
    The proof of Lemmas~\ref{lemma:concentration_rho} and~\ref{lem:gap_ind_upper_rho} can be found at the end of this section.
    
        \item \corral maintains a probability distribution on base learners $q_t = (q_{t,b}: b \in [B])$ over time. 
        At time step $t$, each base learner $b$ proposes their own arm pull decisions $(i_t^p(b): p \in [M])$; the \corral master chooses a base learner with probability according to $q_t$, that is,
        $i_t^p = i_t^p(b_t)$ for all $p$, where $b_t \sim q_t$.
        After the arm pulls, learner $b$ receives feedback $\hat{r}_t^p(b) = \frac{\ind\cbr{b_t = b}}{q_{t,b}} r_t^p$, which is equivalent to interacting with an importance weighted environment discussed before---$q_{t,b}$ and $\ind\cbr{b_t=b}$ correspond to $w_t$ and $W_t$, respectively; when $b_t = b$, $r_t^p$ is drawn from $\Dcal_{i_t^p}^p$ for all $p$ in $[M]$.
        
        \corral also uses the above feedback to update $q_{t+1}$, its weighting of the base learners: define $\ell_{t,b} = \frac{\ind\cbr{b_t = b}}{q_{t,b}} \mathds{1}\{b_t = b\} (\sum_{p=1}^M (1- r_t^p) )$ to be the importance weighted loss of base learner $b$ at time step $t$; $q_t$ is updated to $q_{t+1}$ using $(\ell_{t,b}: b \in [B])$, with online mirror descent with the log-barrier regularizer and learning rate $\eta > 0$.
        A small complication of directly applying the existing results of \corral is that \corral originally assumes that the losses suffered by the base learner from each round have range $[0,1]$. In the multi-player setting, the losses suffered by the base learner is the sum of the losses of all players, which has range $[0,M]$. Nevertheless, we can obtain a similar guarantee. Denote by $\rho_b$ be the final value of $\rho$ of base learner $b$ (see also the next item).
        A slight modification of~\citet[Lemma 13]{alns17} shows that for all base learner $b$, 
    \[
    \sum_{t=1}^T \sum_{b'=1}^B q_{t,b'} \ell_{t,b'}
    - 
    \sum_{t=1}^T \ell_{t,b}
    \leq
    O\del{
    \frac{B}{\eta}
    +
    \eta M^2 T
    }
    -
    \frac{\rho_b}{40 \eta \ln T}.
    \]
    Taking expectation on both sides, and observing that 
    $\EE\sbr{\sum_{t=1}^T \sum_{b'=1}^B q_{t,b'} \ell_{t,b'}} = \EE\sbr{ \sum_{t=1}^T \sum_{p=1}^M (1 - \mu_{i_t^p}^p) }$,
    and
    $\EE\sbr{\sum_{t=1}^T \ell_{t,b}}
    =
    \EE \sbr{ \sum_{t=1}^T \sum_{p=1}^M (1 -\mu_{i_t^p(b)}^p) }
    $, along with some algebra, we get the following lemma.
    \begin{lemma}
    \label{lem:master_rho_b}
    Suppose \arobustagg is run for $T$ rounds. Then, for every $b$ in $[B]$, we have that the regret of the master algorithm with respect to base learner $b$ is bounded by
    \[
    \EE\sbr{    
    \sum_{t=1}^T \sum_{p=1}^M \mu_{i_t^p(b)}^p 
    -
    \sum_{t=1}^T \sum_{p=1}^M \mu_{i_t^p}^p
    }
    \leq
    O\del{
    \frac{B}{\eta}
    +
    \eta M^2 T
    }
    -
    \frac{\EE\sbr{\rho_b}}{40 \eta \ln T}.
    \]
    \end{lemma}
    
    \item Following~\cite{alns17}, a doubling trick is used for maintaining the value of $\rho$'s for all base learners over time. Specifically, each base learner $b$ maintains a separate guess of $\rho$, an upper bound of $\max_{s=1}^t \frac{1}{q_{s,b}}$; if the upper bound is violated, its $\rho$ gets doubled and the base learner restarts. As \corral initializes $\rho$ as $2B$ for each base learner, and maintains the invariant that $\rho \leq BT$, the number of doublings/restarts for each base learner is at most $\lceil \log T \rceil$.
    For a fixed $b$, summing over the regret guarantees between different restarts of base learner $b$, we have the following regret guarantee.
    \begin{lemma}
    \label{lem:gap_ind_upper_rho_b}
    Suppose $\epsilon_b \geq \epsilon$, and $\robustagg(\epsilon_b)$ is run as a base learner of $\arobustagg$, on a 
    $\epsilon$-MPMAB problem instance for $T$ rounds. Denote by $\rho_b$ the final value of $\rho$. 
    Then its expected collective regret satisfies
    \[ 
    \EE\sbr{ 
    T \sum_{p=1}^M \mu_*^p
    -
    \sum_{t=1}^T \sum_{p=1}^M \mu_{i_t^p(b)}^p
    }
    \le
    \tilde{\order}\rbr{
    \sqrt{\EE\sbr{\rho_b} \abr{\Ical_{\epsilon_b}} M T}
    + 
    \min\del{ M \sqrt{\EE\sbr{\rho_b} \abr{\Ical^C_{\epsilon_b}} T},
    \epsilon M T}
    + 
    M \abr{\Ical_{\epsilon_b}}
    }.
    \]
    \end{lemma}
    The proof of this lemma can be found at the end of this section;
    we also refer the reader to~\citep[Appendix D]{alns17} for details.
    
\end{enumerate}

Combining all the lemmas above, we are now ready to prove Theorem~\ref{thm:gap_ind_upper_unk}, restated below for convenience.

\newtheorem*{T12}{Theorem~\ref{thm:gap_ind_upper_unk}}
\begin{T12}
Let \arobustagg run on an $\epsilon$-MPMAB problem instance with any $\epsilon \in [0,1]$.
Its expected collective regret in a horizon of $T$ rounds satisfies
$$
\mathbb{E}[\mathcal{R}(T)] \le \tilde{O}\left( \left( \abs{\Ical_{2\epsilon}} + M \abr{\Ical_{2\epsilon}^C} \right) \sqrt{T} + M |\Ical_{\epsilon}| \right).
$$
\end{T12}

\begin{proof}[Proof of Theorem~\ref{thm:gap_ind_upper_unk}]

Suppose \arobustagg interacts with an $\epsilon$-MPMAB problem instance. 
Let $b_0 = \max\cbr{ b \in [B]: \epsilon_b \geq \epsilon}$. From the definition of $E = \cbr{1, 2^{-1}, \ldots, 2^{-B+1}}$ and $\epsilon \in [0,1]$, $b_0$ is well-defined.

We present the following technical claim that elucidates the guarantee provided by learner $b_0$ based on Lemma~\ref{lem:gap_ind_upper_rho_b}; we defer its proof after the proof of the theorem. 

\begin{claim}
\label{claim:gap_ind_upper_rho_b_0}
Let $b_0$ be defined above.
$\robustagg(\epsilon_{b_0})$ is run as a base learner of $\arobustagg$, on a 
    $\epsilon$-MPMAB problem instance for $T$ rounds. Denote by $\rho_{b_0}$ the final value of $\rho$. 
    Then its expected collective regret satisfies
    \[ 
    \EE\sbr{ 
    T \sum_{p=1}^M \mu_*^p
    -
    \sum_{t=1}^T \sum_{p=1}^M \mu_{i_t^p(b_0)}^p
    }
    \le
    \tilde{\order}\rbr{
    \sqrt{
    \EE\sbr{\rho_{b_0}} M T 
    ( \abr{\Ical_{2\epsilon}} 
      +
      M  \abr{\Ical_{2\epsilon}^C} 
    )
    } 
    + 
    M \abr{\Ical_{\epsilon}}
    }.
    \]
\end{claim}

Combining Claim~\ref{claim:gap_ind_upper_rho_b_0} and Lemma~\ref{lem:master_rho_b} with $b = b_0$, we have the following regret guarantee for \arobustagg:
\begin{align*}
\EE\sbr{\Rcal(T)} 
&=
    \EE\sbr{ 
    T \sum_{p=1}^M \mu_*^p
    -
    \sum_{t=1}^T \sum_{p=1}^M \mu_{i_t^p}^p
    } \\
&=
    \EE\sbr{ 
    T \sum_{p=1}^M \mu_*^p
    -
    \sum_{t=1}^T \sum_{p=1}^M \mu_{i_t^p(b_0)}^p
    }
    +
    \EE\sbr{    
    \sum_{t=1}^T \sum_{p=1}^M \mu_{i_t^p(b_0)}^p 
    -
    \sum_{t=1}^T \sum_{p=1}^M \mu_{i_t^p}^p
    } \\
&\leq 
    \tilde{\order}\rbr{
    \sqrt{
    \EE\sbr{\rho_{b_0}} M T 
    ( \abr{\Ical_{2\epsilon}} 
      +
      M  \abr{\Ical_{2\epsilon}^C} 
    )
    }
    +
    M \abr{\Ical_{\epsilon}}
    +
    \frac{B}{\eta}
    +
    \eta M^2 T
    }
    -
    \frac{\EE\sbr{\rho_{b_0}}}{40 \eta \ln T} \\
& \leq 
    \tilde{\order}\rbr{
    \eta  M T 
    ( \abr{\Ical_{2\epsilon}} 
      +
      M  \abr{\Ical_{2\epsilon}^C} 
    ) 
    +
    M \abr{\Ical_{\epsilon}}
    +
    \frac{B}{\eta}
    +
    \eta M^2 T
    },
\end{align*}
where the first inequality is from Claim~\ref{claim:gap_ind_upper_rho_b_0} and Lemma~\ref{lem:master_rho_b}; the second inequality is from the AM-GM inequality that 
    $\sqrt{
    \EE\sbr{\rho_{b_0}} M T 
    ( \abr{\Ical_{2\epsilon}} 
      +
      M  \abr{\Ical_{2\epsilon}^C} 
    )
    } 
    \leq 
    O(
    \eta  M T 
    ( \abr{\Ical_{2\epsilon}} 
      +
      M  \abr{\Ical_{2\epsilon}^C} 
    ) 
    +
    \frac{\EE\sbr{\rho_{b_0}}}{\eta \ln T}
    )
    $ and algebra (canceling out the second term in the last expression with $-
    \frac{\EE\sbr{\rho_{b_0}}}{40 \eta \ln T}$). 
    As \arobustagg chooses \corral's master learning rate $\eta = \frac{1}{M\sqrt{T}}$, and $B = \tilde{O}(1)$, we have that 
    \[
      \EE\sbr{\Rcal(T)} 
      \leq 
      \tilde{O} 
      \del{ (M + \abr{\Ical_{2\epsilon}}
      +
      M  \abr{\Ical_{2\epsilon}^C}) \sqrt{T}
      +
      M \abr{\Ical_{\epsilon}}
      }
      \leq
      \tilde{O} 
      \del{ (\abr{\Ical_{2\epsilon}} 
      +
      M \abr{\Ical_{2\epsilon}^C}) \sqrt{T} 
      +
      M \abr{\Ical_{\epsilon}}
      },
    \]
    where the second inequality uses the fact that $\abs{\Ical_{2\epsilon}^C} \geq 1$.
\end{proof}

\begin{proof}[Proof of Claim~\ref{claim:gap_ind_upper_rho_b_0}]
As $\epsilon_b \geq \epsilon$ always holds,
$M \abr{\Ical_{\epsilon_{b_0}}} \leq M \abr{\Ical_{\epsilon}}$. 
It remains to check by algebra that 
\begin{equation}
    \sqrt{\EE\sbr{\rho_{b_0}} \abr{\Ical_{\epsilon_{b_0}}} M T}
    + 
    \min\del{
    M \sqrt{\EE\sbr{\rho_{b_0}} \abr{\Ical^C_{\epsilon_{b_0}}} T}
    ,
    M T \epsilon_{b_0}
    }
    =
    \tilde{\order}\rbr{
    \sqrt{
    \EE\sbr{\rho_{b_0}} M T 
    ( \abr{\Ical_{2\epsilon}} 
      +
      M  \abr{\Ical_{2\epsilon}^C} 
    )
    } 
    }.
\label{eqn:eps-b-2eps}
\end{equation}
We consider two cases:
\begin{enumerate}
    \item $\epsilon_{b_0} \leq 2\epsilon$. In this case,  we have $\Ical_{\epsilon_{b_0}} \subset \Ical_{2\epsilon}$.
    We have the following derivation:
    \begin{align*}
        \sqrt{\EE\sbr{\rho_{b_0}} \abr{\Ical_{\epsilon_{b_0}}} M T}
        + 
        M \sqrt{\EE\sbr{\rho_{b_0}} \abr{\Ical^C_{\epsilon_{b_0}}} T} 
        & \leq 
        2\sqrt{
        \EE\sbr{\rho_{b_0}} M T 
        \cdot ( \abr{\Ical_{\epsilon_b}} 
        +
        M \abr{\Ical_{\epsilon_b}^C} )
        } \\
        & \leq 
        2 \sqrt{
        \EE\sbr{\rho_{b_0}} M T 
        ( \abr{\Ical_{2\epsilon}} 
        +
        M \abr{\Ical_{2\epsilon}^C} 
        )
        }
    \end{align*}
    where the first inequality is from the basic fact that $\sqrt{A} + \sqrt{B} \leq 2\sqrt{A + B}$ for positive $A$, $B$; the second inequality is from the fact that  
    $
      \abr{\Ical_{\epsilon_b}} 
        +
        M \abr{\Ical_{\epsilon_b}^C}
     \leq 
     \abr{\Ical_{2\epsilon}}  +
        M \abr{\Ical_{2\epsilon}^C}
    $,
    as $\abs{\Ical_{\epsilon_b}} \geq \abs{\Ical_{2\epsilon}}$, $M \geq 1$,  and $\abs{\Ical_{\alpha}} + \abs{\Ical_{\alpha}^C} = K$ for any $\alpha$. This verifies Eq.~\eqref{eqn:eps-b-2eps}.
        
    \item $\epsilon_{b_0} > 2\epsilon$. In this case, $b_0 = B = 1 + \lceil \log (MT) \rceil$ and $\epsilon_{b_0} \leq \frac{1}{MT}$. Although we no longer have $\Ical_{\epsilon_{b_0}} \subset \Ical_{2\epsilon}$, we can still upper bound the left hand side as follows. 
    
    First, the second term, 
    $\min\del{
    M \sqrt{\EE\sbr{\rho_{b_0}} \abr{\Ical^C_{\epsilon_{b_0}}} T}
    ,
    M T \epsilon_{b_0}
    } \leq MT \cdot \frac{1}{MT} = 1$.
    
    Moreover, the first term, $\sqrt{\EE\sbr{\rho_{b_0}} \abr{\Ical_{\epsilon_{b_0}}} M T} \leq \sqrt{ \EE\sbr{\rho_{b_0}} K M T }$.
    As $\abr{\Ical_{2\epsilon}} 
      +
      M  \abr{\Ical_{2\epsilon}^C}
      \geq K
    $, we have
    $\sqrt{\EE\sbr{\rho_{b_0}} \abr{\Ical_{\epsilon_{b_0}}} M T} \leq \sqrt{ \EE\sbr{\rho_{b_0}} (\abr{\Ical_{2\epsilon}} 
      +
      M  \abr{\Ical_{2\epsilon}^C}) M T }$. Combining the above two,  Eq.~\eqref{eqn:eps-b-2eps} is proved.
     \qedhere      
\end{enumerate}
\end{proof}

\begin{proof}[Proof sketch of Lemma~\ref{lemma:concentration_rho}]

Since the proof of Lemma~\ref{lem:Q} can be almost directly carried over here, we only sketch the proof by pointing out the major differences. We also refer the reader to \citep[Appendix C.3]{amm20} for a similar reasoning.

We first consider the concentration of $\zeta_i^p(j, t)$. We define a filtration $\cbr{\Bcal_t}_{t=1}^T$, where 
$$\Bcal_t = \sigma( \{ w_s, i_s^{p'}, \hat{r}_s^{p'}: s \in [t], p' \in [M] \} \cup \cbr{ i_{t+1}^{p'}: p' \in [M]})$$
is the $\sigma$-algebra generated by the %
history (including that of $w_s$'s) up to round $t$ and the arm selection of all players at time step $t+1$

Let $X_t = \mathds{1}\{ i_t^p = i \}  \del{ \hat{r}_t^p - \mu_i^p}$. 
We have $\EE\sbr{X_t \mid \Bcal_{t-1}} = 0$.
In addition,
\begin{align*}
\VV\sbr{X_t \mid \Bcal_{t-1}} & = 
\EE\sbr{(X_t - \EE[X_t \mid \Bcal_{t-1}])^2 \mid \Bcal_{t-1}} \\
& =
\EE\sbr{X_t^2 \mid \Bcal_{t-1}} \\
& \leq \mathds{1} \{ i_t^p = i \} \EE \sbr{ w_{t} \left( \frac{r_t^p}{w_{t}} - \mu_i^p\right)^2 +  (1-w_t) 0 \mid \Bcal_{t-1}} \\
& \le \mathds{1} \{ i_t^p = i \} \EE \sbr{ w_{t} \left( \frac{r_t^p}{w_{t}} \right)^2 \mid \Bcal_{t-1}} \\
& \le \mathds{1} \{ i_t^p = i \} \rho.
\end{align*}
Also, $\abs{X_t} \leq \rho$ with probability 1.
Applying Freedman's inequality~\cite[Lemma 2]{bartlett2008high} with 
$\sigma = \sqrt{\sum_{s=1}^{t-1} \VV\sbr{X_s \mid \Bcal_{s-1}}}$ and $b = \rho$, and using $\sigma \leq \sqrt{\sum_{s=1}^{t-1} \mathds{1} \{i_s^p = i\} \rho}$,
we have that with probability at least $1-2 T^{-5}$,
\begin{equation}
\abs{\sum_{s=1}^{t-1} X_s}
\leq
4 \sqrt{  \sum_{s=1}^{t-1} \mathds{1} \{ i_s^p = i \} \rho \cdot  \ln(T^5\log_2 T) }
+
2 \rho \ln(T^5 \log_2 T).
\end{equation}

We can then show that
\[
\abs{  \frac{\sum_{s=1}^{t-1} \mathds{1} \{i_s^p = i \} \hat{r}_s^p}{\wbar{n^p_i}(t-1)} -\mu_i^p }
\leq 
8 \sqrt{ \frac{3 \rho \ln T}{\wbar{n^p_i}(t-1)}  }.
\]
following the same strategy in the proof for Lemma~\ref{lem:Q}.

Similarly, we show the concentration of $\eta_i^p(t)$.
We define a filtration $\{\mathcal{G}_{t,q} \}_{t \in [T], q \in [M]}$, where
$$\mathcal{G}_{t,q} = \sigma( \cbr{ w_s, i_s^{p'}, \hat{r}_s^{p'}: s \in [t], p' \in [M], i \in [K] } \cup \cbr{ i_{t+1}^{p'}: p' \in [M], p' \le q})$$ 
is the $\sigma$-algebra generated by 
the history (including that of $w_s$'s) up to round $t$ and the arm selection of players $1,2,\ldots,q$ at round $t+1$. %

Let random variable $Y_{t,q} = \mathds{1} \{ q \neq p,\ i_t^q = i \}  \del{\hat{r}^q_s - \mu_i^q}$. We have $\EE\sbr{Y_{t,q} \mid \Gcal_{t-1,q}} = 0$; in addition, $\VV\sbr{Y_t \mid \Gcal_{t-1}} = \EE\sbr{Y_{t,q}^2 \mid \Gcal_{t-1, q}} \leq \mathds{1} \{ q \neq p,\ i_t^q = i \} \rho$ and $\abs{Y_{t,q}} \leq \rho$.

Again, applying Freedman's inequality~\cite[Lemma 2]{bartlett2008high}, we have that with probability at least $1-2T^{-5}$,
\begin{equation}
\abs{\sum_{s=1}^{t-1} \sum_{q=1}^{M} Y_{s,q}}
\leq
4 \sqrt{  \sum_{s=1}^{t-1} \sum_{q=1}^M \mathds{1}\{ q \neq p,\ i_s^q = i \} \rho \cdot  \ln(T^5\log_2 (TM)) }
+
2 \rho \ln(T^5 \log_2 (TM)).
\end{equation}

Using the same strategy from the proof for Lemma~\ref{lem:Q}, we can show that
\[
\abs{ \eta_i^p(t-1) -\sum_{q \neq p}\frac{n^q_i(t-1)}{\wbar{m^p_i}(t-1) }\mu_i^q } \leq 4 \sqrt{\frac{14 \rho \ln T}{\wbar{m^p_i}(t-1)}}. 
\qedhere
\]

The lemma then follows by applying the union bound.
\end{proof}

\begin{proof}[Proof sketch of Lemma~\ref{lem:gap_ind_upper_rho}]
Similar to the proof of Theorem~\ref{thm:gap_ind_upper}, we define
$\mathcal{E} = \cap_{t=1}^{T} \cap_{i=1}^K \Qcal_i(t)$, where
\[
\Qcal_i(t)
=
\cbr{ 
\forall p, 
\left|\zeta^p_i(t) - \mu^p_i\right| \le 8\sqrt{\frac{3\rho\ln T}{\wbar{n^p_i}(t-1)}},\quad
\abs{ \eta_i^p(t) - \sum_{q \neq p} \frac{n_i^q(t-1)}{\wbar{m_i^p}(t-1)} \mu_i^q} \leq 4\sqrt{\frac{14 \rho \ln T}{\wbar{m_i^p}(t-1)}} 
};
\]
note that the new definition of $\Qcal_i(t)$ has a dependence on $\rho$.

Similar to the proof of Theorem~\ref{thm:gap_dep_ub}, we have, 
\begin{align*}
    \EE[\mathcal{R}(T)] \le \EE[\mathcal{R}(T) \vert \mathcal{E}] + O(1),
\end{align*}
and
\begin{align*}
    \EE[\mathcal{R}(T) \vert \Ecal] &= 
    \sum_{i \in [K]} \sum_{p \in [M]} \mathbb{E} [n^p_i(T) | \Ecal] \cdot \Delta^p_i \nonumber \\
    & \le  \sum_{i \in \Ical_\epsilon} \mathbb{E} [n_i(T) | \Ecal] \cdot \Delta_i^{\max} 
    +  
    \sum_{i \in \Ical^C_\epsilon} \sum_{p \in [M]: \Delta_i^p > 0} \mathbb{E} [n^p_i(T) | \Ecal] \cdot \Delta_i^p   
\end{align*}

We bound these two terms respectively, applying the technique from \citep[Theorem 7.2]{lattimore2020bandit}.

\begin{enumerate}
    \item We can show the following analogue of Lemma~\ref{lem:arm_pulls_I}:
    there exists some constant $C_1 > 0$ such that for each $i \in \Ical_\epsilon$,
    \begin{align}
        \mathbb{E} [n_i(T) \vert \mathcal{E}] \le  C_1 \left(\frac{\rho \ln T}{(\Delta_i^{\min})^2}  + M\right). \nonumber
    \end{align}
    Using the above fact, and from a similar calculation of Equation~\eqref{eq:gap_dep_term_1} in the proof of Theorem~\ref{thm:gap_ind_upper}, we get
    \[
    \sum_{i \in \Ical_\epsilon} \mathbb{E} [n_i(T) \vert \Ecal] \cdot \Delta_i^{\max}
    \leq
    4\sqrt{C_1 \rho |\Ical_\epsilon|MT\ln T} + 2C_1M |\Ical_\epsilon|.
    \]

    \item 
    We can show the following analogue of Lemma~\ref{lem:arm_pulls_IC}: there exists some constant $C_2 > 0$ such that for each $i \in \Ical^C_\epsilon$ and $p \in [M]$ with $\Delta_i^p > 0$, 
    \begin{align*}
        \mathbb{E} [n^p_i(T) \vert \mathcal{E}] \le C_2 \left( \frac{\ln T}{(\Delta^p_i)^2} \right). 
    \end{align*}

    Using the above fact, and from a similar calculation of Equation~\eqref{eq:gap_dep_term_2} in the proof of Theorem~\ref{thm:gap_ind_upper}, we get
    \[
    \sum_{i \in \Ical^C_\epsilon} \sum_{p \in [M]: \Delta_i^p > 0} \mathbb{E} [n^p_i(T) | \Ecal] \cdot \Delta_i^p
    \leq 
    2M \sqrt{C_2 \rho |\Ical^C_\epsilon|T\ln T}.
    \]

    On the other hand, we trivially have that for all $i$ in $\Ical_\epsilon^C$, $\Delta_i^p \leq 5\epsilon$; therefore,
    \[
    \sum_{i \in \Ical^C_\epsilon} \sum_{p \in [M]: \Delta_i^p > 0}  \mathbb{E} [n^p_i(T)] \cdot \Delta_i^p
    \leq 
    5 \epsilon MT.
    \]

\end{enumerate}

Therefore,
\begin{align*}
\EE[\mathcal{R}(T)] & \le
\left(4\sqrt{C_1 \rho |\Ical_\epsilon|MT\ln T} + 2C_1 M |\Ical_\epsilon|\right)
+ 
\min \del{2M  \sqrt{C_2 \rho |\Ical^C_\epsilon|T\ln T}, 5 \epsilon MT}
+
O(1) \\
& \le
\tilde{\order}\rbr{
\sqrt{\rho \abr{\Ical_\epsilon} M T}
+ 
M \abr{\Ical_\epsilon}
+
\min\del{
M \sqrt{\rho \abr{\Ical^C_\epsilon} T}
,
\epsilon M T
}
}.
\qedhere
\end{align*}
\end{proof}

\begin{proof}[Proof of Lemma~\ref{lem:gap_ind_upper_rho_b}]
The proof closely follows~\citep[][Theorem 15]{alns17}; we cannot directly repeat that proof here, because Lemma~\ref{lem:gap_ind_upper_rho} is not precisely a weak stability statement (see footnote~\ref{footnote:not-stab}).

For base learner $b$, suppose that its $\rho$
gets doubled $n_b$ times throughout the process, where $n_b$ is a random number in $[\lceil \log T\rceil]$.
For every $i \in [n_b]$, denote by random variable $t_i$ the $i$-th time step where the value of $\rho$ gets doubled. In addition, denote by $t_0 = 0$ and $t_{n_b+1} = T$.
In this notation, for all $t \in \cbr{t_i+1, t_i+2, \ldots, t_{i+1}}$, the value of $\rho$ is equal to $\rho^i = 2B \cdot 2^i$; in addition, $\rho_b = 2B \cdot 2^{n_b}$.

Therefore, we have:
\begin{align*}
\EE\sbr{\Rcal(T) \mid n_b = n}
& =
\sum_{i=0}^{n} \EE\sbr{\sum_{t=t_i + 1}^{t_{i+1}} \del{\sum_{p=1}^M \mu_*^p
    -
    \sum_{p=1}^M \mu_{i_t^p(b)}^p} \mid n_b = n} \\
& =
\sum_{i=0}^{n}
    \tilde{\order}\rbr{ 
    \sqrt{\rho^i \abr{\Ical_{\epsilon_b}} M T}
    + 
    M \abr{\Ical_{\epsilon_b}}
    +
    \min\del{
    M \sqrt{\rho^i \abr{\Ical^C_{\epsilon_b}} T}
    ,
    {\epsilon_b} M T
    }
    }  \\ 
& =
\tilde{\order}\rbr{ 
    \sqrt{\rho^n \abr{\Ical_{\epsilon_b}} M T}
    + 
    M \abr{\Ical_{\epsilon_b}}
    +
    \min\del{
    M \sqrt{\rho^n \abr{\Ical^C_{\epsilon_b}} T}
    ,
    {\epsilon_b} M T
    }
    },
\end{align*}
where the first equality of by the definition of $\Rcal(T)$, and $[T] = \cup_{i=1}^{n} \cbr{t_i+1, t_i+2, \ldots, t_{i+1}}$; the second equality is from Lemma~\ref{lem:gap_ind_upper_rho_b}'s guarantee in each time interval $\cbr{t_i+1, t_i+2, \ldots, t_{i+1}}$ and $\epsilon_b \geq \epsilon$;
and the third equality is by algebra.

As $n_b = n$ is equivalent to $\rho^n = \rho_b$, this implies that 
\[
\EE\sbr{\Rcal(T) \mid \rho_b}
=
\tilde{\order}\rbr{ 
    \sqrt{\rho_b \abr{\Ical_{\epsilon_b}} M T}
    + 
    M \abr{\Ical_{\epsilon_b}}
    +
    \min\del{
    M \sqrt{\rho_b \abr{\Ical^C_{\epsilon_b}} T}
    ,
    {\epsilon_b} M T
    }
    };
\]
observe that the expression inside $\tilde{O}$ in the last line is a concave function of $\rho_b$.

Now, by the law of total expectation,
\begin{align*}
    \EE \sbr{\Rcal(T)}
    & =  \EE\sbr{ \EE\sbr{\Rcal(T) \mid \rho_b} } \\
    & = \EE\sbr{ \tilde{\order}\rbr{ 
    \sqrt{\rho_b \abr{\Ical_{\epsilon_b}} M T}}
    + 
    M \abr{\Ical_{\epsilon_b}}
    +
    \min\del{
    M \sqrt{\rho_b \abr{\Ical^C_{\epsilon_b}} T}
    ,
    {\epsilon_b} M T
    }
    } \\
    & =
    \tilde{\order}\rbr{ 
    \EE\sbr{
    \sqrt{\rho_b \abr{\Ical_{\epsilon_b}} M T}
    + 
    M \abr{\Ical_{\epsilon_b}}
    +
    \min\del{
    M \sqrt{\rho_b \abr{\Ical^C_{\epsilon_b}} T}
    ,
    {\epsilon_b} M T
    }
    }} \\
    &=
    \tilde{\order}\rbr{ 
    \sqrt{\EE\sbr{\rho_b} \abr{\Ical_{\epsilon_b}} M T}
    + 
    M \abr{\Ical_{\epsilon_b}}
    +
    \min\del{
    M \sqrt{\EE\sbr{\rho_b} \abr{\Ical^C_{\epsilon_b}} T}
    ,
    {\epsilon_b} M T
    }
    },
\end{align*}
where the third equality is by algebra, and the last equality uses Jensen's inequality. 
\end{proof}

\begin{figure}[htbp]
    \centering
    \begin{subfigure}{0.32\textwidth}
        \centering
        \includegraphics[height=0.85\linewidth]{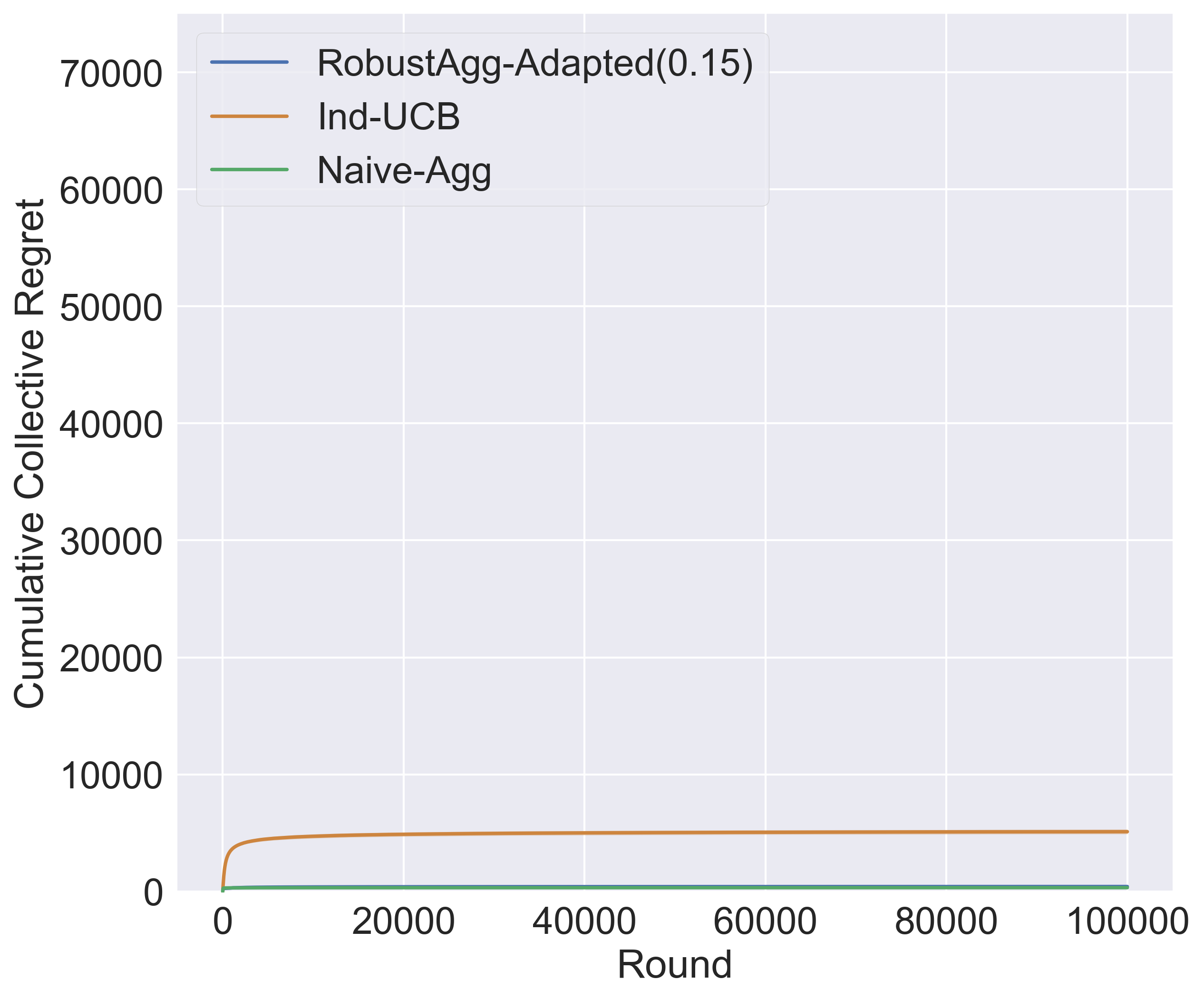}
        \caption{$|\Ical_\epsilon| = 9$}
        \label{figure:exp1_9}
    \end{subfigure}
    \begin{subfigure}{0.32\textwidth}
        \centering
        \includegraphics[height=0.85\linewidth]{Figures/Exp1/5eps_mpl_ical=8_100000x30.png}
        \caption{$|\Ical_\epsilon| = 8$}
        \label{figure:exp1_8}
    \end{subfigure}
    \begin{subfigure}{0.32\textwidth}
        \centering
        \includegraphics[height=0.85\linewidth]{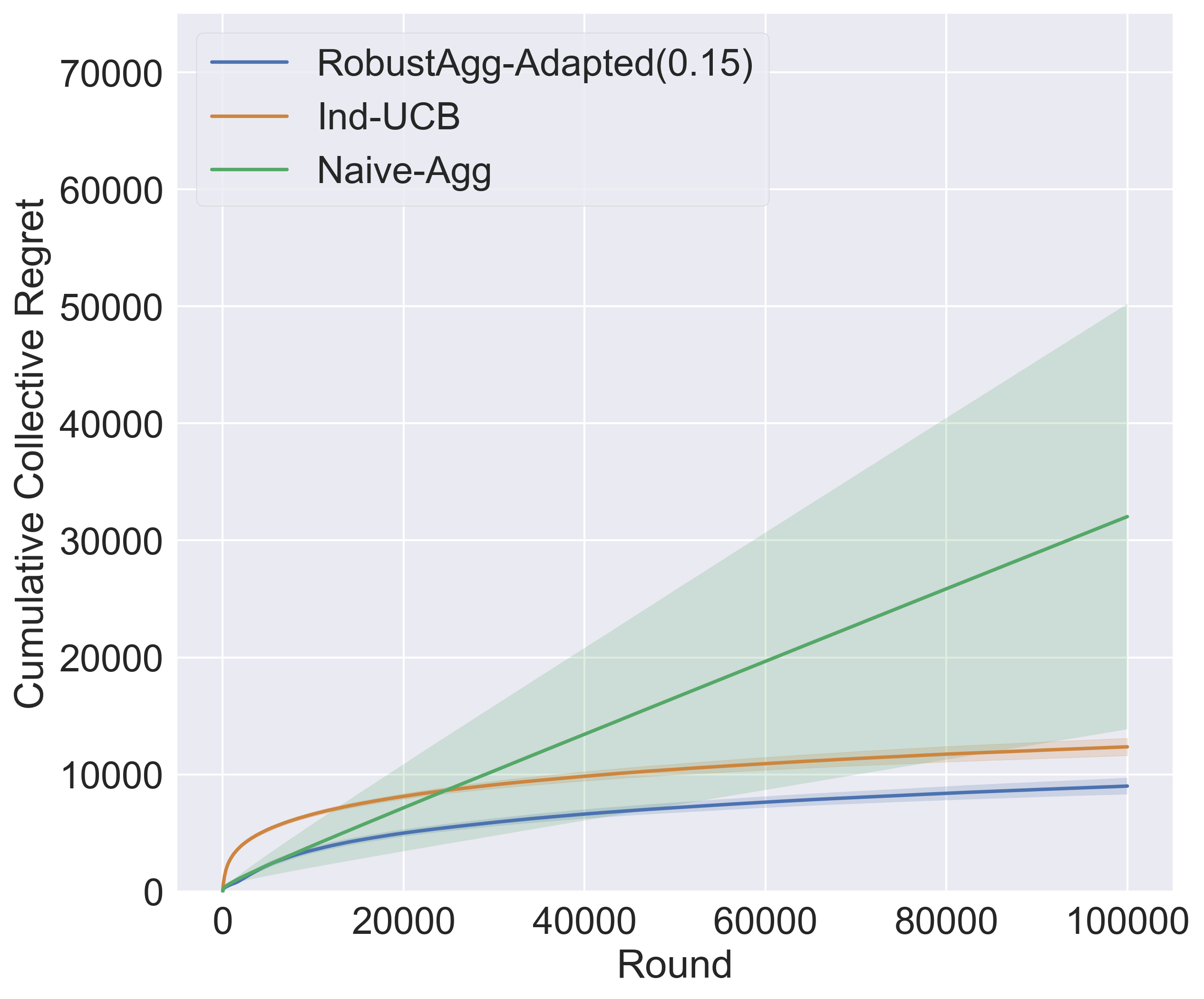}
        \caption{$|\Ical_\epsilon| = 7$}
        \label{figure:exp1_7}
    \end{subfigure}
        \begin{subfigure}{.32\textwidth}
        \centering
        \includegraphics[height=0.85\linewidth]{Figures/Exp1/5eps_mpl_ical=6_100000x30.png}
        \caption{$|\Ical_\epsilon| = 6$}
        \label{figure:exp1_6}
    \end{subfigure}
    \begin{subfigure}{0.32\textwidth}
        \centering
        \includegraphics[height=0.85\linewidth]{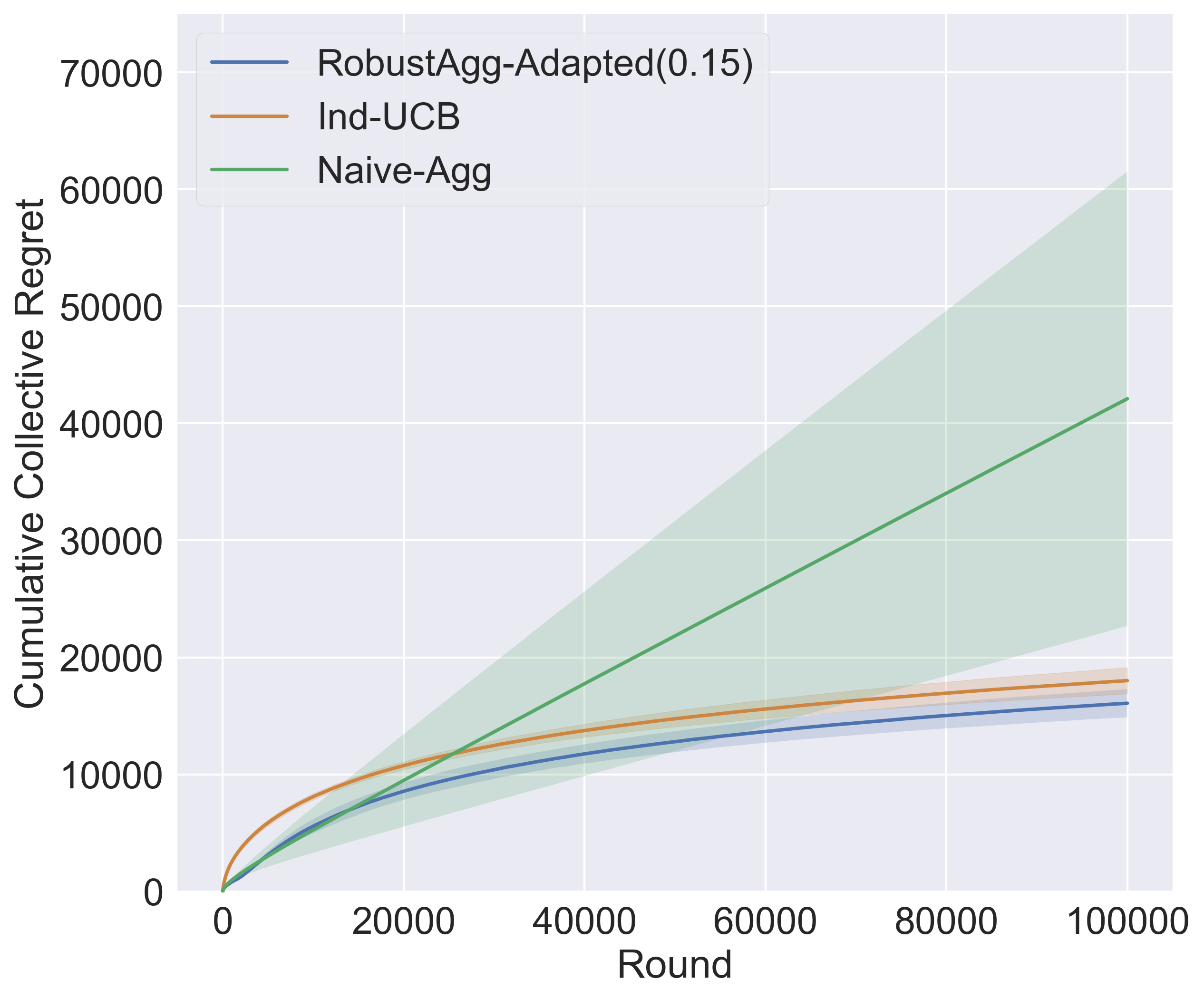}
        \caption{$|\Ical_\epsilon| = 5$}
        \label{figure:exp1_5}
    \end{subfigure}
    \begin{subfigure}{0.32\textwidth}
        \centering
        \includegraphics[height=0.85\linewidth]{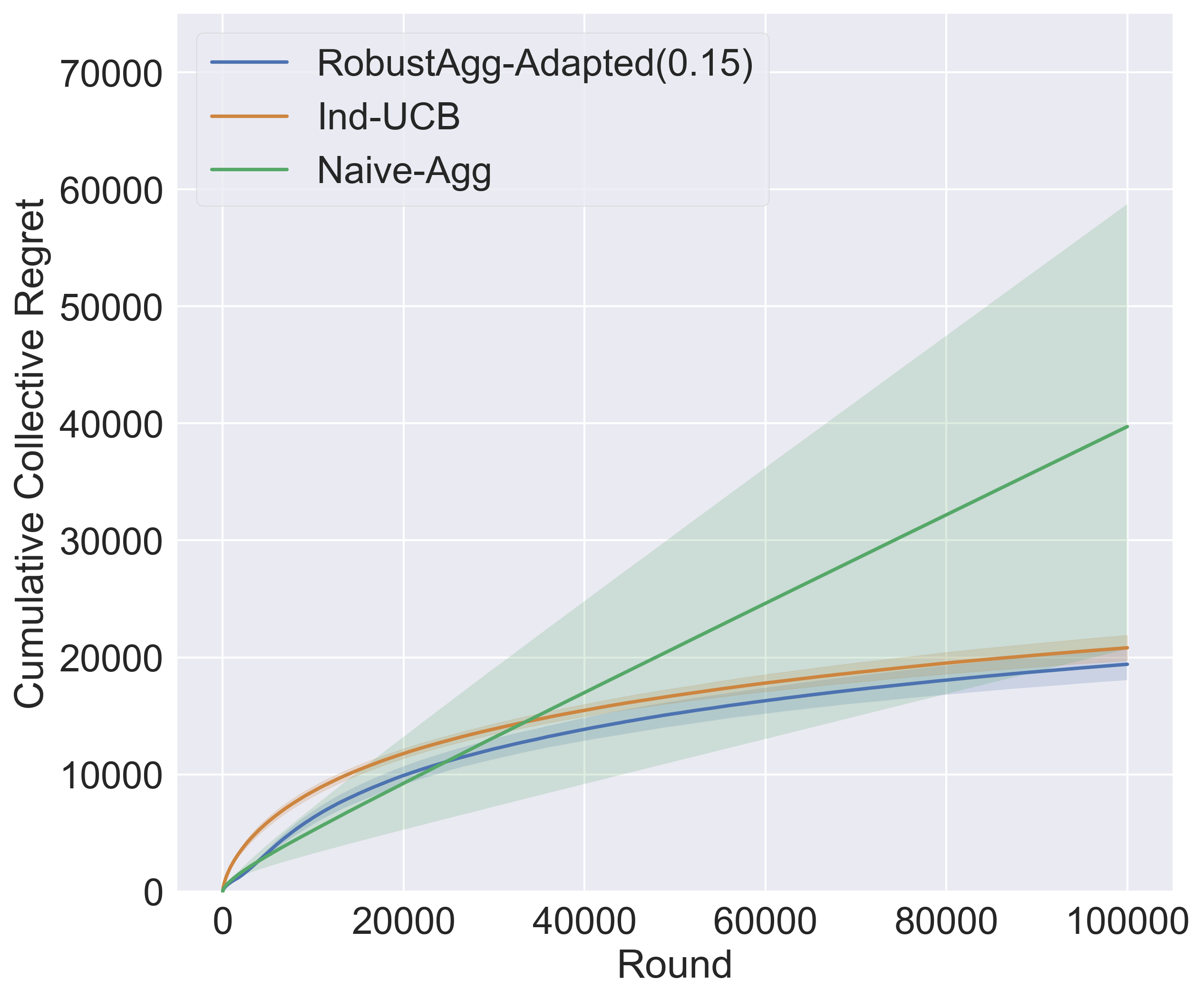}
        \caption{$|\Ical_\epsilon| = 4$}
        \label{figure:exp1_4}
    \end{subfigure}
        \begin{subfigure}{.32\textwidth}
        \centering
        \includegraphics[height=0.85\linewidth]{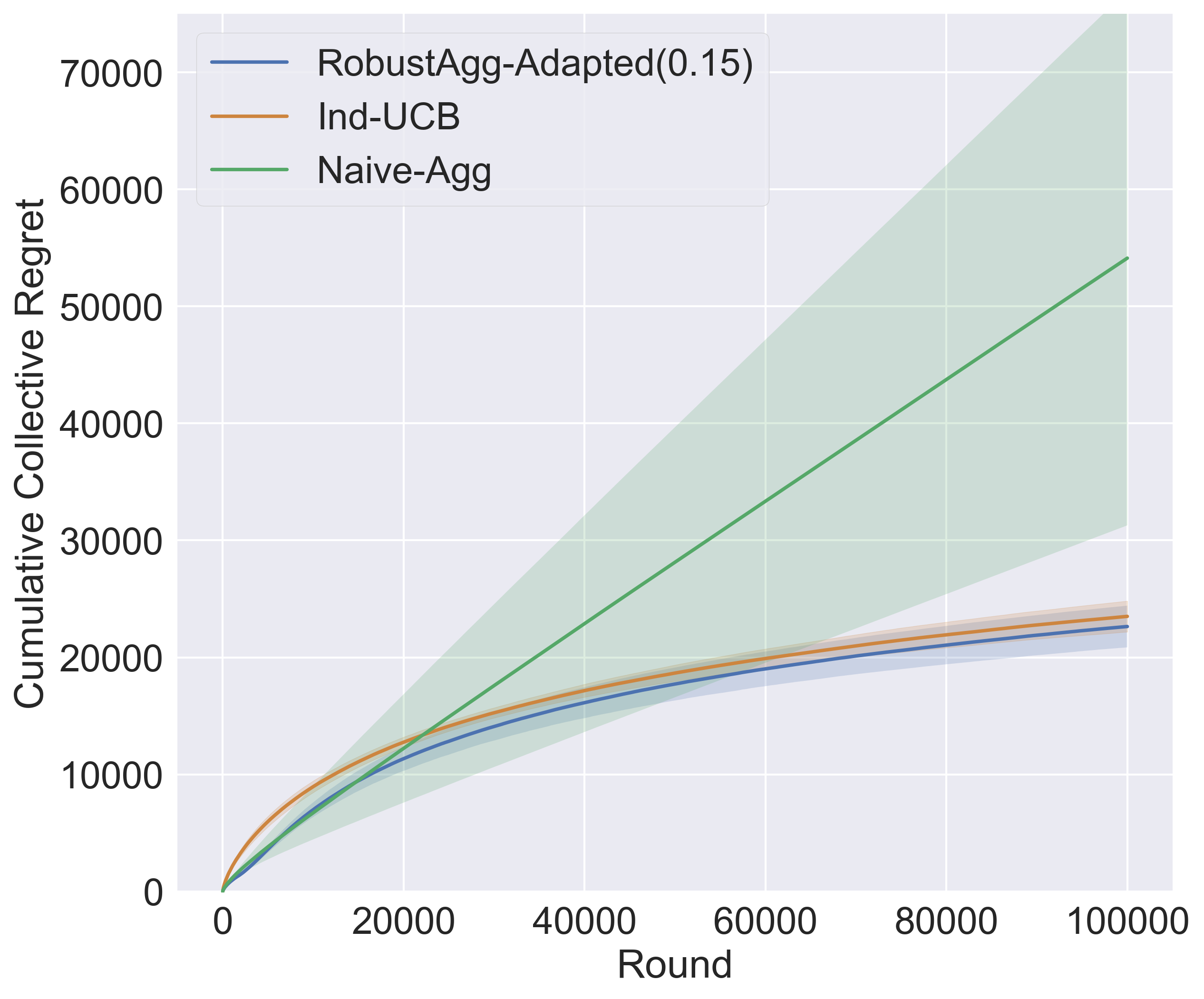}
        \caption{$|\Ical_\epsilon| = 3$}
        \label{figure:exp1_3}
    \end{subfigure}
    \begin{subfigure}{0.32\textwidth}
        \centering
        \includegraphics[height=0.85\linewidth]{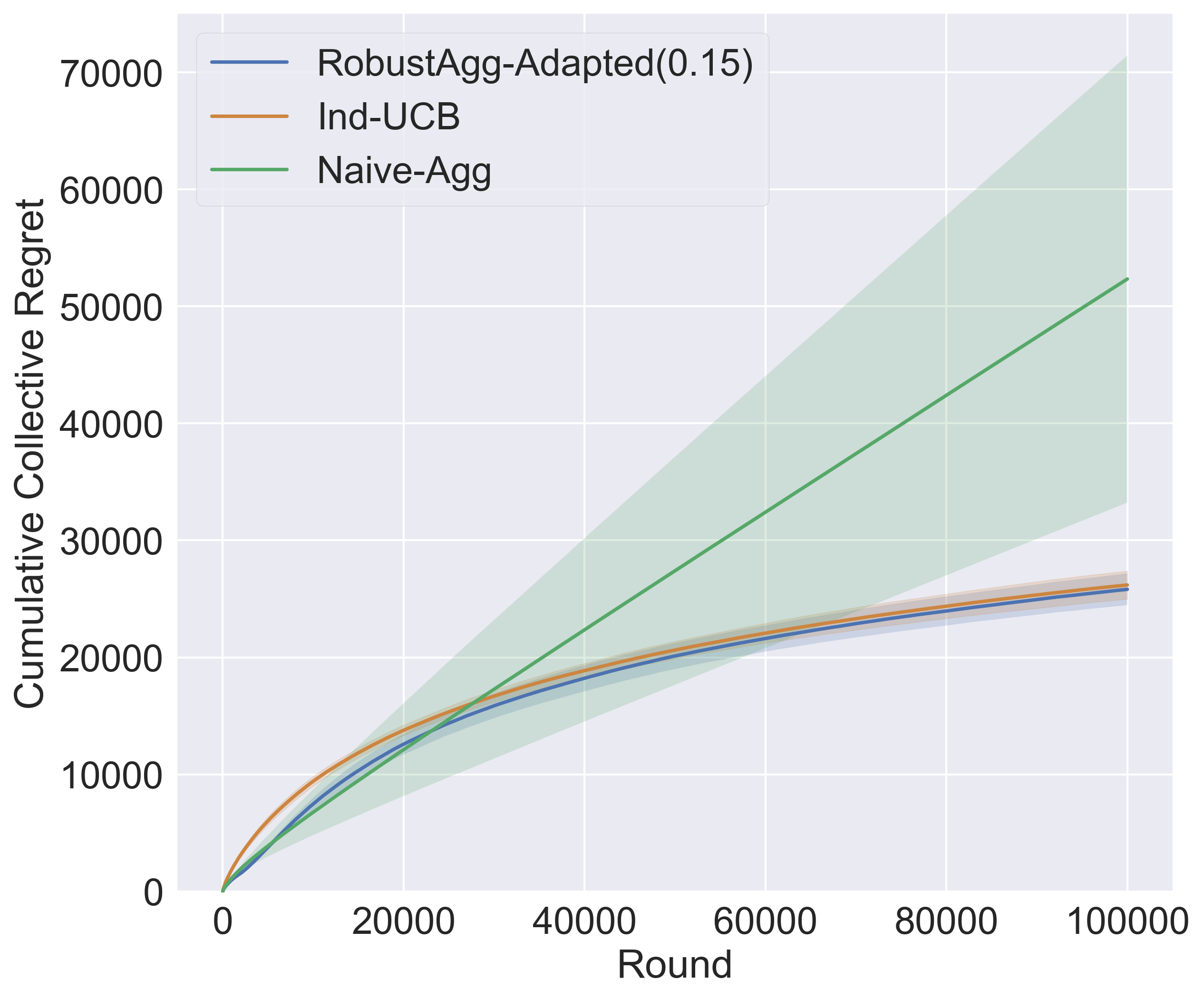}
        \caption{$|\Ical_\epsilon| = 2$}
        \label{figure:exp1_2}
    \end{subfigure}
    \begin{subfigure}{0.32\textwidth}
        \centering
        \includegraphics[height=0.85\linewidth]{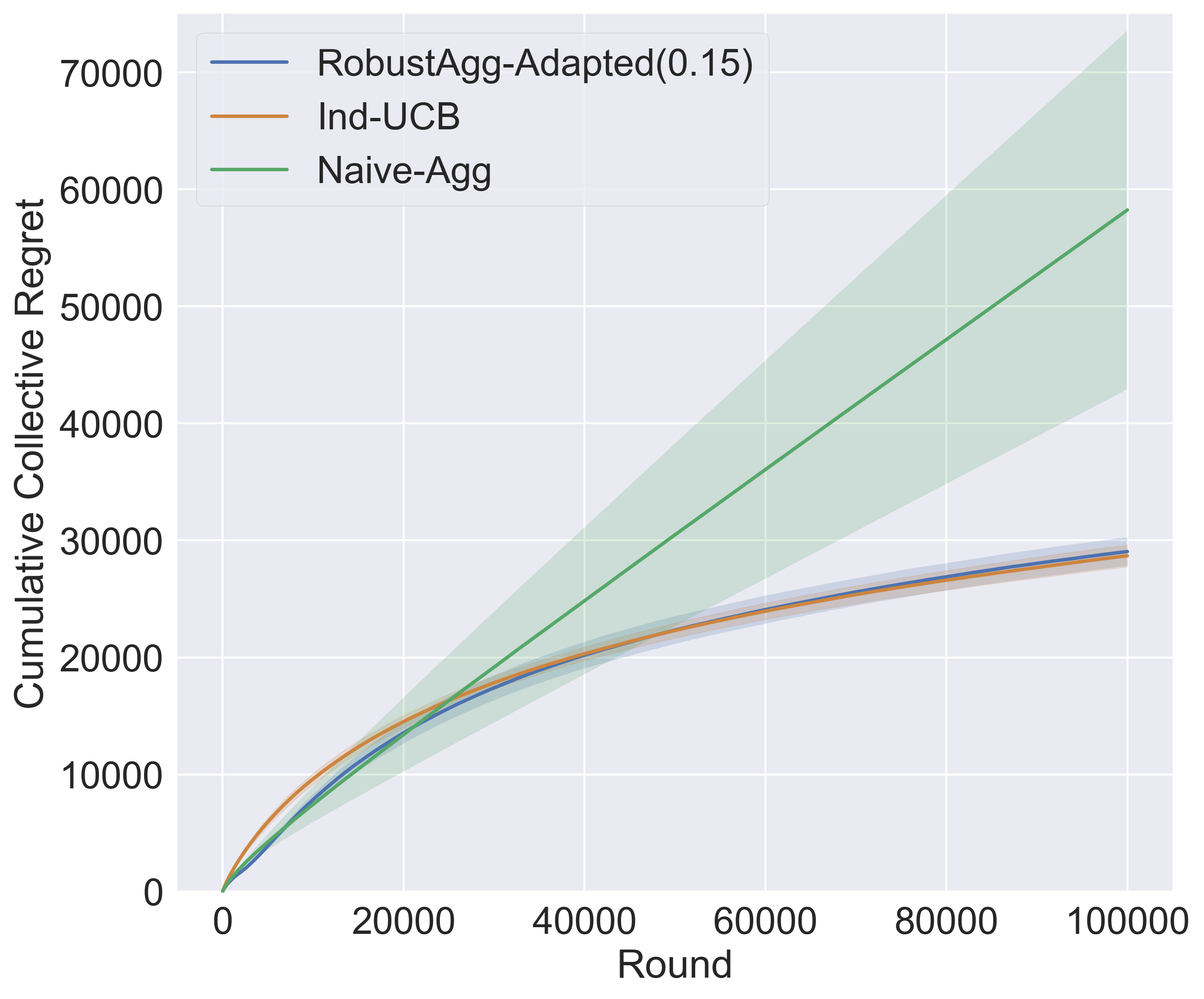}
        \caption{$|\Ical_\epsilon| = 1$}
        \label{figure:exp1_1}
    \end{subfigure}
    \begin{subfigure}{0.32\textwidth}
        \centering
        \includegraphics[height=0.85\linewidth]{Figures/Exp1/5eps_mpl_ical=0_100000x30.png}
        \caption{$|\Ical_\epsilon| = 0$}
        \label{figure:exp1_0}
    \end{subfigure}
    \caption{Compares the average performance of $\adaptedrobustagg(0.15)$, \inducb, and \naiveagg over randomly generated Bernoulli $0.15$-MPMAB problem instances with $K = 10$ and $M = 20$. The $x$-axis shows a horizon of $T = 100,000$ rounds, and the $y$-axis shows the cumulative collective regret of the players.}
    \label{figure:experiment1_allplots}
\end{figure}

\begin{figure}[htbp]
    \centering
    \begin{subfigure}{.32\textwidth}
        \centering
        \includegraphics[height=0.85\linewidth]{Figures/Exp2/5eps_mpl_ical=9_100000x30.png}
        \caption{$|\Ical_\epsilon| = 9$}
        \label{figure:exp2_9}
    \end{subfigure}
    \begin{subfigure}{0.32\textwidth}
        \centering
        \includegraphics[height=0.85\linewidth]{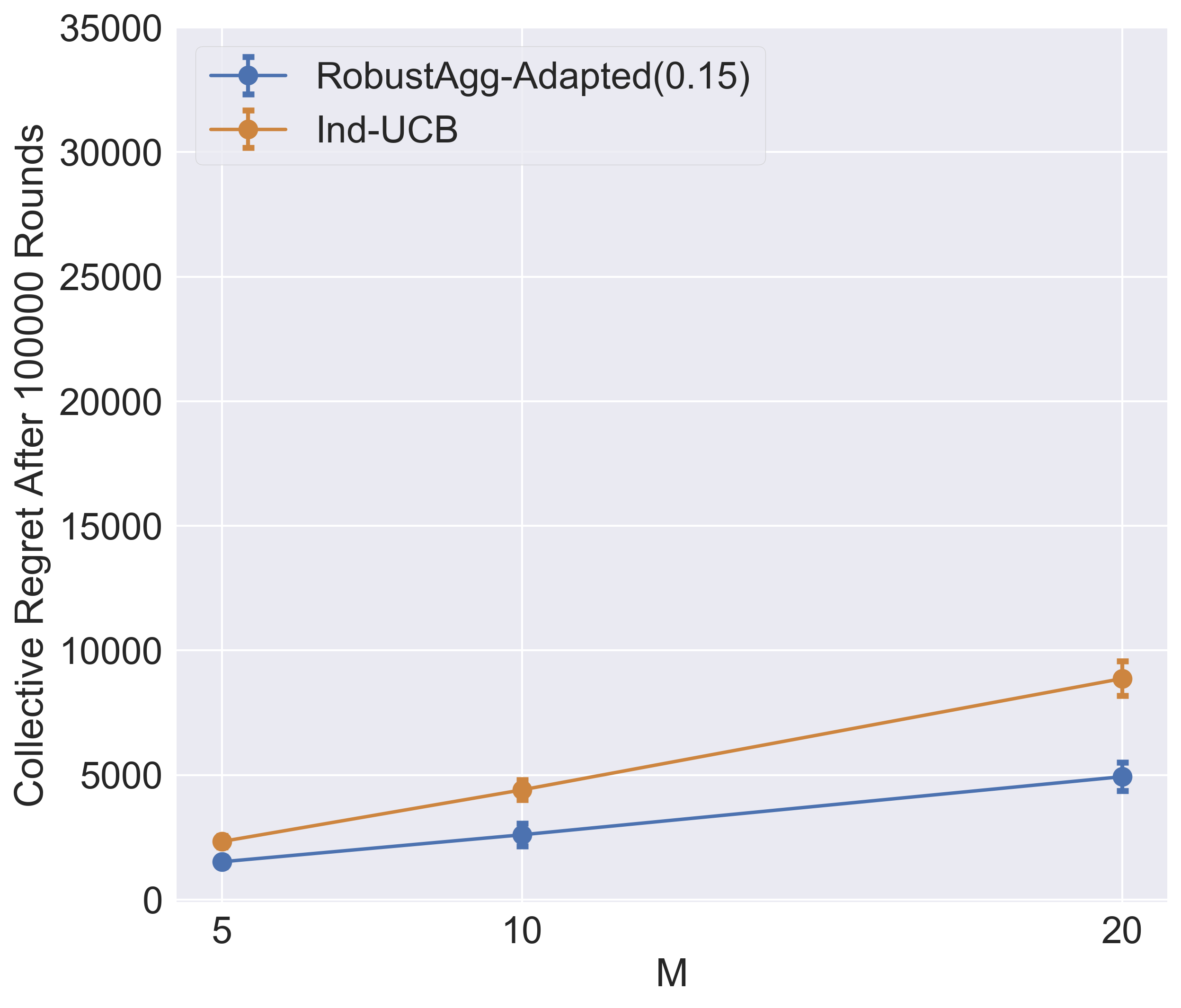}
        \caption{$|\Ical_\epsilon| = 8$}
        \label{figure:exp2_8}
    \end{subfigure}
    \begin{subfigure}{0.32\textwidth}
        \centering
        \includegraphics[height=0.85\linewidth]{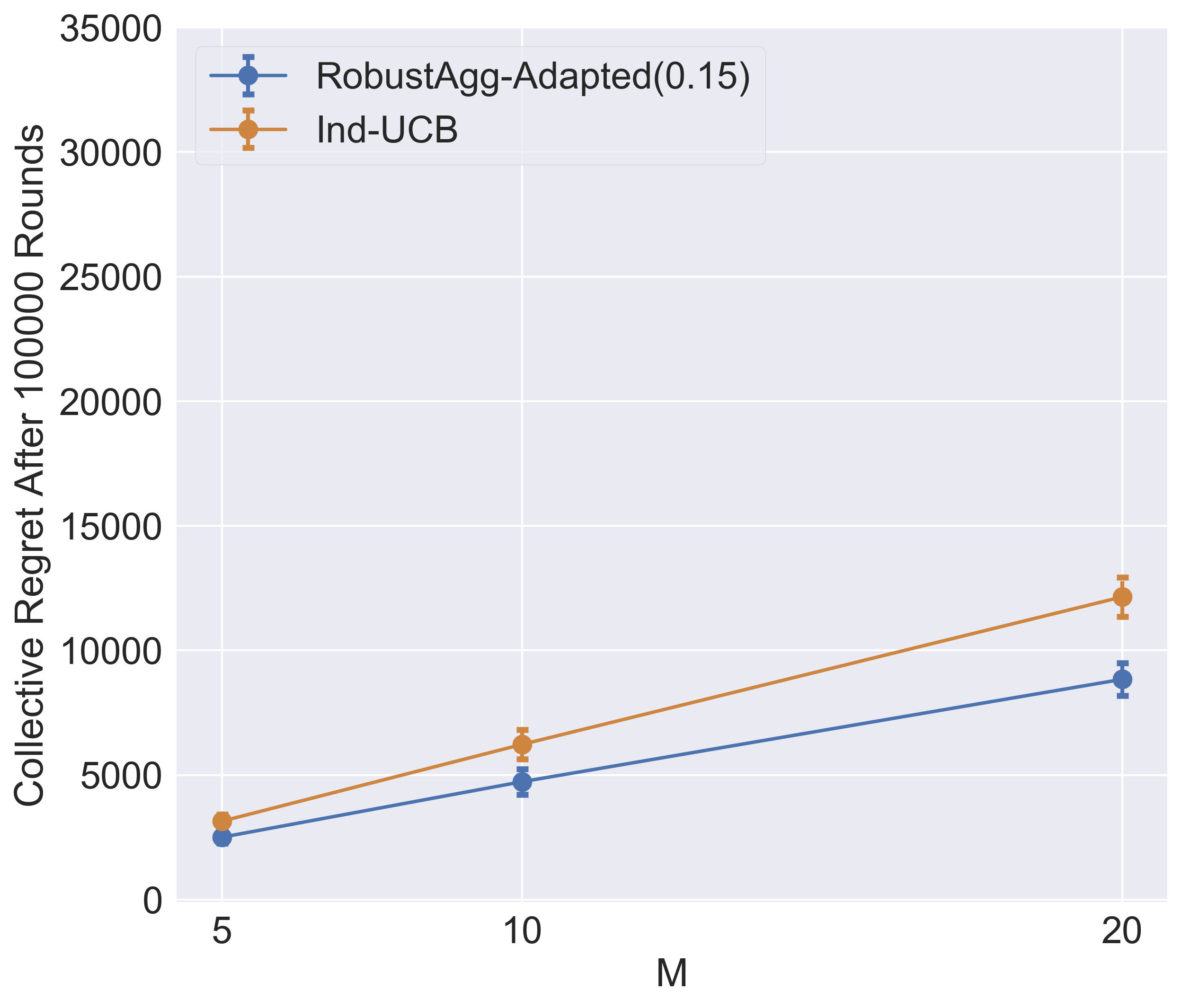}
        \caption{$|\Ical_\epsilon| = 7$}
        \label{figure:exp2_7}
    \end{subfigure}
        \begin{subfigure}{.32\textwidth}
        \centering
        \includegraphics[height=0.85\linewidth]{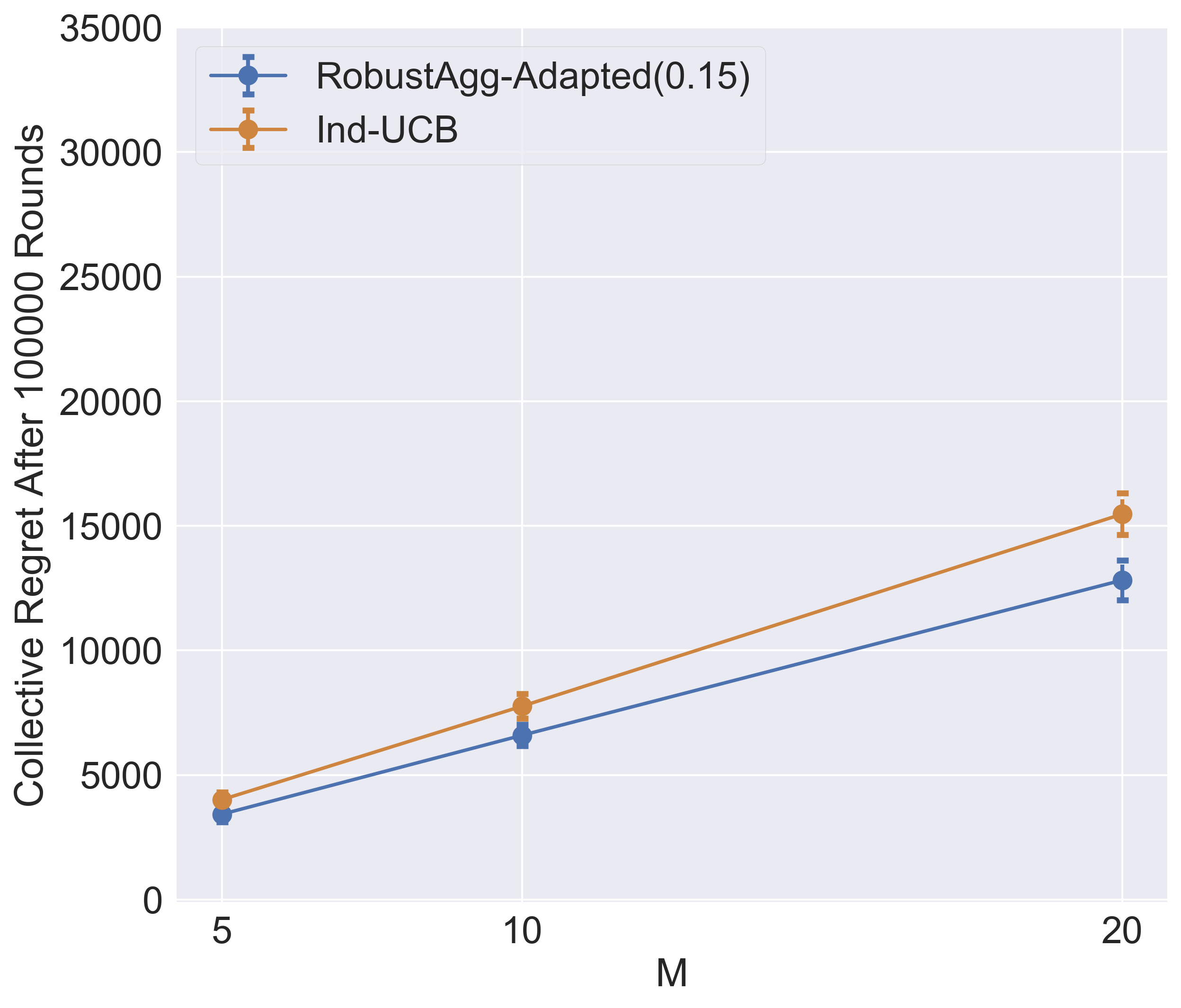}
        \caption{$|\Ical_\epsilon| = 6$}
        \label{figure:exp2_6}
    \end{subfigure}
    \begin{subfigure}{0.32\textwidth}
        \centering
        \includegraphics[height=0.85\linewidth]{Figures/Exp2/5eps_mpl_ical=5_100000x30.png}
        \caption{$|\Ical_\epsilon| = 5$}
        \label{figure:exp2_5}
    \end{subfigure}
    \begin{subfigure}{0.32\textwidth}
        \centering
        \includegraphics[height=0.85\linewidth]{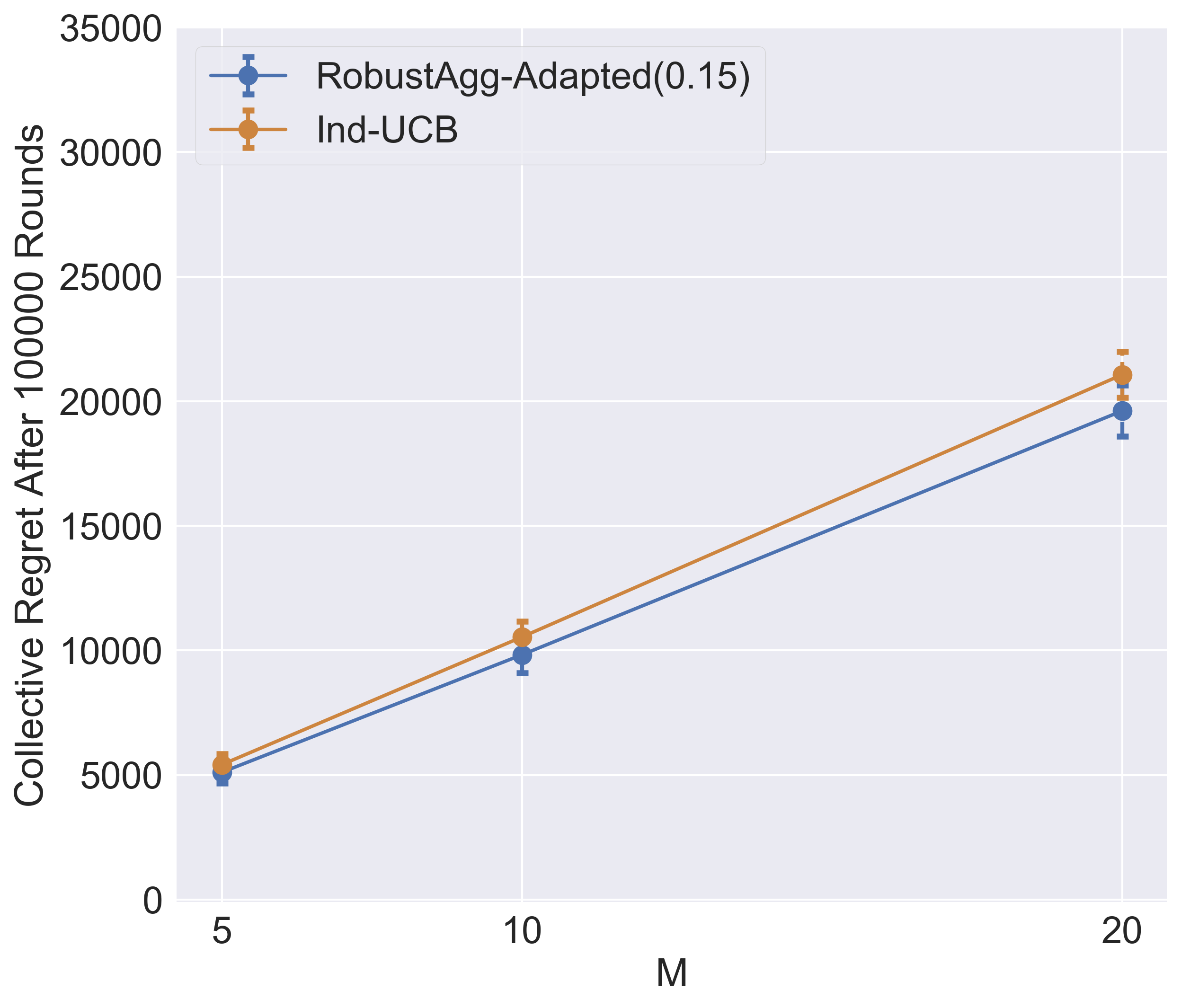}
        \caption{$|\Ical_\epsilon| = 4$}
        \label{figure:exp2_4}
    \end{subfigure}
        \begin{subfigure}{.32\textwidth}
        \centering
        \includegraphics[height=0.85\linewidth]{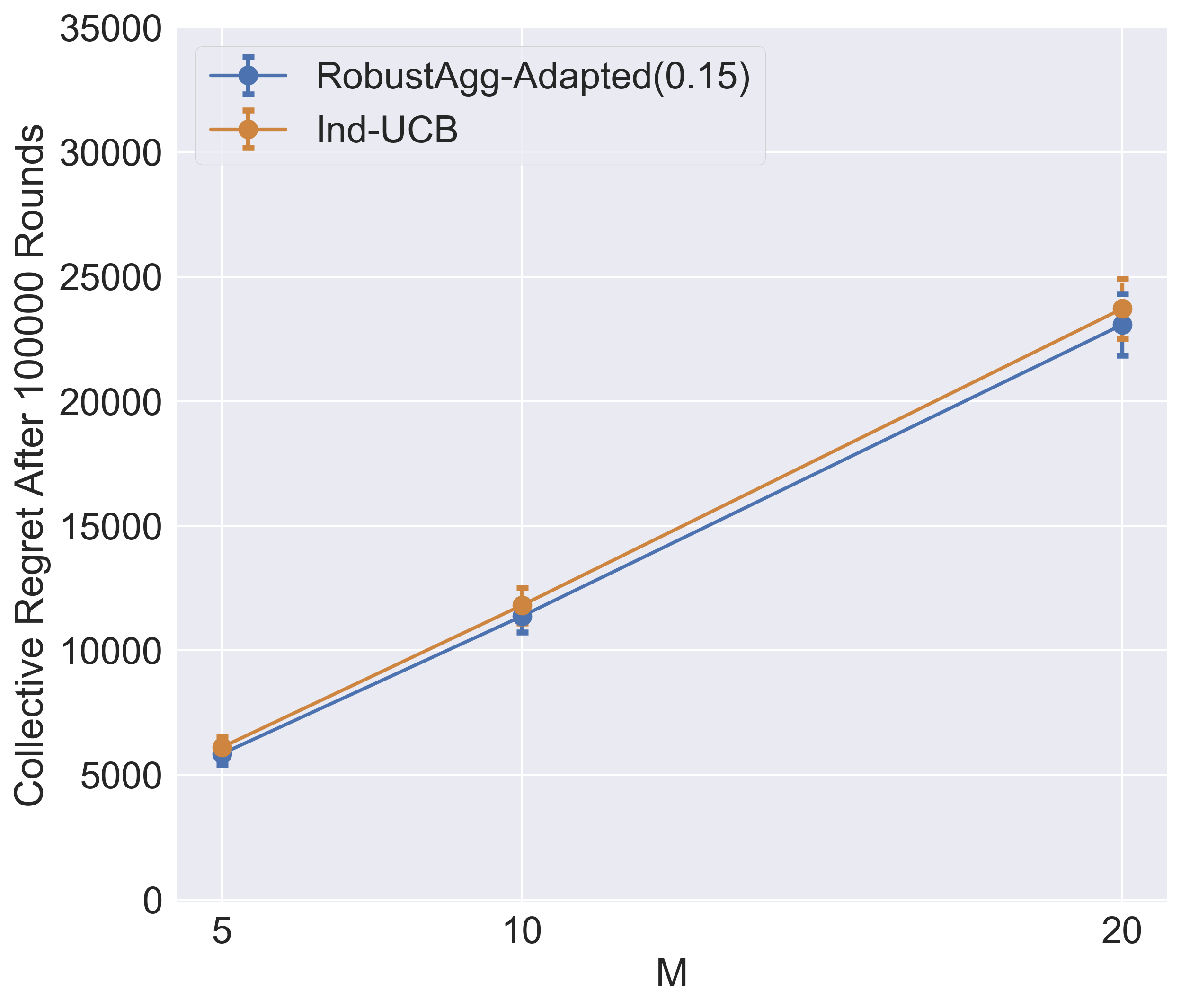}
        \caption{$|\Ical_\epsilon| = 3$}
        \label{figure:exp2_3}
    \end{subfigure}
    \begin{subfigure}{0.32\textwidth}
        \centering
        \includegraphics[height=0.85\linewidth]{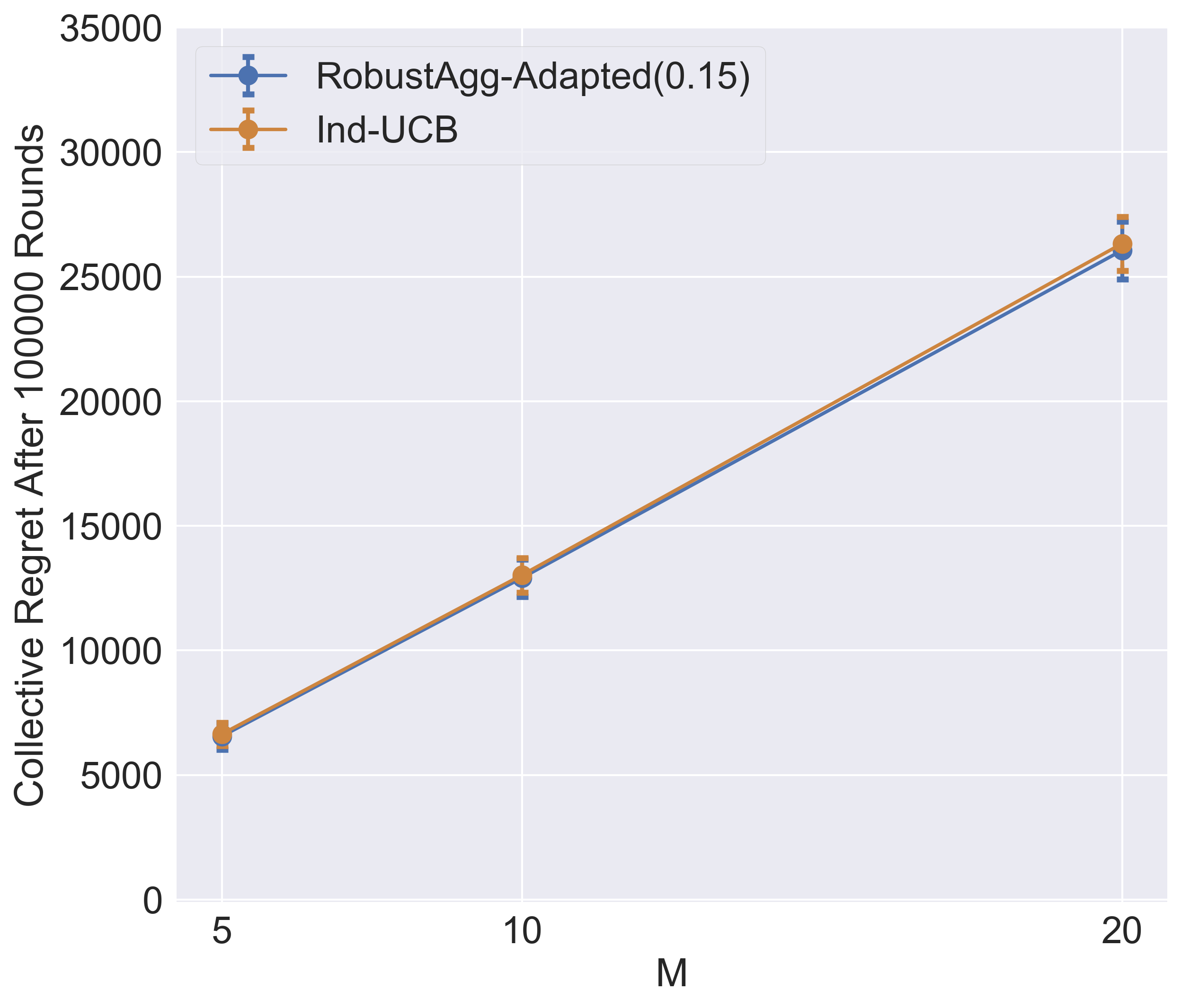}
        \caption{$|\Ical_\epsilon| = 2$}
        \label{figure:exp2_2}
    \end{subfigure}
    \begin{subfigure}{0.32\textwidth}
        \centering
        \includegraphics[height=0.85\linewidth]{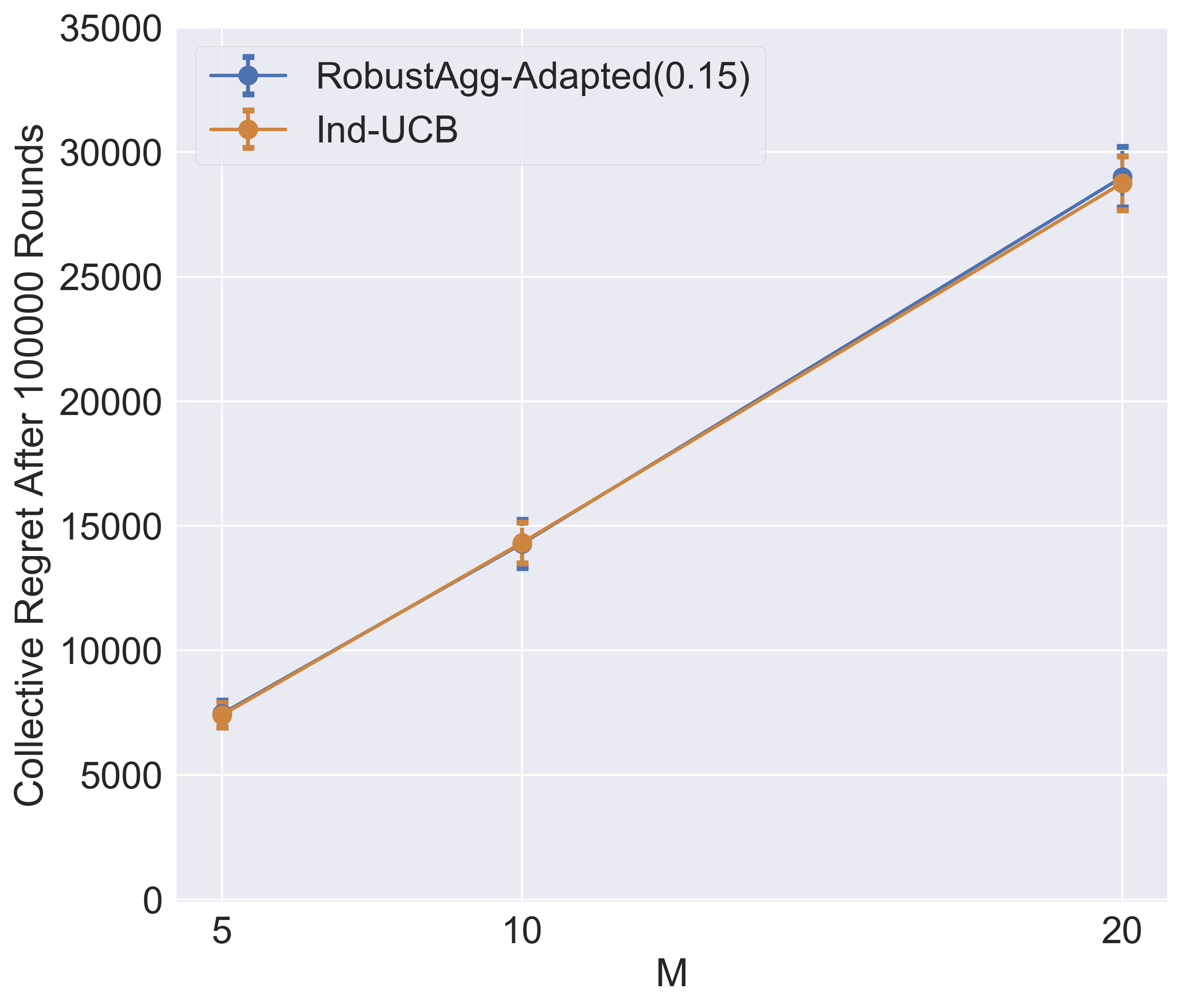}
        \caption{$|\Ical_\epsilon| = 1$}
        \label{figure:exp2_1}
    \end{subfigure}
    \begin{subfigure}{0.32\textwidth}
        \centering
        \includegraphics[height=0.85\linewidth]{Figures/Exp2/5eps_mpl_ical=0_100000x30.png}
        \caption{$|\Ical_\epsilon| = 0$}
        \label{figure:exp2_0}
    \end{subfigure}
    \caption{Compares the average performance of $\adaptedrobustagg(0.15)$ and \inducb over randomly generated Bernoulli $0.15$-MPMAB problem instances with $K = 10$ and $M \in \{5, 10, 20\}$. The $x$-axis shows different values of $M$, and the $y$-axis shows the cumulative collective regret of the players after $100,000$ rounds.}
    \label{figure:experiment2_allplots}
\end{figure}

\section{Experimental Details}
\label{app:experiments}

\label{appendix:experiments}

In Appendix~\ref{appendix:instance_generation_proof}, we provide a proof of Fact~\ref{fact:synthetic} which is about the instance generation procedure.
Then, in Appendix~\ref{app:more_exp_results}, we present comprehensive results from the simulations we performed.

\subsection{%
Proof of Fact~\ref{fact:synthetic}}
\label{appendix:instance_generation_proof}

\begin{proof}[Proof of Fact~\ref{fact:synthetic}]
For every $i$, as $\mu_i^p \in [\mu_i^1 -\frac\epsilon2, \mu_i^1 +\frac\epsilon2]$ for all $p$ in $[M]$, we have that for all $p,q$ in $[M]$,
$\abs{\mu_i^p - \mu_i^q} \leq \epsilon$. This proves that $\mu = (\mu_i^p)_{i \in [K], p \in [M]}$ is indeed a Bernoulli $\epsilon$-MPMAB instance.

 Recall that $d = \max_{i \in [c]} \mu_i^1 = \max_{i \in [K]} \mu_i^1$ is the optimal mean reward for player 1. We now show that $\Ical_\epsilon =  \cbr{c+1,\ldots,K}$ by a case analysis:
\begin{enumerate}
    \item First, we show that 
for all $i$ in $\cbr{c+1,\ldots,K}$, $i$ is in $\Ical_{\epsilon}$.
This is because $\mu_i^1$ is chosen from $[0, d - 5\epsilon)$, which implies that $\Delta_i^1 > 5\epsilon$.
    \item Second, for all $i$ in $\cbr{1,\ldots,c}$, we claim that $i \notin \Ical_{\epsilon}$. To this end, we show that for all $p$, $\Delta_i^p \leq 5\epsilon$.
    
    We start with the observation that
    $\mu_i^1 \in (d-\epsilon, d]$, which implies that $\Delta_i^1 = d - \mu_i^1 \leq \epsilon$.
    Now, it follows from Fact~\ref{fact:delta_difference} in Appendix~\ref{appendix:proof_of_fact} that for any $i \in [K]$ and $p \in [M]$, $|\Delta_i^p - \Delta_i^1| \le 2\epsilon$. 
    Therefore, we have $\Delta_i^p \le 3\epsilon$ for all $p$, which implies that any $i \in [c]$ cannot be in $\Ical_\epsilon$.
    
    \qedhere
\end{enumerate}
\end{proof}

\subsection{Extended results}
\label{app:more_exp_results}
Here, we present comprehensive results from the simulations we performed.

\paragraph{Experiment 1.}
Recall that for each $v \in \{0, 1, 2, \ldots, 9\}$, we generated $30$ Bernoulli $0.15$-MPMAB problem instances such that $|\Ical_\epsilon| = v$. \Cref{figure:experiment1_allplots} compares the average cumulative collective regrets of the three algorithms in a horizon of $100,000$ rounds over instances with different values of $|\Ical_\epsilon|$:

\begin{itemize}
    \item Notice that $\adaptedrobustagg(0.15)$ outperforms both baseline algorithms when $|\Ical_\epsilon| \in [2,8]$, as shown in Figures \ref{figure:exp1_8}, \ref{figure:exp1_7}, $\ldots$, \ref{figure:exp1_2}, especially when $|\Ical_\epsilon|$ is large.
    
    \item Figure~\ref{figure:exp1_9} shows that when $|\Ical_\epsilon| = 9$---i.e., when one arm is optimal for all players and the other arms are all subpar arms---\naiveagg and $\adaptedrobustagg(0.15)$ perform much better than \inducb with little difference between themselves. However, note that as long as there are more than one ``competitive'' arms---e.g., in Figure~\ref{figure:exp1_8} when $|\Ical_\epsilon^C| = 2$---the collective regret of \naiveagg can easily be nearly linear in the number of rounds.
    
    \item \Cref{figure:exp1_1} and \Cref{figure:exp1_0} demonstrate that when when there are very few arms or even no arm that is amenable to data aggregation, the performance of $\adaptedrobustagg(0.15)$ is still on par with that of \inducb.
\end{itemize}

\paragraph{Experiment 2.} 
Recall that for each $M \in \{5, 10, 20\}$ and $v \in \{0, 1, 2, \ldots, 9\}$, we generated $30$ Bernoulli $0.15$-MPMAB problem instances with $M$ players such that $|\Ical_\epsilon| = v$.
\Cref{figure:experiment2_allplots} shows and compares the average collective regrets of $\adaptedrobustagg(0.15)$ and \inducb after $100,000$ rounds in problem instances with $M = 5, 10$, and $20$, and in each subfigure, $|\Ical_\epsilon|$ takes a different value.

Observe that when $|\Ical_\epsilon|$ is large (e.g., in Figures \ref{figure:exp2_9}, \ref{figure:exp2_8},$\ldots$, \ref{figure:exp2_5}), the collective regret of $\adaptedrobustagg(0.15)$ is less sensitive to the number of players $M$, in comparison with \inducb. Especially, in the extreme case when $|\Ical_\epsilon| = 9$---i.e., all suboptimal arms are subpar arms---\Cref{figure:exp2_9} shows that the collective regret of $\adaptedrobustagg(0.15)$ has negligible dependence on $M$.

In conclusion, our empirical evaluation validate our theoretical results in Section~\ref{sec:known_eps}.

\section{Analytical Solution to $\lambda^*$}
\label{appendix:lambda_star}

We first present the following proposition simliar to the results in~\citep[Section 6 thereof]{bbckpv10}.
The original solution in~\cite{bbckpv10} has a $\min(1, \cdot)$ operation in the second case; we slightly simplify that result by showing that this operation is unnecessary.\footnote{In~\cite{bbckpv10}'s notation, this can also be seen directly by observing that when $m_T \geq D^2$, $v = \frac{m_T}{m_T + m_S} \cdot \del{1 + \frac{m_S}{ \sqrt{ D^2 (m_S + m_T) - m_S m_T }}} \leq \frac{m_T}{m_T + m_S} \cdot (1 + \frac{m_S}{m_T}) = 1$.}

\begin{proposition}
\label{prop:lambda_star}
Suppose $\beta \in (0,1)$.
Define function 
\[
    f(\alpha) = 2B\sqrt{\left( \frac{\alpha^2}{\beta} + \frac{(1-\alpha)^2}{1-\beta} \right)} + 2(1-\alpha)A,
\]
Then, $\alpha^* = \argmin_{\alpha \in [0,1]} f(\alpha)$ has the following form:
\[
\alpha^*
=
\begin{cases}
1 & \beta \geq \frac{B^2}{A^2}, \\
\beta \del{1 + \frac{1-\beta}{\sqrt{\frac{B^2}{A^2} - \beta(1-\beta)}}} & \beta < \frac{B^2}{A^2}.
\end{cases}
\]
\end{proposition}
Observe that when $\beta < \frac{B^2}{A^2}$, $\frac{B^2}{A^2} - \beta(1-\beta) > 0$, so the expression in the second case is well defined.

\begin{proof}
First, observe that $f$ is a strictly convex function, and therefore has at most one stationary point in $\RR$; and if it exists, it must be $f$'s global minimum.

Second, we study the monotonicity property of $f$ in $\RR$. To this end, we calculate $\alpha_0$, the stationary point of $f$.
We have
\[
f'(\alpha) 
=
2B 
\frac
{ \frac \alpha \beta - \frac {1-\alpha} {1-\beta} }
{ \sqrt{\frac{\alpha^2}{\beta} + \frac{(1-\alpha)^2}{1-\beta}} } 
-
2A
\]
By algebraic calculations, $f'(\alpha) = 0$
is equivalent to 
\[
\frac{\alpha - \beta}{\beta (1-\beta)}
=
\frac A B 
\sqrt{\frac{\alpha^2 - 2\beta\alpha + 1}{\beta (1-\beta)}}.
\]

This yields the following quadratic equation:
\[
\del{\frac{B^2}{A^2} - \beta(1-\beta) }\alpha^2
-
2\beta \del{\frac{B^2}{A^2} - \beta(1-\beta)} \alpha 
+
\beta^2 \del{\frac{B^2}{A^2} - (1-\beta)}
=
0,
\]
with the constraint that $\alpha > \beta$.
The discriminant of the above quadratic equation is $\Delta = 4\beta^2(1-\beta)^2 (\frac{B^2}{A^2} - \beta (1-\beta))$. If $\Delta \geq 0$, the stationary point is
\[
\alpha_0 
= 
\frac{2 \beta (\frac{B^2}{A^2} - \beta(1-\beta) ) + \sqrt{\Delta}}{ 2(\frac{B^2}{A^2} - \beta(1-\beta)) }
=
\beta
\del{1 + \frac{1-\beta}{\sqrt{\frac{B^2}{A^2} - \beta(1-\beta)}}}
\]

We now consider two cases:
\begin{enumerate}
    \item If $\beta(1-\beta) > \frac{B^2}{A^2}$, it can be checked that $\Delta < 0$, and consequently $f'(\alpha) < 0$ for all $\alpha \in \RR$, i.e., $f$ is monotonically decreasing in $\RR$.
    \item $\beta(1-\beta) \leq \frac{B^2}{A^2}$, we have that $f$ is monotonically decreasing in $(-\infty, \alpha_0]$, and monotonically increasing in $[\alpha_0, +\infty)$.
\end{enumerate}

We are now ready to calculate $\alpha^* = \argmin_{\alpha \in [0,1]} f(\alpha)$. 
\begin{enumerate}
    \item If $\beta(1-\beta) > \frac{B^2}{A^2}$, as $f$ is monotonically decreasing in $\RR$, $\alpha^* = 1$.
    \item If $\beta(1-\beta) \leq \frac{B^2}{A^2}$ and $\beta > \frac{B^2}{A^2}$, it can be checked that $\alpha_0 \geq 1$.
    As $f$ is monotonically decreasing in $(-\infty, \alpha_0] \supset [0,1]$, we also have $\alpha^* = 1$.
    \item If $\beta \leq \frac{B^2}{A^2}$, $\alpha_0 \in [0,1]$. Therefore, $\alpha^* = \alpha_0 = \beta
    \del{1 + \frac{1-\beta}{\sqrt{\frac{B^2}{A^2} - \beta(1-\beta)}}}$.
\end{enumerate}
In summary, we have the expression of $\alpha^*$ as desired.
\end{proof}

Algorithm~\ref{alg: mpmab}'s line~\ref{line:opt-lambda} computes
\begin{align*}
    \lambda^* = \argmin_{\lambda \in [0,1]}\ 8\sqrt{13(\ln T)[\frac{\lambda^2}{\wbar{n^p_i}(t-1)} + \frac{(1-\lambda)^2}{\wbar{m^p_i}(t-1)}]} + (1-\lambda)\epsilon;
\end{align*}
we now use Proposition~\ref{prop:lambda_star} to give its analytical form.
For notational simplicity, let $n = \wbar{n^p_i}(t-1)$ and $m = \wbar{m^p_i}(t-1)$. Applying Proposition~\ref{prop:lambda_star} with $A = \frac{\epsilon}{2}, B = 4\sqrt{\frac{13(\ln T)}{n+m}}$, and $\beta = \frac{n}{n+m}$, we have
\[
\lambda^*
=
\begin{cases}
1 & \epsilon > 0 \text{ and } n \geq \frac{832 (\ln T) }{\epsilon^2}, \\
\frac{n}{n+m} \left(1 + \epsilon m \sqrt{\frac{1}{832(\ln T)(n+m) - \epsilon^2 nm}}\right) & \text{otherwise}.
\end{cases}
\]

\section{On adaptive reward confidence interval construction under unknown $\epsilon$}
\label{sec:adaptive-ci}

Recall that $\robustagg(\epsilon)$ carefully utilizes the dissimilarity parameter $\epsilon$ to construct high-probability reward confidence bounds for all arms and players by inter-player information sharing.
In this section, we investigate the limitations of constructing reward confidence bounds by utilizing auxiliary data, in the setting when $\epsilon$ is unknown.

To this end, we consider a basic interval estimation problem: given source data $(z_1,\ldots,z_n)$ and target data $(x_1,\ldots,x_m)$ drawn iid from distributions $D_Z$ and $D_X$ supported on $[0,1]$ of unknown dissimilarity, how can one design an adaptive interval estimator for $\mu_X$, the mean of $D_X$? 
A naive baseline is to build the estimator by ignoring the source data: by Hoeffding's inequality, the interval centered at $\bar{x} = \frac1m \sum_{i=1}^m x_i$ of width $O(\sqrt{\frac{\ln\frac1\delta}{m}})$ is a $(1-\delta)$-confidence interval. This motivates the question: can one construct  adaptive $(1-\delta)$-confidence intervals that become much narrower when $D_X$ and $D_Z$ are very close and $n$ is large?

We show in this section that the aforementioned naive interval estimation is about the best one can do: any ``valid'' confidence interval construction for $\mu_X$ must have width $\Omega(\sqrt{\frac{1}{m}})$ for a wide family of $(D_X, D_Z)$ where $D_X$ and $D_Z$ are identical to $D_Z$, regardless of the value of $n$. This is in sharp contrast to the results in the setting when a dissimiliarity parameter between $D_X$ and $D_Z$ is known: if we know that $D_X = D_Z$ {\em apriori}, it is easy to construct a confidence interval of $\mu_X$ of length $O(\sqrt{\frac{1}{n+m}})$ by setting its center at $\frac{1}{m+n} (\sum_{i=1}^m x_i + \sum_{i=1}^n z_i)$.
Similar impossibility results of constructing adaptive and honest confidence intervals have appeared in~\citep{low1997nonparametric,juditsky2003nonparametric} in nonparametric regression.

To formally present our results, we set up some useful notation. Denote by $S = (x_1, \ldots, x_m, z_1, \ldots, z_n)$ the sample we observe; in this notation, sample $S$ is drawn from $D_X^m \otimes D_Z^n$. 
The following notion formalizes the idea of valid confidence interval construction. 

\begin{definition}
$I: \RR^{m+n} \times (0,1) \to \cbr{[a,b]: 0 \leq a \leq b \leq 1}$ is said to be an honest confidence interval construction procedure for $\mu_X$, if under all distributions $D_X$ and $D_Z$ supported on $[0,1]$, and $\delta \in (0,1)$,
\[
\PP_{S \sim D_X^m \otimes D_Z^n}\del{ \mu_X \in I( S, \delta ) }
\geq 
1 - \delta.
\]
\end{definition}

We have the following theorem. 
\begin{theorem}
Suppose $I$ is an honest confidence interval construction procedure for $\mu_X$. Then for any $m \geq 10$, $n$, $\mu \in [\frac 3 8, \frac 5 8]$, we have the following: 
under distributions $D_X = D_Z = \Ber(\mu)$,
\[
\EE_{S \sim D_X^m \otimes D_Z^n} \sbr{ \lambda\del{I(S, 0.1)} } 
\geq 
\frac{1}{10 \sqrt{m}},
\]
where $\lambda \del{J}$ denotes the length of interval $J$.
\end{theorem}

\begin{proof}
We consider two hypotheses: 
\begin{enumerate}
\item $H_1: D_X = \Ber(\mu), D_Z = \Ber(\mu)$; under this hypothesis, $\mu_X = \mu$.
\item $H_2: D_X = \Ber(\mu + \frac{1}{3\sqrt{m}}), D_Z = \Ber(\mu)$;
under this hypothesis, $\mu_X = \mu + \frac{1}{3\sqrt{m}}$.
\end{enumerate}
Denote by $\PP_{H_1}$ and $\PP_{H_2}$ the probability distributions of $S$ under hypotheses $H_1$ and $H_2$ respectively.
As $I$ is an honest confidence interval procedure for $\mu_X$, we must have
\begin{equation}
\PP_{H_1}\del{ \mu \in I(S, 0.1) } \geq 0.9,
\label{eqn:h1-include}
\end{equation}
and
\begin{equation}
\PP_{H_2}\del{ \mu + \frac{1}{3\sqrt{m}} \in I(S, 0.1) } \geq 0.9,
\label{eqn:h2-include}
\end{equation}
holding simultaneously. 
We now show that $\EE_{H_1} \abs{I(S, 0.1)} \geq \frac{1}{10\sqrt{m}}$.

We first establish a lower bound on $\PP_{H_1}\del{ \mu + \frac{1}{3\sqrt{m}} \in I(S, 0.1) }$ using Eq.~\eqref{eqn:h2-include}. The KL divergence between $\PP_{H_1}$ and $\PP_{H_2}$ can be bounded by:
\begin{align*} 
\KL(\PP_{H_1}, \PP_{H_2}) 
& = 
\sum_{i=1}^m \KL\del{\Ber(\mu), \Ber(\mu+\frac{1}{3\sqrt{m}})}
+
\sum_{i=1}^n \KL\del{\Ber(\mu), \Ber(\mu)} \\
& \leq 
m \cdot 3 (\frac{1}{3\sqrt{m}})^2 + n \cdot 0
= 
\frac13.
\end{align*}
where the inequality uses Lemma~\ref{lem:kl-ber} and $\mu, \mu + \frac{1}{3\sqrt{m}} \in [\frac14, \frac34]$.

By Pinsker's inequality, 
\[ 
d_{\TV}(\PP_{H_1}, \PP_{H_2}) \leq \sqrt{\frac12 \KL(\PP_{H_1}, \PP_{H_2}) } \leq 0.5.
\]

Therefore, 
\[ \PP_{H_1}\del{ \mu + \frac{1}{3\sqrt{m}} \in I(S, 0.1) } \geq \PP_{H_2}\del{ \mu + \frac{1}{3\sqrt{m}} \in I(S, 0.1) }  - d_{\TV}(\PP_{H_1}, \PP_{H_2}) \geq 0.4. \]

Combining the above inequality with Eq.~\eqref{eqn:h1-include}, using the fact that $\PP(U \cap V) \geq \PP(U) + \PP(V) - 1$, we have
\[
\PP_{H_1} \del{ \mu \in I(S, 0.1), \mu + \frac{1}{3\sqrt{m}} \in I(S, 0.1) }
\geq 
0.9+0.4-1
\geq
0.3.
\]
Observe that if $\mu \in I(S, 0.1)$ and $\mu + \frac{1}{3\sqrt{m}} \in I(S, 0.1)$ both happens, $\lambda\del{I(S, 0.1)} \geq \frac{1}{3\sqrt{m}}$. Therefore,
\[
\EE_{H_1} \sbr{ \lambda\del{I(S, 0.1)} } 
\geq 
\PP_{H_1} \del{ \mu \in I(S, 0.1), \mu + \frac{1}{3\sqrt{m}} \in I(S, 0.1) } \cdot 
\frac{1}{3\sqrt{m}} 
\geq
\frac{1}{10\sqrt{m}}.
\qedhere
\]
\end{proof}

\putbib
\end{bibunit}

\end{document}